\documentclass{article}

\usepackage[utf8]{inputenc} % allow utf-8 input
\usepackage[T1]{fontenc}    % use 8-bit T1 fonts
\usepackage[hidelinks]{hyperref}       % hyperlinks
\usepackage{url}            % simple URL typesetting
\usepackage{booktabs}       % professional-quality tables
\usepackage{amsfonts}       % blackboard math symbols
\usepackage{nicefrac}       % compact symbols for 1/2, etc.
\usepackage{microtype}      % microtypography
\usepackage{xcolor}         % colors
\usepackage{bbm}
\usepackage{graphicx}

\usepackage{authblk}
\usepackage{fullpage}

\usepackage{enumitem}

\usepackage[sort&compress, numbers]{natbib}

\usepackage{definitions}

\title{Controlling Counterfactual Harm in\\ Decision Support Systems Based on Prediction Sets}

\author{Eleni Straitouri}
\author{Suhas Thejaswi}
\author{Manuel Gomez Rodriguez}

\affil{Max Planck Institute for Software Systems \\ \{estraitouri, thejaswi, manuel\}@mpi-sws.org}

\date{}

\begin{document}

\maketitle

\begin{abstract}
  %
% 1. Decision support systems based on prediction sets
%
Decision support systems based on prediction sets help humans solve multiclass classification tasks 
by narrowing down the set of potential label values to a subset of them, namely a prediction set, and 
asking them to always predict label values from the prediction sets. 
%
% 2. Optimized for accuracy, not to prevent harm
%
While this type of systems have been proven to be~effec\-tive at im\-pro\-ving the average accuracy of the predictions made by humans, by restricting human agency, they may cause harm---a human who has succeeded 
at predicting the ground-truth label of an instance on their own may have failed had they used these systems.
%
% 3. Our goal
%
In this paper, our goal is to control how frequently a decision support system based on prediction sets 
may cause harm, by design.
%
% 4. Harm formalization using SCMs
%
To this end, we start by characterizing the above notion of harm using the theoretical framework of structural causal models.
%
% 4. Monotonicity properties allow for causal identifiability
%
Then, we show that, under a natural, albeit unverifiable, mo\-no\-to\-ni\-ci\-ty~assump\-tion, we can estimate how frequently a system may cause harm using only predictions made by humans on their own. 
Further, we also show that, under a weaker monotonicity assumption, which can be verified experimentally, we can bound how frequently a system may cause harm again using only predictions made by humans on their own.
%
% 5. Computational framework using conformal risk control
%
Building upon these assumptions, we introduce a computational framework to design decision support systems based on prediction sets that are guaranteed to cause harm less frequently than a user-specified value 
using conformal risk control. 
%
% 6. Showcase our framework using data from a human subject study
%
We validate our framework using real human predictions from two different human subject studies and show that, in decision support systems based on prediction sets, there is a trade-off between accuracy and counterfactual harm. 
\end{abstract}

\vspace{-2mm}
\section{Introduction}\label{sec:intro}
\vspace{-2mm}
%
% 1. Harm in decision making
%
The principle of ``first, do no harm'' holds profound significance in a variety of professions across multiple high-stakes domains. 
For example, in the field of medicine, doctors swear an oath to prioritize their patient's well-being or, 
in the legal justice system, preserving the innocence of individuals is paramount.
As a result, in all of these domains, rules and guidelines have been established to prevent decision makers---doctors or judges---from making decisions that harm individuals---patients or suspects. 
%
% 2. Harm in decision support systems
%
In recent years, it has been increasingly argued that a similar principle should apply to decision
support systems using machine learning algorithms in high-stakes domains~\cite{richens2022counterfactual,beckers2022causal,li2023trustworthy,beckers2023quantifying}.\footnote{The \href{https://www.europarl.europa.eu/doceo/document/TA-9-2024-0138-FNL-COR01_EN.pdf}{European Unions'{} AI} act mentions the term ``harm'' more than $35$ times and points out that, its crucial role in the design of algorithmic systems must be defined carefully.} 

%
% 3. Definition of harm
%
The definition of harm is not unequivocally agreed upon, however, the most widely accepted definition is the counterfactual comparative account of harm (in short, counterfactual harm)~\cite{feinberg1986wrongful,hanser2008metaphysics,klocksiem2012defense}, which we adopt in our work. 
Under this definition, an action causes harm to an individual if they would~have been in a worse state had the action been taken.
Building upon this definition, we say that a decision support system causes harm to an individual if a decision maker would have made a worse decision about the individual had they used the system.

%
% 3. Decision support systems based on prediction sets for  
% multiclass classification tasks
%
In machine learning for decision support, one of the main focus has been classification tasks. 
Here, the most studied setting assumes the decision support system uses a classifier to predict the value of a (ground-truth) label
of interest and a human expert uses the predicted value to update their own prediction~\cite{bansal2019beyond, lubars2019ask, bordt2020humans, vodrahalli2022uncalibrated,corvelo2024human}.
Unfortunately, in this setting, it is yet unclear how to~gua\-ran\-tee that the (average) accuracy of the predictions made by an expert who uses the system is higher than the accuracy of the predictions made by the expert and the classifier on their own, what is often referred to as human-AI complementarity~\cite{yin2019understanding,zhang2020effect,suresh2020misplaced,lai2021towards}.
In this context, a recent line of work~\cite{straitouri2023improving,straitouri2024designing} have argued, both theoretically and empirically, that an alternative setting may enable human-AI complementarity.
In this alternative setting, the decision support system helps a human expert by providing a set of label predictions, namely a prediction set, and asking them to always predict a label value from the set.
% \footnote{For example, Google recently developed a tool 
% that uses patient history and skin condition images to 
% provide prediction sets, which is referred as 
% differential diagnosis in the medical 
% domain~\cite{google2020skin}. A study by Jain et 
% al.~\cite{jain2021development} found that using this 
% tool improved diagnoses by physicians and nurses for 
% $1$ in every $8$ to $10$ cases.}
%
The key principle is that, by restricting human agency, good performance does not depend on the expert developing a good sense of when to predict a label from the prediction set.
In this context, it is also worth noting that Google has recently developed a tool that uses patient history and skin condition images to provide decision support using prediction sets~\cite{google2020skin}, and a study by Jain et al.~\cite{jain2021development} has found that physicians and nurses using this tool improved diagnoses for $1$ in every $8$ to $10$ cases.

%
% 4. Harm in decision support systems based on prediction sets
%
In this work, we argue that the same principle that enables human-AI complementarity on~de\-ci\-sion support systems based on prediction sets may also cause counterfactual harm---a human expert who has succeeded at predicting the label of an instance on their own may have failed had they used these systems.
Consequently, our goal is to design decision support systems based on prediction~sets that are guaranteed to cause, in average, less counterfactual harm than a user-specified~value.

%
% 5. Our contributions
% 
\xhdr{Our contributions} 
%
% 5.1 Causal model
%
We start by formally characterizing the predictions made by a decision maker using a decision support system based on prediction sets using a structural causal model (SCM) and,
%
% 5.2 Definition of counterfactual harm
%
based on this characterization, formalize our notion of counterfactual harm.
%
% 5.3 Monotonicity assumptions
%
In general, since counterfactual harm lies within level three in the ``ladder of causation''~\cite{pearl2009causal}, it is not (partially) identifiable---it cannot be estimated (bounded) from data.
However, we show that, under a natural counterfactual monotonicity assumption on the predictions made by decision makers using decision support systems based on prediction sets, counterfactual harm is identifiable.
Further, we show that, under a weaker interventional mo\-no\-to\-ni\-ci\-ty assumption, which can be verified experimentally, the average counterfactual harm is~par\-tia\-lly identifiable.
%
% 5.4 Computational framework
%
Then, building upon these assumptions, we develop a computational framework to design decision support systems based on prediction sets that are guaranteed to cause, in average, less counterfactual harm than a user-specified value  using conformal risk control~\cite{angelopoulos2024conformal}.
% 
% 5.5 Experiments
%
Finally, we validate our framework using real human predictions from two different human subject studies and show that, in decision support systems based on prediction sets, there is 
a trade-off between accuracy and counterfactual harm. 

%
% 6 Further related work
%
\xhdr{Further related work}
Our work builds upon further related work on set-valued predictors, critiques of prediction optimization, counterfactual harm and algorithmic triage.

%
% 6.1 Set-valued predictors
%
Set-valued predictors output a set of label values, namely a prediction set, rather than single labels~\cite{chzhen2021set}.
However, set-valued predictors have not been typically designed nor evaluated by their ability to help human experts make more accurate predictions~\cite{yang2017cautious,mortier2021efficient,ma2021partial,nguyen2021multilabel,angelopoulos2021learn,bates2021distribution}.
Only very recently, an emerging line of work has shown that conformal predictors, a specific type of set-valued predictors, may help human experts make more accurate predictions~\cite{straitouri2023improving,straitouri2024designing,babbar2022on,cresswell2024conformal,zhang2024evaluating,de2024towards}.
Within this line of work, the work most closely related to ours is by Straitouri et al.~\cite{straitouri2023improving,straitouri2024designing}, which has introduced the setting, and counterfactual and interventional monotonicity properties we build upon. 
%
% While improvements in prediction accuracy 
% is often the focus, some studies argue that such improvements do not always translate to better decision-making, highlighting the need for careful consideration of how decision support systems are integrated into real-world~\cite{wang2024against, liu2024actionability}.
% %
% In our work, instead of focusing on maximizing the average accuracy of the predictions made by humans as they do, 
% %
% we focus on controlling counterfactual harm.

%
% 6.2 Critiques of prediction optimization
%
Prediction optimization has been recently put into question in the context of decision support~\cite{van2023accurate, wang2024against, liu2024actionability}. More specifically, it has been argued that optimizing decision support systems to improve prediction accuracy
does not always translate to better decision-making. 
Our work aligns with this critique since we argue that improving prediction accuracy may come at the cost of counterfactual harm.  

%
% 6.3 Counterfactual harm
%
The literature on counterfactual harm in machine learning is still quite small and has focused on traditional machine learning settings in which machine learning models replace human decision makers and make automated decisions~\cite{richens2022counterfactual,beckers2022causal,li2023trustworthy,beckers2023quantifying}.
Within this literature, the work most closely related to ours is by Richens et al.~\cite{richens2022counterfactual}, which also uses
a structural causal model to define counterfactual harm. 
However, their definition of counterfactual harm differs in a subtle, but important, way from ours.
In our setting, their definition of counterfactual harm would compare the accuracy of a factual prediction made by an expert using the decision support system against the counterfactual prediction that the same expert would have made on their own.
That means, under their definition, one would need to deploy the system to estimate the harm it may cause, potentially causing harm.
In contrast, under our definition, one does not need to deploy the system to estimate (or bound) the harm it may cause, as it will become clear in Section~\ref{sec:assumptions}, and thus we argue that our definition may be more practical. 
% Refer to Appendix~\ref{app:notation} for a more detailed comparison between our 
% and their definition.
%
%

% 6.4 algorithmic triage
%
Learning under algorithmic triage seeks to develop classifiers that make predictions for a given fraction of the samples and leave the remaining ones to human experts, as instructed by a triage policy~\citep{raghu2019algorithmic,mozannar2020consistent,de2021classification,okati2021differentiable,charusaie2022sample,mozannar2023should}. 
In contrast, in our work, for each sample, a classifier is used to construct a prediction set and a human expert needs to predict a label value from the set.
In this context, it is also worth noting that learning under algorithmic triage has been extended to reinforcement learning settings~\cite{straitouri2021reinforcement,balazadeh2022learning,fuchs2023optimizing,tsirtsis2024responsibility}.

\vspace{-2mm}
\section{Decision support systems based on prediction sets}
\vspace{-2mm}
\label{sec:problem}
We consider a multiclass classification task in which, for each task instance, a human expert has to predict the value of a ground-truth label $y \in \Ycal = \{1, \ldots, L\}$.
Then, our goal is to design a decision support system $\Ccal : \Xcal \rightarrow 2^{\Ycal}$ that, 
given a set of features $x \in \Xcal$, helps the expert by narrowing down the label values to a subset of them $\Ccal(x) \subseteq \Ycal$, namely a prediction set. % , using a set-valued predictor~\cite{chzhen2021set}.
Here, we focus on a setting in which the system asks the expert to \emph{always} predict a label value $\hat y$ from the prediction set $\Ccal(x)$.
Note that, by restricting the expert's agency, good performance does not depend on the human expert developing a good sense of when to predict a label from the prediction set.
Moreover, we assume that the set of features, the ground-truth label and the expert'{}s prediction are sampled from an unknown fixed distribution\footnote{We denote random variables with capital letters and realizations of random variables with lowercase letters.}, \ie, $x, y \sim P(X, Y)$ and $\hat y \sim P(\hat Y \given X, Y, \Ccal(X))$.

Further, similarly as in Straitouri et al.~\citep{straitouri2023improving,straitouri2024designing}, 
we consider that, given a set of features $x \in \Xcal$, the system constructs the prediction set $\Ccal(x)$ using the following set-valued predictor~\cite{chzhen2021set}.
First, the set-valued predictor ranks each potential label value $y \in \Ycal$ using the softmax output of a pre-trained classifier $m_{y}(x) \in [0, 1]$.
Then, given a user-specified threshold $\lambda \in [0,1]$, it uses the resulting ranking to construct the prediction set $\Ccal(x) = \Ccal_{\lambda}(x)$ as follows: 
\begin{equation}\label{eq:set-valued-predictor}
\Ccal_{\lambda}(x) = \{y_{(i)}\}_{i=1}^{k}, \,\, \text{with} \,\, k = 1 + \sum_{j=2}^{L}\mathbbm{1}\left \{ m_{y_{(j)}}(x) \geq 1 - \lambda \right\},
\end{equation}
% where $f_{y}(x)$ is the output of the predictive model $f$ for the label $y$. 
% 
where $\cdot_{(i)}$ denotes the $i$-th label value in the ranking. 
Here, note that, for $\lambda=0$, the prediction set contains just the top ranked label value, for $\lambda =1$, it contains all label values.\footnote{The assumption that $m_y(x) \in [0,1]$ and $\lambda \in [0, 1]$ is without loss of generality.}

Given the above parameterization, one may just focus on finding the optimal threshold $\lambda^{*}$ under which the human expert achieves the highest average accuracy as in Straitouri et al.~\citep{straitouri2023improving,straitouri2024designing}, \ie,
\begin{equation}\label{eq:acc-objective}
 \lambda^{*} = \argmax_{\lambda \in [0,1]} A(\lambda), \,\, \text{with} \,\, A(\lambda) = \EE_{X,Y \sim P(X, Y),\, \hat{Y} \sim P(\hat{Y}\given X, Y, \Ccal_{\lambda}(X)) }[\mathbbm{1} \{\hat{Y} = Y \}].
\end{equation}
However, this focus does not prevent the resulting system $\Ccal_{\lambda}$ from causing harm---an expert may succeed to predict the value of the ground-truth label on their own on instances in which they would have failed had they used $\Ccal_{\lambda}$.
In this work, our goal is to design a computational framework that, 
given a user-specified bound $\alpha \in [0,1]$, 
finds the set of $\lambda$ values which are all guaranteed to cause less harm, in average, than the bound $\alpha$.
\begin{figure}
    \centering
    \includegraphics[scale=0.45]{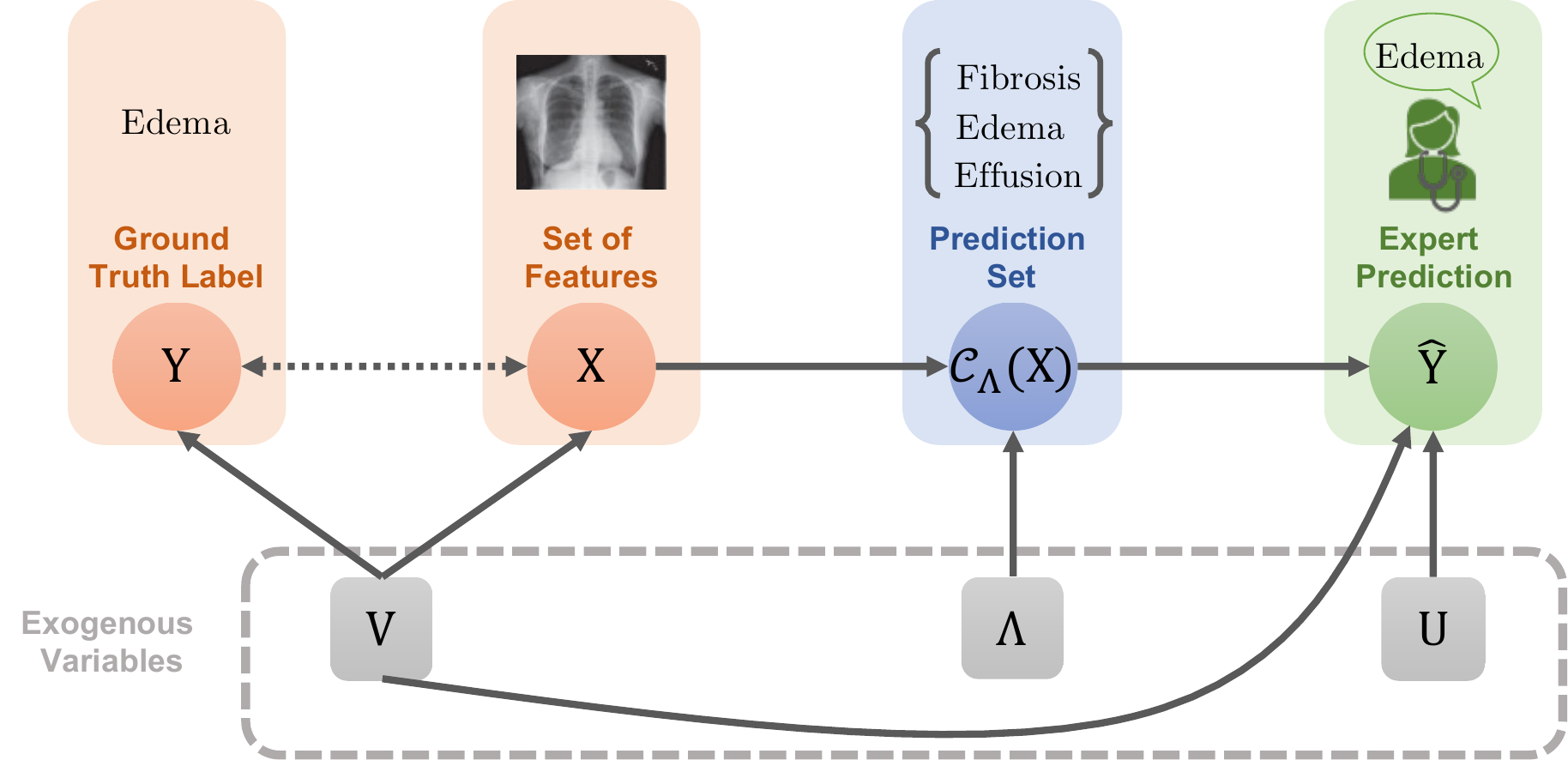}
    \caption{Our structural causal model $\Mcal$. Circles represent endogenous random variables and boxes represent exogenous random variables. The value of each endogenous variable is given by a function of the values of its ancestors, as defined by Eq.~\ref{eq:scm}. The value of each exogenous variable is sampled independently from a given distribution.}
    \label{fig:scm}
\end{figure}

\xhdr{Remark} We would like to clarify that, if one sets the value of the threshold $\lambda$ to be roughly the $1-\alpha$ quantile of the empirical distribution of the scores $1-m_{y}(x)$ in a calibration set, then, the set valued predictor defined by Eq.~\ref{eq:set-valued-predictor} is equivalent to a vanilla conformal predictor with nonconformity scores $1-m_{y}(x)$ and coverage $1-\alpha$.
Under this view, it becomes apparent that, by searching for $\lambda$ values in $[0,1]$ that are harm controlling, we are essentially searching for vanilla conformal predictors that are harm controlling.
In this context, we would like to further clarify that our framework is agnostic to the choice of nonconformity score or more generally, the set-valued predictor, used to construct the prediction sets~\cite{romano2020aps, angelopoulos2021raps, huang2023conformal}.
Motivated by this observation, in Appendix~\ref{app:saps}, we include additional experiments where we evaluate our framework using a more complex set-valued predictor~\cite{huang2023conformal}.  

\vspace{-2mm}
\section{Counterfactual harm  of decision support systems}
\vspace{-2mm}
\label{sec:cf-harm-definition}
% A causal model
To formalize our notion of harm, we characterize how human experts make predictions using a decision support system $\Ccal$
via a structural causal model (SCM)~\cite{pearl2009causal}, which we denote as $\Mcal$.
More specifically, similarly as in Straitouri et al.~\cite{straitouri2024designing}, we define $\Mcal$ by the following set of assignments:
\begin{align}\label{eq:scm}
     X = f_X(V), \qquad Y = f_Y(V), \qquad \Ccal_{\Lambda}(X) = f_{\Ccal}(\Lambda, X),
     \qquad \hat{Y} &= f_{\hat{Y}}(U,V, \Ccal_{\Lambda}(X)),
\end{align}
where $\Lambda$, $U$ and $V$ are exogenous random variables and $f_X$, $f_Y$, $f_{\Ccal}$
and $f_{\hat{Y}}$ are given functions.\footnote{The functions $f_X$, $f_Y$, $f_{\Ccal}$
and $f_{\hat{Y}}$ are causal mechanisms and not equations that can be manipulated~\cite{pearl2009causality}.}
The exogenous variables $\Lambda$, $U$ and $V$ characterize the user-specified threshold, the expert'{}s individual characteristics and the data generating process, respectively. 
The function $f_{\Ccal}$ is directly defined by Eq.~\ref{eq:set-valued-predictor}, \ie, $f_{\Ccal}(\Lambda, X) = \Ccal_{\Lambda}(X)$.
Further, as argued elsewhere~\citep{pearl2009causal},  
we can always find a distribution for the exogenous variables $\Lambda$, $U$ and $V$ as well as a functional form for the functions $f_{X}$, $f_{Y}$ and $f_{\hat Y}$ such that the distributions of the features, the ground-truth label and the expert'{}s prediction introduced in Section~\ref{sec:problem} are given by the observational distribution entailed by the SCM $\Mcal$.
%
%  finding such a distribution and functional forms 
% requires no further assumptions, due to the high 
% expressiveness of SCMs. 
% 
For ease of exposition, we assume that, under no interventions, the distribution of the exogenous variable $\Lambda$ is $P(\Lambda) = \mathbbm{1}\{\Lambda = 1\}$ and thus human experts make predictions on their own.
% since the prediction set always contains all label values.
%
Figure~\ref{fig:scm} shows a visual representation of our SCM $\Mcal$.

Building upon the above characterization, 
we are now ready to formalize the following notion of counterfactual harm, 
which essentially compares the accuracy of a factual prediction made by an 
expert on their own against the counterfactual prediction that the same expert would have made had they used a decision support system $\Ccal_{\lambda}$: 
\begin{definition}[Counterfactual Harm]
\label{def:cf-harm}
    For any $x, y, \hat y \sim P^{\Mcal}$, the counterfactual harm that a decision support system $\Ccal_{\lambda}$ would have caused, if deployed, is given by\footnote{To denote interventions in a counterfactual distribution, we follow the notation by Peters et al.~\cite{peters2017elements}. Refer to Appendix~\ref{app:notation} for a comparison to Pearl's notation~\cite{pearl2009causality}.}
    \begin{equation} \label{eq:cf-harm}
        h_{\lambda}(x, y, \hat{y}) = \EE_{\hat{Y} \sim P^{\Mcal ; \text{do}(\Lambda=\lambda) \given \hat{Y}=\hat{y}, X=x, Y=y}(\hat Y)}
        [\max\{0, \mathbbm{1}\{\hat{y} = y\} - \mathbbm{1}\{\hat{Y} = y \} \}],
    \end{equation}
    where $\text{do}(\Lambda=\lambda)$ denotes a (hard) intervention on the exogenous variable $\Lambda$.
\end{definition}
Here, note that counterfactual harm can only be nonzero if the expert has made a successful prediction on their own, \ie, $\hat y = y$. 
Otherwise, the expert'{}s prediction $\hat y$ could not have become worse had they used the decision support system $\Ccal_{\lambda}$.

% Our goal: counterfactual harm control
Given the above definition of counterfactual harm and a user-specified bound $\alpha \in [0, 1]$, 
our goal is to find the largest harm-controlling set of values $\Lambda(\alpha) \subseteq [0, 1]$ such that, for each $\lambda \in \Lambda(\alpha)$, it holds that the counterfactual harm is, in expectation across all possible instances, smaller than $\alpha$, \ie,
\begin{equation} \label{eq:cf-harm-risk-control}
    \Lambda(\alpha) = \left\{ \lambda \in [0, 1] \, \given \, H(\lambda) = \EE_{X, Y, \hat Y \sim P^{\Mcal}(X, Y, \hat Y)}[h_{\lambda}(X, Y, \hat Y)] \leq \alpha \right\}.
\end{equation}
At this point, we cannot expect to find the set $\Lambda(\alpha)$ because counterfactual harm lies within level three in the ``ladder of causation''~\citep{pearl2009causal} and thus it is not identifiable from observational data without further assumptions. 
However, in what follows, we will show that, 
under certain assumptions, the average counterfactual harm $H(\lambda)$ is (partially) identifiable, \ie, it can be estimated (bounded) using observational data. 

\xhdr{Comparison to Richens's definition of counterfactual harm}
Richens et al.~\cite{richens2022counterfactual} define counterfactual harm as follows:
\begin{align*}
    h_{\lambda}(x, y, \hat{y}) = E_{\hat{Y} \sim P^{\mathcal{M} | \hat{Y} = \hat{y}, X = x, Y=y}(\hat{Y})}[ \max\{0, \mathbbm{1}\{\hat{Y} = y\} - \mathbbm{1}\{\hat{y} = y\}\} ],
\end{align*}
where $x, y, \hat{y} \sim P^{\mathcal{M} ; do(\Lambda=\lambda)}$ and $\lambda=1$ is considered to be the \emph{default action} in the language of Richens et al.~\cite{richens2022counterfactual}. 
This definition implicitly assumes that the system $\Ccal_{\lambda}$ is deployed and it compares the factual prediction $\hat{y}$ made by an expert using the deployed system against the counterfactual prediction $\hat{Y}$ had the expert made on their own. 
On the contrary, our definition does not assume that the system $\Ccal_{\lambda}$ is deployed and instead it compares the factual prediction $\hat{y}$ made by an expert on their own against the counterfactual prediction $\hat{Y}$ had the expert made using the system $\Ccal_{\lambda}$.

\vspace{-2mm}
\section{Counterfactual harm under counterfactual and interventional monotonicity}
\vspace{-2mm}
\label{sec:assumptions}
In this section, we analyze counterfactual harm $h_{\lambda}(x, y, \hat y)$, as defined in Eq.~\ref{eq:cf-harm}, from the perspective 
of causal identifiability under two natural monotonicity assumptions---counterfactual monotonicity and interventional monotonicity.
Both of these assumptions, which were first studied by Straitouri et al.~\cite{straitouri2024designing}, formalize the intuition that increasing the number of label values in a prediction set increases its difficulty.

Under counterfactual monotonicity, 
for any $x \in \Xcal$ and $\lambda, \lambda'  \in [0,1]$ such that $Y \in \Ccal_{\lambda}(x) \subseteq \Ccal_{\lambda'}(x)$, 
if an expert has succeeded at predicting $Y$ using $\Ccal_{\lambda'}$, they would have also succeeded had they used $\Ccal_{\lambda}$ and, conversely, if they have failed at predicting $Y$ using $\Ccal_{\lambda}$, they would have also failed had they used $\Ccal_{\lambda'}$, while holding ``everything else fixed''.
More formally, counterfactual monotonicity is defined as follows:
\begin{assumption}[Counterfactual Monotonicity]\label{asm:counterfactual-monotonicity}
Counterfactual monotonicity holds if and only if, for any $x \in \Xcal$ and any $\lambda, \lambda' \in [0,1]$ such that $Y \in \Ccal_{\lambda}(x) \subseteq \Ccal_{\lambda'}(x)$, 
we have that 
\begin{equation} \label{eq:counterfactual-monotonicity}
\mathbbm{1}\{ f_{\hat Y}(u, v, \Ccal_{\lambda}(x)) = Y \} \geq \mathbbm{1}\{ f_{\hat Y}(u, v, \Ccal_{\lambda'}(x)) = Y \}
\end{equation}
for any $u \sim P^{\Mcal}(U)$ and $v \sim P^{\Mcal}(V \given X = x)$.
\end{assumption} 
Under interventional monotonicity,
for any $x \in \Xcal$ and $\lambda, \lambda'  \in [0,1]$ such that $Y \in \Ccal_{\lambda}(x) \subseteq \Ccal_{\lambda'}(x)$,
the probability that experts succeed at predicting $Y$ using $\Ccal_{\lambda}$ is equal or greater than using $\Ccal_{\lambda'}$. 
More formally, interventional monotonicity is defined as follows\footnote{In Straitouri et al.~\cite{straitouri2024designing}, interventional monotonicity is originally defined unconditionally of the value of ground-truth label $Y$.}:
\begin{assumption}[Interventional Monotonicity]
\label{asm:cond-interventional-monotonicity}
Interventional monotonicity holds if and only if, for any $x \in \Xcal$, $y \in \Ycal$, and $\lambda, \lambda' \in [0,1]$ such that $y \in \Ccal_{\lambda}(x) \subseteq \Ccal_{\lambda'}(x)$, we have that 
\begin{equation} \label{eq:interventional-monotonicity}
    P^{\Mcal \,;\, \text{do}(\Lambda = \lambda)}(\hat{Y} = Y \given X = x, Y=y) \geq P^{\Mcal \,;\, \text{do}(\Lambda = \lambda')}(\hat{Y} = Y \given X = x ,Y=y),
\end{equation}
where the probability is over the exogenous random variables $U$ and $V$ characterizing the expert'{}s individual characteristics and the data generating process, respectively.
\end{assumption}

In what follows, we first show that, under the counterfactual monotonicity assumption, counterfactual harm is identifiable (we provide all proofs in Appendix~\ref{app:proofs}):
\begin{proposition}\label{prop:harm-counterfactual-monotonicity}
    Under the counterfactual monotonicity assumption, for any $x, y, \hat y \sim P^{\Mcal}$, the counterfactual harm that a decision support system $\Ccal_{\lambda}$ would have caused, if deployed, is given by
    \begin{equation}
    \label{eq:harm-counterfactual-mon}
    h_{\lambda}(x, y, \hat y) = \mathbbm{1}\{\hat{y} = y \wedge y \notin \Ccal_{\lambda}(x)\}.
   \end{equation}
\end{proposition}
As an immediate consequence, we can conclude that, under counterfactual monotonicity, the average counterfactual harm $H(\lambda)$, as defined in Eq.~\ref{eq:cf-harm-risk-control}, is identifiable and can be estimated using observational data sampled from $P^{\Mcal}$.
However, since the inequality condition of the counterfactual monotonicity assumption compares counterfactual predictions and thus cannot be experimentally verified, one should be cautious about using the above proposition to estimate counterfactual harm in high-stakes applications.

Next, we show that, under the interventional monotonicity assumption, the average counterfactual harm is partially identifiable:
% , and it can be lower- and upper-bounded using observational data sampled from 
% $P^{\Mcal}$:
%
\begin{proposition}\label{prop:harm-bound-interventional-monotonicity}
Under the interventional monotonicity assumption, the average counterfactual harm $H(\lambda)$ that a decision support system $\Ccal_{\lambda}$ would have caused, if deployed, satisfies that
\begin{multline}\label{eq:harm-bound-interventional-monotonicity}
 \EE_{X,Y,\hat Y \sim P^{\Mcal}(X,Y,\hat Y)}[ \mathbbm{1}\{\hat{Y} = Y \wedge Y \notin \Ccal_{\lambda}(X)\} ] \leq H(\lambda) \\ \leq 
 \EE_{X, Y, \hat{Y} \sim P^{\Mcal}(X, Y, \hat Y)} [ \mathbbm{1}\{\hat{Y} = Y \wedge Y \notin \Ccal_{\lambda}(X)\} + \mathbbm{1}\{\hat{Y} \neq Y  \wedge Y \in \Ccal_{\lambda}(X)\} ].
\end{multline}
\end{proposition}
Importantly, note that, in the above proposition, the lower bound on the left hand side of Eq.~\ref{eq:harm-bound-interventional-monotonicity}~matches the ave\-rage counterfactual harm under the counterfactual monotonicity assumption and thus, holding ``everything else fixed'', the average counterfactual harm 
under interventional monotonicity is always greater or equal than the average counterfactual harm under counterfactual
monotonicity.
Moreover, further note that the ine\-qua\-lity condition of the interventional monotonicity assumption compares interventional probabilities and thus we can experimentally verify it (see Appendix~\ref{app:interventional}), lending support to using the above proposition to bound average counterfactual harm in high-stakes applications.

\vspace{-2mm}
\section{Controlling counterfactual harm using conformal risk control}
\vspace{-2mm}
\label{sec:control}
In this section, we develop a computational framework that, 
given a decision support system $\Ccal_{\lambda}$ and a user-specified bound $\alpha$,
aims to find the largest harm-controlling set $\Lambda(\alpha)$, as defined in Eq.~\ref{eq:cf-harm-risk-control}.
In the development of our framework, we will first assume that counterfactual monotonicity holds and, later on, we will relax this assumption, and assume instead that interventional monotonicity holds.

Our framework builds upon the the idea of conformal risk control, which has been introduced very recently by Angelopoulos et al.~\cite{angelopoulos2024conformal}. 
Given any monotone loss function $\ell(\Ccal_{\lambda}(X), Y)$ with respect to $\lambda$ and a calibration set $\{(X_i, Y_i)\}_{i=1}^{n}$, with $(X_i, Y_i) \sim P(X, Y)$, conformal risk control finds a value of $\lambda$ under which the expected loss of a test sample $(X_{n+1}, Y_{n+1})~\sim~P(X, Y)$ does not ex\-ceed a user-specified bound $\alpha$, \ie, $\EE[\ell(\Ccal_{\lambda}(X_{n+1}), Y_{n+1})] \leq \alpha$. 
However, in our framework, we re-define the loss $\ell$ so that it does not only depend on the prediction set and the label value but also on the expert'{}s prediction on their own. 

Under the counterfactual monotonicity assumption, we set the value of the loss $\ell$ using the expression of counterfactual harm in Eq.~\ref{eq:harm-counterfactual-mon}
and, using a similar proof technique as in Angelopoulous et al., first prove the following theorem:
\begin{theorem}
\label{thm:harm-control-counterfactual-mon}
    Let $\Dcal = \{(X_i, Y_i, \hat{Y}_{i})\}_{i=1}^{n}$ be a calibration set, with $(X_i, Y_i, \hat{Y}_{i})\sim P^{\Mcal}(X,Y,\hat Y)$, $\alpha \in [0,1]$ be a user-specified bound, and
    \begin{equation} \label{eq:hat-lambda}
    \hat{\lambda}(\alpha) = \inf \left \{ \lambda \,:\, \frac{n}{n+1} \hat{H}_{n}(\lambda) + \frac{1}{n+1} \leq \alpha \right \} \,\, \text{where} \,\, \hat{H}_n(\lambda) = \frac{\sum_{i=1}^{n} \mathbbm{1}\{\hat{Y_i} =  Y_i \wedge Y_i \notin \Ccal_{\lambda}(X_i)\} }{n}.
    \end{equation}
    If counterfactual monotonicity holds, a test sample $(X_{n+1}, Y_{n+1}, \hat Y_{n+1}) \sim P^{\Mcal}(X,Y,\hat Y)$ satisfies that
    \begin{equation*}
        \EE\left[ \mathbbm{1}\{\hat{Y}_{n+1} =  Y_{n+1} \wedge Y_{n+1} \notin \Ccal_{\hat{\lambda}(\alpha)}(X_{n+1})\} \right] \leq \alpha,
    \end{equation*}
    where the expectation is over the randomness in the calibration set used to compute the threshold $\hat \lambda(\alpha)$ and the test sample.
\end{theorem}
Then, we leverage the above theorem and the fact that, under counterfactual monotonicity, the counterfactual harm is nonincreasing with respect to $\lambda$, as shown in Lemma~\ref{lem:non-increasing} in Appendix~\ref{app:lemmas}, to recover the largest harm-controlling set $\Lambda(\alpha)$: % , as formalized in the following corollary: 
\begin{corollary}\label{cor:counterfactual}
Let $\alpha \in [0,1]$ be a user-specified bound. Then, under the counterfactual monotonicity assumption, it
holds that $\Lambda(\alpha) = \{ \lambda \in [0,1] \given \lambda \geq \hat{\lambda}(\alpha) \}$, 
% \end{equation*}
where $\hat{\lambda}(\alpha)$ is given by Eq.~\ref{eq:hat-lambda}. 
\end{corollary}

Under the interventional monotonicity assumption, rather than directly controlling the counterfactual harm, 
we will control the upper bound given by Eq.~\ref{eq:harm-bound-interventional-monotonicity}. Consequently, 
rather than recovering the largest harm-controlling set $\Lambda(\alpha)$, we will recover a harm-controlling set $\Lambda'(\alpha) \subseteq \Lambda(\alpha)$.
However, since the expression inside the expectation of the upper bound is nonmonotone with respect to $\lambda$, we cannot directly use it to set the value of the loss $\ell$ in conformal risk control.
That said, since the first term of the expression is nonincreasing and the second term is nondecreasing, as shown in Lemmas~\ref{lem:non-increasing} and~\ref{lem:non-decreasing} in Appendix~\ref{app:lemmas},
we can apply conformal risk control separately for each term.
For the first term, we use Theorem~1 because the term matches the counterfactual harm under the counterfactual monotonicity assumption.
For the second term, we prove the following theorem:
\begin{theorem}
    \label{thm:conformal-cf-benefit-control-cI}
    Let $\Dcal = \{(X_i, Y_i, \hat{Y}_{i})\}_{i=1}^{n}$ be a calibration set, with $(X_i, Y_i, \hat{Y}_{i})\sim P^{\Mcal}$, $\alpha \in \left [\frac{1}{n+1},1 \right ]$ be a user-specified bound, and 
    \begin{equation} \label{eq:check-lambda}
        \check{\lambda}(\alpha) = \sup \left \{ \lambda \,:\, \frac{n}{n+1} \hat{G}_{n}(\lambda) + \frac{1}{n+1} \leq \alpha \right \} \,\, \text{where} \,\, \hat{G}_{n}(\lambda) = \frac{\sum_{i=1}^{n} \mathbbm{1}\{\hat{Y}_i \neq Y_i \wedge Y_i \in \Ccal_{\lambda}(X_i)\}}{n}.
    \end{equation}
    If interventional monotonicity holds and $\check{\lambda}$ exists, a test sample $(X_{n+1}, Y_{n+1}, \hat Y_{n+1}) \sim P^{\Mcal}(X,Y,\hat Y)$ satisfies that
    \begin{align*}
        \EE \left [\mathbbm{1}\{\hat{Y}_{n+1} \neq Y_{n+1} \wedge Y_{n+1} \in \Ccal_{\check{\lambda}(\alpha)}(X_{n+1})\} \right]  \leq \alpha,
    \end{align*}
    where the expectation is over the randomness in the calibration set used to compute the threshold $\check{\lambda}(\alpha)$ and the test sample $(X_{n+1}, Y_{n+1}, \hat Y_{n+1})$.
\end{theorem}
Finally, we leverage the above theorems and the fact that the first term inside the expectation of the upper bound of counterfactual harm in Eq.~\ref{eq:harm-bound-interventional-monotonicity} is nonincreasing and the second term is nondecreasing with respect to $\lambda$ to recover a harm-controlling set $\Lambda'(\alpha) \subseteq \Lambda(\alpha)$:
%
% as formalized by the following corollary:
%
\begin{corollary}\label{cor:interventional}
Let $\alpha \in [0,1]$ be a user-specified bound. Then, under the interventional monotonicity assumption, for any choice of $\alpha' \leq \alpha$, the set 
\begin{equation*}
    \Lambda'(\alpha) = \{ \lambda \in [0,1] \given \hat{\lambda}(\alpha') \leq \lambda \leq \check{\lambda}(\alpha-\alpha') \},
\end{equation*}
where $\hat{\lambda}(\alpha')$ is given by Eq.~\ref{eq:hat-lambda} and $\check{\lambda}(\alpha-\alpha')$ is given by Eq.~\ref{eq:check-lambda} satisfies that $\Lambda'(\alpha) \subseteq \Lambda(\alpha)$.
\end{corollary}
Note that, in practice, the above corollaries can be implemented efficiently since, by definition, the functions $\hat H_{n}(\lambda)$ and $\hat{G}_{n}(\lambda)$ are piecewise constant functions of $\lambda$ with $n$ different pieces.

\vspace{-2mm}
\section{Experiments}
\vspace{-2mm}
\label{sec:experiments}
In this section, we use data from two different human subject studies to: a) evaluate the average counterfactual harm caused by decision support systems based on prediction sets; b) validate the theoretical guarantees offered by our computational framework (\ie, Corollaries~\ref{cor:counterfactual} and~\ref{cor:interventional});
c) investigate the trade-off between the average counterfactual harm caused by decision support systems based on prediction sets and the average accuracy achieved by human experts using these systems.
In what follows, we assume that counterfactual monotonicity holds. 
In Appendix~\ref{app:additional-results}, we conduct experiments where we relax this assumption and assume instead that interventional monotonicity~holds.\footnote{All experiments ran on a Mac OS machine with an M1 processor and 16GB Memory. An open-source implementation of our methodology is publicly available at \href{https://github.com/Networks-Learning/controlling-counterfactual-harm-prediction-sets}{https://github.com/Networks-Learning/controlling-counterfactual-harm-prediction-sets}.}
% We provide the code as supplementary material and will publicly release it with the final version of the paper.}
%
%
\begin{figure}[t]
    \centering
    \subfloat[$\alpha = 0.01$]{
    \label{fig:class-alpha0.01}
    \includegraphics[width=.45\linewidth, valign=t]{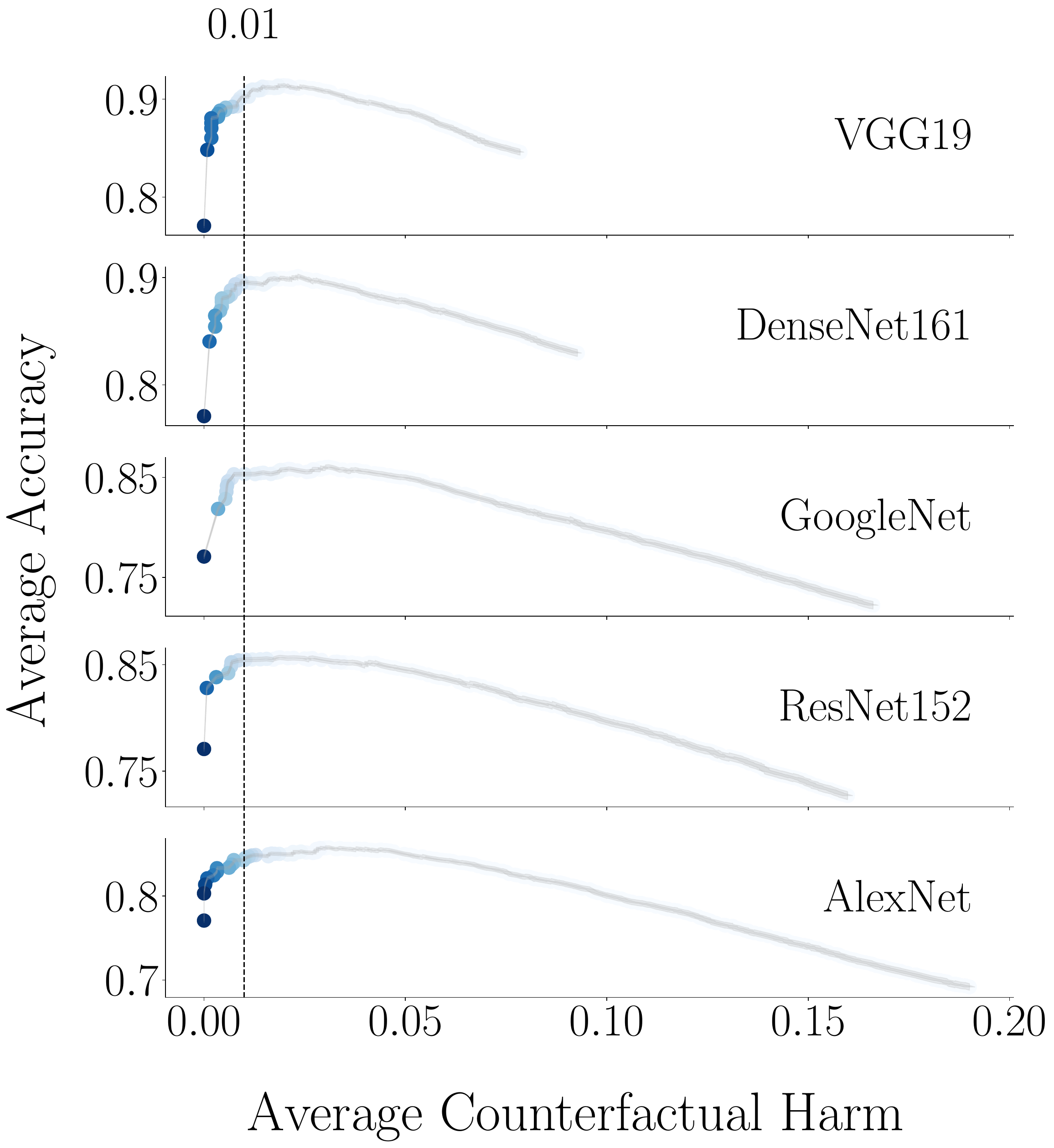}
    }
    \hspace{-0.5mm}
    \subfloat[$\alpha=0.05$]{
    \label{fig:class-alpha0.05}
    \includegraphics[width=.45\linewidth, valign=t]{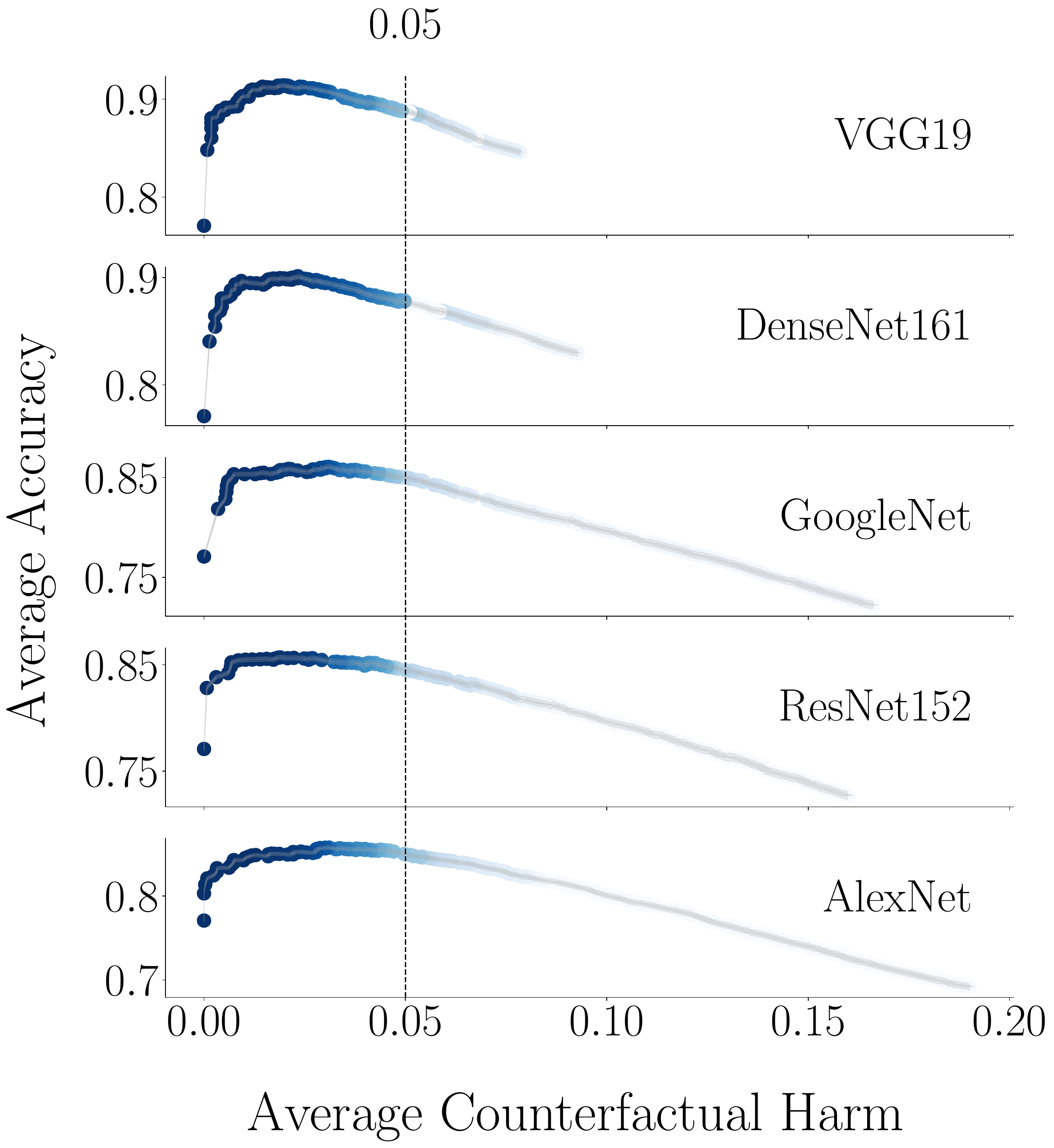}}
    \includegraphics[scale=0.2, valign=t]{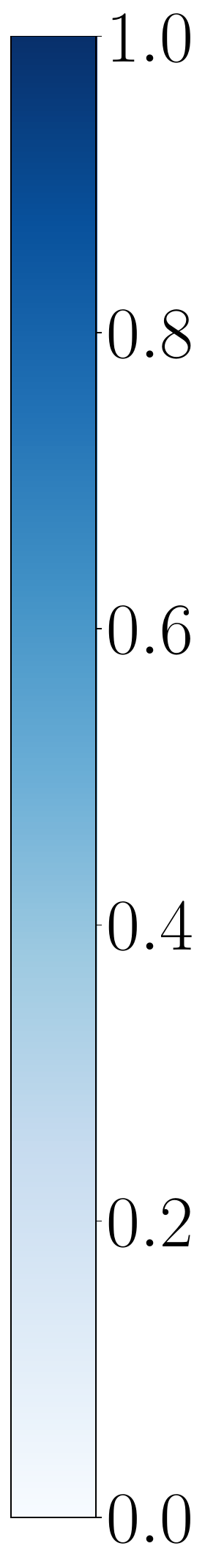}
    \caption{Average accuracy estimated by the mixture of MNLs against the average counterfactual harm for images with $\omega = 110$.
    Each point corresponds to a $\lambda$ value from $0$ to $1$ with step $0.001$ and the coloring indicates the relative frequency with which each $\lambda$ value is in $\Lambda(\alpha)$ across random samplings of the calibration set.
    Each row corresponds to decision support systems $\Ccal_{\lambda}$ with a different pre-trained classifier with average accuracies   \vgg~(VGG19),  \densenet~(DenseNet),  \googlenet~(GoogleNet), \resnet~(ResNet152), and \alexnet~(AlexNet).
    The average accuracy achieved by the simulated human experts on their own is $0.771$.
    The results are averaged across $50$ random samplings of the test and calibration set.
    In both panels, $95\%$ confidence intervals are represented using shaded areas and always have width below $0.02$.
    }
    \label{fig:acc_vs_harm_mnl}
    % \vspace{-3mm}
\end{figure}

\xhdr{Experimental setup}
We first experiment with the ImageNet16H dataset by Steyvers et al.~\cite{steyvers2022bayesian}, 
which comprises $32{,}431$ predictions made by $145$ human participants on their own about noisy images created using $1{,}200$ unique natural images from the ImageNet Large Scale Visual Recognition Challenge (ILSRVR) 2012 dataset~\cite{russakovsky2015imagenet}. 
More specifically, each of the natural images was used to generate four noisy images with different amount of phase noise distortion $\omega \in \{80, 95, 110, 125\}$ and with the same ground-truth label $y$ from a label set $\Ycal$ of size $L=16$.
Here, the amount of phase noise controls the difficulty of the classification task---the higher the noise, the more difficult the classification task.   
In our experiments, we use (all) the noisy images with noise value $\omega \in \{80, 95, 110\}$ because, for the noisy images with $\omega=125$, humans perform poorly;
% well on their own at solving the classification task, \ie, their average accuracy is at 
% least $0{.}771$;
% 
moreover, we stratify these images (and human predictions) with respect to their amount of phase noise.

For each stratum of images, we apply our framework to decision support systems $\Ccal_{\lambda}$ with different pre-trained classifiers, namely VGG19~\cite{simonyan2014very}, Dense\-Net161~\cite{huang2017densely}, GoogleNet~\cite{szegedy2015going}, ResNet152~\cite{he2016deep} and AlexNet~\cite{krizhevsky2012imagenet} after 10 epochs of fine-tuning, as provided by Steyvers et al.~\cite{steyvers2022bayesian}.\footnote{All classifiers and images are publicly available at \href{https://osf.io/2ntrf}{https://osf.io/2ntrf}.}
To this end, we randomly split the images (and human predictions) 
% in each strata 
into a calibration set ($10$\%), which we use to find the harm-controlling sets $\Lambda(\alpha)$ by applying Corollary~\ref{cor:counterfactual}, 
and a test set ($90$\%), which we use to estimate the average counterfactual harm $H(\lambda)$ caused by the decision support systems $\Ccal_{\lambda}$ as well as the average accuracy $A(\lambda)$ of the predictions made by a human expert using $\Ccal_{\lambda}$.
Here, since the dataset only contains predictions made by human participants on their own, we use the mixture of multinomial logit models (MNLs) introduced by Straitouri et al.~\cite{straitouri2023improving} to estimate the average accuracy $A(\lambda)$.
Refer to Appendix~\ref{app:additional-details} for more details regarding the mixture of MNLs and to Appendix~\ref{app:additional-results-counterfactual} for additional experiments studying the relationship between average counterfactual harm $H(\lambda)$, average prediction set size and empirical coverage\footnote{By empirical coverage, we refer to the fraction of the test samples for which the prediction sets include the ground truth label value.}. 

Then, we experiment with the ImageNet16H-PS dataset\footnote{The dataset is publicly available at \href{https://github.com/Networks-Learning/counterfactual-prediction-sets/}{https://github.com/Networks-Learning/counterfactual-prediction-sets/}.} by Straitouri et al.~\cite{straitouri2024designing}, which comprises $194{,}407$ predictions made by $2{,}751$ human participants using decision support systems $\Ccal_{\lambda}$ about the set of noisy images with $\omega = 110$ described above.
More specifically, for each noisy image, the dataset contains human predictions made under any possible prediction set that can be constructed using Eq.~\ref{eq:set-valued-predictor} with (the softmax output of) VGG19 after 10 epochs of fine-tuning, as provided by Steyvers et al.~\cite{steyvers2022bayesian}.
Here, similarly as in the ImageNet16H dataset, we randomly split the images (and human predictions) into a calibration set ($10$\%), which we use to find $\Lambda(\alpha)$, and a test set ($90$\%), which we use to estimate $H(\lambda)$ and $A(\lambda)$. 
However, in this case, we can estimate $A(\lambda)$ using the predictions made by human participants using $\Ccal_{\lambda}$ from the dataset, and we can compare this empirical estimate to the one using the mixture of MNLs.

In both datasets, we calculate confidence intervals and validate the theoretical guarantees offered by Corollary~\ref{cor:counterfactual} by repeating each experiment $50$ times and, each time, sampling different calibration and test sets.

\begin{figure}[t]
    \centering
    \includegraphics[width=.48\linewidth]{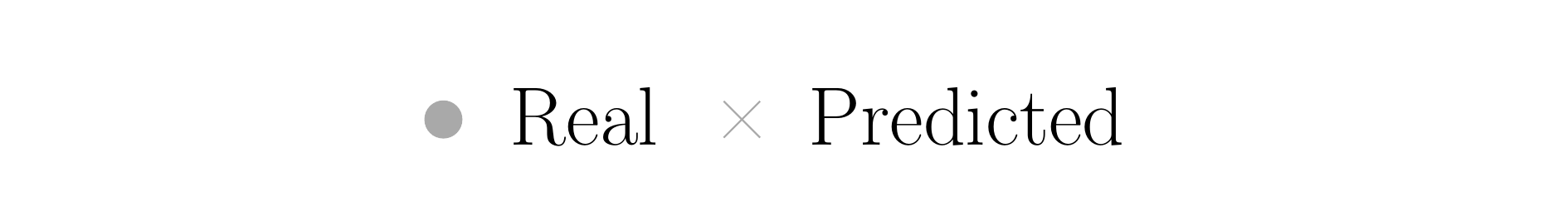} \vspace{-5mm} \\
    \subfloat[$\alpha = 0.01$]{
    \label{fig:1class-alpha0.01}
    \includegraphics[width=.46\linewidth, valign=t]{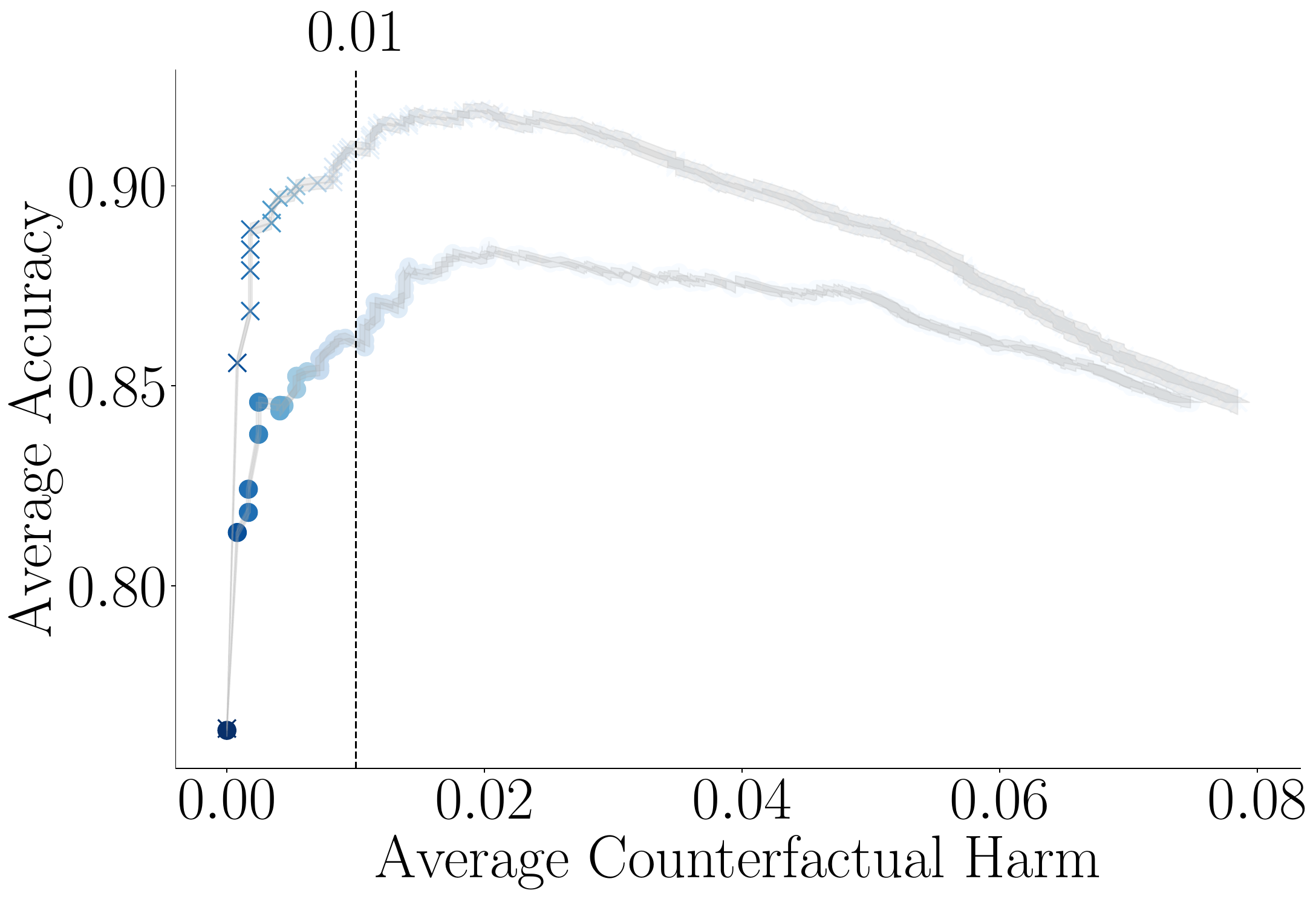}
    }
    % \hspace{-1mm}
    \subfloat[$\alpha=0.05$]{
    \label{fig:1class-alpha0.05}
    \hspace{-1mm}
    \includegraphics[width=.46\linewidth, valign=t]{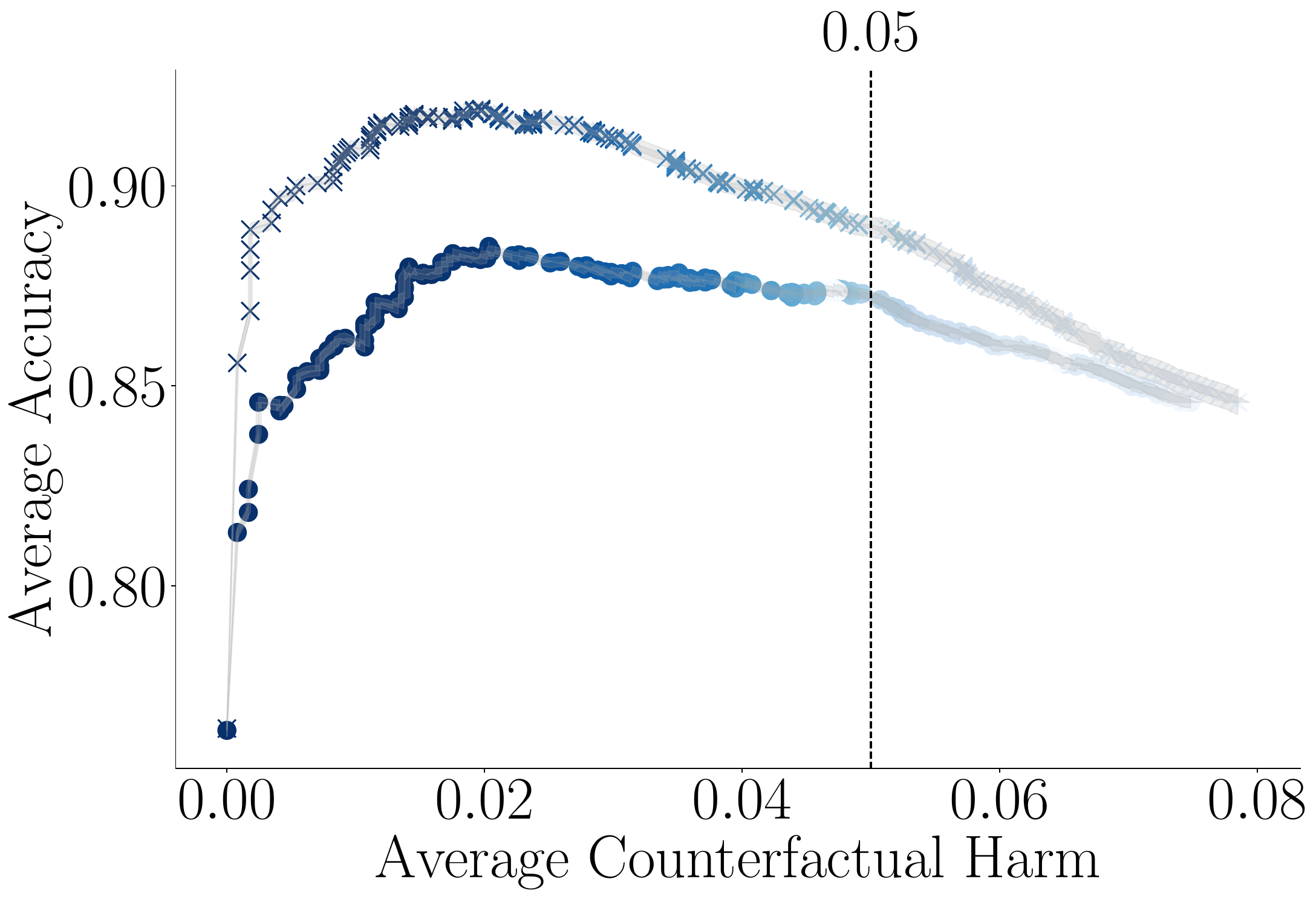}}
    \includegraphics[scale=0.181, valign=t]{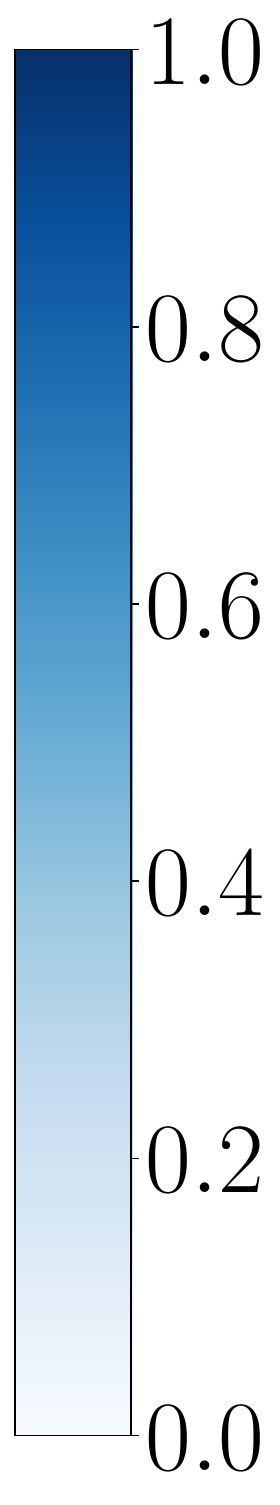}
    % \hspace{-1mm}
    \caption{Average accuracy estimated using predictions by human participants (Real) and using the mixture of MNLs (Predicted) against the average counterfactual harm for images with $\omega=110$.
    Each point corresponds to a $\lambda$ value from $0$ to $1$  with step $0.001$ and the coloring indicates the relative frequency with which the $\lambda$ value is in $\Lambda(\alpha)$ across random samplings of the calibration set.
    The decision support systems $\Ccal_{\lambda}$ use the pre-trained classifier VGG19.
    The results are averaged  across $50$ random samplings of the test and calibration set.
    In both panels, $95\%$ confidence intervals are represented using shaded areas and always have width below $0.02$.
    }
      \label{fig:acc_vs_harm}
    % \vspace{-3mm}
\end{figure}

\xhdr{Results} Figure~\ref{fig:acc_vs_harm_mnl} shows the average accuracy $A(\lambda)$ achieved by a human participant using $\Ccal_{\lambda}$, as predicted by the mixture of MNLs, against the average counterfactual harm $H(\lambda)$ caused by $\Ccal_{\lambda}$ for $\alpha \in \{0.01, 0.05\}$ on the stratum of images with $\omega = 110$ from the ImageNet16H dataset. 
Refer to Appendix~\ref{app:additional-results-counterfactual} for results on other strata. 
Here, each point corresponds to a different $\lambda$ value and its coloring indicates the empirical probability that $\lambda$ is included in the harm-controlling set $\Lambda(\alpha)$.
The results show several interesting insights.
First, we find that, as long as $\lambda < 1$, the decision support systems $\Ccal_{\lambda}$ always cause some amount of counterfactual harm\footnote{For $\lambda = 1$, the decision support system $\Ccal_{\lambda}$ does not cause harm because the prediction sets always contain all label values and thus human participants always make predictions on their own.}.
This suggests that, while restricting human agency enables human-AI complementarity, it inevitably causes (some amount of) counterfactual harm.
Second, we find that the sets of $\lambda$ values provided by our framework are typically harm-controlling, \ie, they do not include $\lambda$ values such that $H(\lambda) > \alpha$. 
However, the sets are often conservative and do not include \emph{all} the harm-controlling $\lambda$ 
values due to estimation error in the 
the empirical estimate $\hat{H}_n(\lambda)$ of the average counterfactual harm using data from the calibration set. 
In Appendix~\ref{app:additional-results-counterfactual}, we show that using larger calibration sets reduces the above mentioned estimation error and results in sets of $\lambda$ values that are less conservative.
Third, we find that there is a trade-off between accuracy and counterfactual harm and this trade-off is qualitatively consistent across decision support systems using different pre-trained classifiers.
In fact, for $\alpha = 0.01$, the decision support systems $\Ccal_{\lambda^{*}}$ offering the greatest average accuracy $A(\lambda^{*})$ are not harm-controlling.
% \footnote{Note that, using the decision support system 
% $A(\lambda^{*})$, the (simulated) human experts always 
% achieve equal or higher average accuracy than they 
% achieve on their own, even when using pre-trained 
% classifiers that are less accurate than the simulated 
% human experts.}

Figure~\ref{fig:acc_vs_harm} shows the average accuracy $A(\lambda)$ achieved by a human participant using $\Ccal_{\lambda}$, 
as estimated using predictions made by human participants using $\Ccal_{\lambda}$, against the average counterfactual harm $H(\lambda)$ caused by $\Ccal_{\lambda}$ for $\alpha \in \{0.01, 0.05\}$ on the ImageNet16H-PS dataset.
Here, the meaning and coloring of each point is the same as in Figure~\ref{fig:acc_vs_harm_mnl}.
The results also support the findings derived from the experiments with the ImageNet16H dataset---the decision support systems $\Ccal_{\lambda}$ always cause some amount of harm, our framework succeeds at identifying harm-controlling sets, and there is a trade-off between accuracy and counterfactual harm.
Further, they also show that, 
while there is a gap between the average accuracy estimated using the mixture of MNLs and the average accuracy estimated using predictions made by human participants using $\Ccal_{\lambda}$, 
they follow the same trend and support the same qualitative conclusions.

\vspace{-2mm}
\section{Discussion and limitations}
\vspace{-2mm}
\label{sec:discussion}
In this section, we highlight several limitations of our work, discuss its broader impact, and propose avenues for future work.

\xhdr{Assumptions}
We have assumed that the data samples and the expert predictions are drawn i.i.d. from a fixed distribution 
and the calibration set contains samples with noiseless ground truth labels and expert label predictions. 
It would be very interesting to extend our framework to allow for distribution shifts and label noise.
Moreover, we have considered prediction sets constructed with a \textit{fixed} user-specified threshold value $\lambda$. 
In light of recent work by Gibbs et al.~\cite{gibbs2023conformal}, it would be 
interesting to extend our framework to allow for threshold values that depend on the data samples.  
Furthermore, in our definition of counterfactual harm we have treated all inaccurate predictions as equally harmful. 
Expanding our definition of harm to weigh different inaccurate predictions according to their consequences would be a very interesting avenue for future work.
% however, it may be challenging to quantify/measure the consequences.
% 
In addition, under the interventional monotonicity assumption, we have~em\-pi\-ri\-cally observed that the gap between our lower- and upper-bounds is often large. 
Therefore, it would be useful to identify other natural assumptions under which this gap is smaller.

\xhdr{Methodology}
In our framework, the prediction sets are constructed using the softmax output of a pre-trained classifier. 
However, we hypothesize that, by accounting for the similarity between the mistakes made by humans and those made by the pre-trained classifier, we may be able to construct prediction sets that cause counterfactual harm less frequently.
Moreover, we have focused on con\-trolling the average counterfactual harm. However, whenever the expert's predictions are consequential to individuals, this may lead significant disparities across demographic groups.
Therefore, it would be important to extend our framework to account for fairness considerations.

\xhdr{Evaluation}
Our experimental evaluation comprises a single benchmark dataset of noisy natural images and thus one may question the generalizability of the conclusions we draw from our results.
To overcome this limitation, it would be important to further investigate the trade-off between accuracy and counterfactual harm in decision support systems based on prediction sets in real-world application domains (\eg, medical diagnosis).

\vspace{-2mm}
\section{Conclusions}
\vspace{-2mm}
\label{sec:conclusions}
In this paper, we have initiated the study of counterfactual harm in decision support systems based on prediction sets. 
We have introduced a computational framework that, 
under natural monotonicity assumptions on the predictions made by experts using these systems, 
can control how frequently these systems cause harm.
Moreover, we have validated our framework using data from two different human subject studies and shown that, in decision support systems based on prediction sets, there is a trade-off between accuracy and counterfactual harm.

\xhdr{Acknowledgements} Gomez-Rodriguez acknowledges support from the European Research Council (ERC) under the European Union'{}s Horizon 2020 research and innovation programme (grant agreement No. 945719).

{ 
\bibliographystyle{unsrt}
\bibliography{harm_controlling_sets}

\begin{thebibliography}{10}

\bibitem{richens2022counterfactual}
Jonathan Richens, Rory Beard, and Daniel~H Thompson.
\newblock Counterfactual harm.
\newblock In {\em Advances in Neural Information Processing Systems},
  volume~35, pages 36350--36365, 2022.

\bibitem{beckers2022causal}
Sander Beckers, Hana Chockler, and Joseph Halpern.
\newblock A causal analysis of harm.
\newblock In {\em Advances in Neural Information Processing Systems},
  volume~35, pages 2365--2376, 2022.

\bibitem{li2023trustworthy}
Haoxuan Li, Chunyuan Zheng, Yixiao Cao, Zhi Geng, Yue Liu, and Peng Wu.
\newblock Trustworthy policy learning under the counterfactual no-harm
  criterion.
\newblock In {\em International Conference on Machine Learning}, pages
  20575--20598. PMLR, 2023.

\bibitem{beckers2023quantifying}
Sander Beckers, Hana Chockler, and Joseph~Y Halpern.
\newblock Quantifying harm.
\newblock In {\em Proceedings of the International Joint Conference on
  Artificial Intelligence}. IJCAI, 2023.

\bibitem{feinberg1986wrongful}
Joel Feinberg.
\newblock Wrongful life and the counterfactual element in harming.
\newblock {\em Social Philosophy and Policy}, 4(1):145--178, 1986.

\bibitem{hanser2008metaphysics}
Matthew Hanser.
\newblock The metaphysics of harm.
\newblock {\em Philosophy and Phenomenological Research}, 77(2):421--450, 2008.

\bibitem{klocksiem2012defense}
Justin Klocksiem.
\newblock A defense of the counterfactual comparative account of harm.
\newblock {\em American Philosophical Quarterly}, 49(4):285--300, 2012.

\bibitem{bansal2019beyond}
Gagan Bansal, Besmira Nushi, Ece Kamar, Walter~S. Lasecki, Daniel~S. Weld, and
  Eric Horvitz.
\newblock Beyond accuracy: The role of mental models in human-ai team
  performance.
\newblock In {\em Proceedings of the AAAI Conference on Human Computation and
  Crowdsourcing}. AAAI press, 2019.

\bibitem{lubars2019ask}
Brian Lubars and Chenhao Tan.
\newblock Ask not what ai can do, but what ai should do: Towards a framework of
  task delegability.
\newblock In {\em Advances in Neural Information Processing Systems},
  volume~32, pages 57--67, 2019.

\bibitem{bordt2020humans}
Sebastian Bordt and Ulrike von Luxburg.
\newblock When humans and machines make joint decisions: A non-symmetric bandit
  model.
\newblock {\em arXiv preprint arXiv:2007.04800}, 2020.

\bibitem{vodrahalli2022uncalibrated}
Kailas Vodrahalli, Tobias Gerstenberg, and James~Y Zou.
\newblock Uncalibrated models can improve human-ai collaboration.
\newblock In {\em Advances in Neural Information Processing Systems},
  volume~35, pages 4004--4016, 2022.

\bibitem{corvelo2024human}
Nina Corvelo~Benz and Manuel Rodriguez.
\newblock Human-aligned calibration for ai-assisted decision making.
\newblock In {\em Advances in Neural Information Processing Systems},
  volume~36, 2024.

\bibitem{yin2019understanding}
Ming Yin, Jennifer Wortman~Vaughan, and Hanna Wallach.
\newblock Understanding the effect of accuracy on trust in machine learning
  models.
\newblock In {\em Proceedings of the {CHI} conference on human factors in
  computing systems}, pages 1--12. ACM, 2019.

\bibitem{zhang2020effect}
Yunfeng Zhang, Q~Vera Liao, and Rachel~KE Bellamy.
\newblock Effect of confidence and explanation on accuracy and trust
  calibration in ai-assisted decision making.
\newblock In {\em Proceedings of the Conference on Fairness, Accountability,
  and Transparency}, pages 295--305. ACM, 2020.

\bibitem{suresh2020misplaced}
Harini Suresh, Natalie Lao, and Ilaria Liccardi.
\newblock Misplaced trust: Measuring the interference of machine learning in
  human decision-making.
\newblock In {\em Proceedings of the ACM Conference on Web Science}, pages
  315--324. ACM, 2020.

\bibitem{lai2021towards}
Vivian Lai, Chacha Chen, Alison Smith-Renner, Q~Vera Liao, and Chenhao Tan.
\newblock Towards a science of human-ai decision making: An overview of design
  space in empirical human-subject studies.
\newblock In {\em Proceedings of the 2023 ACM Conference on Fairness,
  Accountability, and Transparency}, pages 1369--1385, 2023.

\bibitem{straitouri2023improving}
Eleni Straitouri, Lequn Wang, Nastaran Okati, and Manuel~Gomez Rodriguez.
\newblock Improving expert predictions with conformal prediction.
\newblock In {\em Proceedings of the International Conference on Machine
  Learning}, pages 32633--32653. PMLR, 2023.

\bibitem{straitouri2024designing}
Eleni Straitouri and Manuel Gomez-Rodriguez.
\newblock Designing decision support systems using counterfactual prediction
  sets.
\newblock In {\em Proceedings of the International Conference on Machine
  Learning}. PMLR, 2024.

\bibitem{google2020skin}
Google Blogpost.
\newblock Improving access to skin disease information.
\newblock \url{https://health.google/health-research/imaging-and-diagnostics/},
  2020.

\bibitem{jain2021development}
Ayush Jain, David Way, Vishakha Gupta, Yi~Gao, Guilherme de~Oliveira~Marinho,
  Jay Hartford, Rory Sayres, Kimberly Kanada, Clara Eng, Kunal Nagpal, et~al.
\newblock Development and assessment of an artificial intelligence--based tool
  for skin condition diagnosis by primary care physicians and nurse
  practitioners in teledermatology practices.
\newblock {\em JAMA network open}, 4(4):e217249--e217249, 2021.

\bibitem{pearl2009causal}
Judea Pearl.
\newblock Causal inference in statistics: An overview.
\newblock {\em Statistics Surveys}, 3:96--146, 01 2009.

\bibitem{angelopoulos2024conformal}
Anastasios~Nikolas Angelopoulos, Stephen Bates, Adam Fisch, Lihua Lei, and Tal
  Schuster.
\newblock Conformal risk control.
\newblock In {\em Proceedings of the International Conference on Learning
  Representations}. ICLR, 2024.

\bibitem{chzhen2021set}
Evgenii Chzhen, Christophe Denis, Mohamed Hebiri, and Titouan Lorieul.
\newblock Set-valued classification--overview via a unified framework.
\newblock {\em arXiv preprint arXiv:2102.12318}, 2021.

\bibitem{yang2017cautious}
Gen Yang, S{\'e}bastien Destercke, and Marie-H{\'e}l{\`e}ne Masson.
\newblock Cautious classification with nested dichotomies and imprecise
  probabilities.
\newblock {\em Soft Computing}, 21(24):7447--7462, 2017.

\bibitem{mortier2021efficient}
Thomas Mortier, Marek Wydmuch, Krzysztof Dembczy{\'n}ski, Eyke H{\"u}llermeier,
  and Willem Waegeman.
\newblock Efficient set-valued prediction in multi-class classification.
\newblock {\em Data Mining and Knowledge Discovery}, 35(4):1435--1469, 2021.

\bibitem{ma2021partial}
Liyao Ma and Thierry Denoeux.
\newblock Partial classification in the belief function framework.
\newblock {\em Knowledge-Based Systems}, 214:106742, 2021.

\bibitem{nguyen2021multilabel}
Vu-Linh Nguyen and Eyke H{\"u}llermeier.
\newblock Multilabel classification with partial abstention: Bayes-optimal
  prediction under label independence.
\newblock {\em Journal of Artificial Intelligence Research}, 72:613--665, 2021.

\bibitem{angelopoulos2021learn}
Anastasios~N Angelopoulos, Stephen Bates, Emmanuel~J Cand{\`e}s, Michael~I
  Jordan, and Lihua Lei.
\newblock Learn then test: Calibrating predictive algorithms to achieve risk
  control.
\newblock {\em arXiv preprint arXiv:2110.01052}, 2021.

\bibitem{bates2021distribution}
Stephen Bates, Anastasios Angelopoulos, Lihua Lei, Jitendra Malik, and Michael
  Jordan.
\newblock Distribution-free, risk-controlling prediction sets.
\newblock {\em Journal of the ACM}, 68(6):1--34, 2021.

\bibitem{babbar2022on}
Varun Babbar, Umang Bhatt, and Adrian Weller.
\newblock On the utility of prediction sets in human-ai teams.
\newblock In {\em Proceedings of the International Joint Conference on
  Artificial Intelligence}, pages 2457--2463. IJCAI, 7 2022.

\bibitem{cresswell2024conformal}
Jesse~C. Cresswell, Yi~Sui, Bhargava Kumar, and No{\"{e}}l Vouitsis.
\newblock Conformal prediction sets improve human decision making.
\newblock In {\em Forty-first International Conference on Machine Learning,
  {ICML} 2024, Vienna, Austria, July 21-27, 2024}. OpenReview.net, 2024.

\bibitem{zhang2024evaluating}
Dongping Zhang, Angelos Chatzimparmpas, Negar Kamali, and Jessica Hullman.
\newblock Evaluating the utility of conformal prediction sets for ai-advised
  image labeling.
\newblock In {\em Proceedings of the {CHI} conference on Human Factors in
  Computing Systems}. ACM, 2024.

\bibitem{de2024towards}
Giovanni De~Toni, Nastaran Okati, Suhas Thejaswi, Eleni Straitouri, and Manuel
  Gomez-Rodriguez.
\newblock Towards human-ai complementarity with predictions sets.
\newblock In {\em Avances on Neural Information Processing Systems}, 2024.

\bibitem{van2023accurate}
Wouter~AC van Amsterdam, Nan van Geloven, Jesse~H Krijthe, Rajesh Ranganath,
  and Giovanni Cin{\'a}.
\newblock When accurate prediction models yield harmful self-fulfilling
  prophecies.
\newblock {\em arXiv preprint arXiv:2312.01210}, 2023.

\bibitem{wang2024against}
Angelina Wang, Sayash Kapoor, Solon Barocas, and Arvind Narayanan.
\newblock Against predictive optimization: On the legitimacy of decision-making
  algorithms that optimize predictive accuracy.
\newblock {\em ACM Journal of Responsible Computing}, 1(1), March 2024.

\bibitem{liu2024actionability}
Lydia~T Liu, Solon Barocas, Jon Kleinberg, and Karen Levy.
\newblock On the actionability of outcome prediction.
\newblock In {\em Proceedings of the AAAI Conference on Artificial
  Intelligence}, volume~38, pages 22240--22249, 2024.

\bibitem{raghu2019algorithmic}
Maithra Raghu, Katy Blumer, Greg Corrado, Jon Kleinberg, Ziad Obermeyer, and
  Sendhil Mullainathan.
\newblock The algorithmic automation problem: Prediction, triage, and human
  effort.
\newblock {\em arXiv preprint arXiv:1903.12220}, 2019.

\bibitem{mozannar2020consistent}
Hussein Mozannar and David Sontag.
\newblock Consistent estimators for learning to defer to an expert.
\newblock In {\em International Conference on Machine Learning}, pages
  7076--7087. PMLR, 2020.

\bibitem{de2021classification}
Abir De, Nastaran Okati, Ali Zarezade, and Manuel~Gomez Rodriguez.
\newblock Classification under human assistance.
\newblock {\em Proceedings of the AAAI Conference on Artificial Intelligence},
  35(7):5905--5913, 2021.

\bibitem{okati2021differentiable}
Nastaran Okati, Abir De, and Manuel Rodriguez.
\newblock Differentiable learning under triage.
\newblock In {\em Advances in Neural Information Processing Systems},
  volume~34, pages 9140--9151, 2021.

\bibitem{charusaie2022sample}
Mohammad-Amin Charusaie, Hussein Mozannar, David Sontag, and Samira Samadi.
\newblock Sample efficient learning of predictors that complement humans.
\newblock In {\em International Conference on Machine Learning}, pages
  2972--3005. PMLR, 2022.

\bibitem{mozannar2023should}
Hussein Mozannar, Hunter Lang, Dennis Wei, Prasanna Sattigeri, Subhro Das, and
  David Sontag.
\newblock Who should predict? exact algorithms for learning to defer to humans.
\newblock In {\em International conference on artificial intelligence and
  statistics}, pages 10520--10545. PMLR, 2023.

\bibitem{straitouri2021reinforcement}
Eleni Straitouri, Adish Singla, Vahid~Balazadeh Meresht, and Manuel
  Gomez-Rodriguez.
\newblock Reinforcement learning under algorithmic triage.
\newblock {\em arXiv preprint arXiv:2109.11328}, 2021.

\bibitem{balazadeh2022learning}
Vahid Balazadeh, Abir De, Adish Singla, and Manuel~Gomez Rodriguez.
\newblock Learning to switch among agents in a team via 2-layer markov decision
  processes.
\newblock {\em Transactions on Machine Learning Research}, 2022.

\bibitem{fuchs2023optimizing}
Andrew Fuchs, Andrea Passarella, and Marco Conti.
\newblock Optimizing delegation between human and ai collaborative agents.
\newblock {\em arXiv preprint arXiv:2309.14718}, 2023.

\bibitem{tsirtsis2024responsibility}
Stratis Tsirtsis, Manuel~Gomez Rodriguez, and Tobias Gerstenberg.
\newblock Responsibility judgments in sequential human-{AI} collaboration.
\newblock In {\em Proceedings of the Annual Conference of the Cognitive Science
  Society}. cognitivesciencesociety.org, 2024.

\bibitem{romano2020aps}
Yaniv Romano, Matteo Sesia, and Emmanuel Candes.
\newblock Classification with valid and adaptive coverage.
\newblock In H.~Larochelle, M.~Ranzato, R.~Hadsell, M.F. Balcan, and H.~Lin,
  editors, {\em Advances in Neural Information Processing Systems}, volume~33,
  pages 3581--3591. Curran Associates, Inc., 2020.

\bibitem{angelopoulos2021raps}
Anastasios~Nikolas Angelopoulos, Stephen Bates, Michael~I. Jordan, and Jitendra
  Malik.
\newblock Uncertainty sets for image classifiers using conformal prediction.
\newblock In {\em 9th International Conference on Learning Representations,
  {ICLR} 2021, Virtual Event, Austria, May 3-7, 2021}. OpenReview.net, 2021.

\bibitem{huang2023conformal}
Jianguo Huang, Huajun Xi, Linjun Zhang, Huaxiu Yao, Yue Qiu, and Hongxin Wei.
\newblock Conformal prediction for deep classifier via label ranking.
\newblock In {\em Forty-first International Conference on Machine Learning,
  {ICML} 2024, Vienna, Austria, July 21-27, 2024}. OpenReview.net, 2024.

\bibitem{pearl2009causality}
J~Pearl.
\newblock {\em Causality}.
\newblock Cambridge university press, 2009.

\bibitem{peters2017elements}
Jonas Peters, Dominik Janzing, and Bernhard Sch{\"o}lkopf.
\newblock {\em Elements of causal inference: foundations and learning
  algorithms}.
\newblock The MIT Press, 2017.

\bibitem{steyvers2022bayesian}
Mark Steyvers, Heliodoro Tejeda, Gavin Kerrigan, and Padhraic Smyth.
\newblock Bayesian modeling of human--ai complementarity.
\newblock {\em Proceedings of the National Academy of Sciences},
  119(11):e2111547119, 2022.

\bibitem{russakovsky2015imagenet}
Olga Russakovsky, Jia Deng, Hao Su, Jonathan Krause, Sanjeev Satheesh, Sean Ma,
  Zhiheng Huang, Andrej Karpathy, Aditya Khosla, Michael Bernstein, et~al.
\newblock Imagenet large scale visual recognition challenge.
\newblock {\em International Journal of Computer Vision}, 115:211--252, 2015.

\bibitem{simonyan2014very}
Karen Simonyan and Andrew Zisserman.
\newblock Very deep convolutional networks for large-scale image recognition.
\newblock In Yoshua Bengio and Yann LeCun, editors, {\em 3rd International
  Conference on Learning Representations, {ICLR} 2015, San Diego, CA, USA, May
  7-9, 2015, Conference Track Proceedings}, 2015.

\bibitem{huang2017densely}
Gao Huang, Zhuang Liu, Laurens Van Der~Maaten, and Kilian~Q Weinberger.
\newblock Densely connected convolutional networks.
\newblock In {\em Proceedings of the IEEE conference on computer vision and
  pattern recognition}, pages 4700--4708. IEEE, 2017.

\bibitem{szegedy2015going}
Christian Szegedy, Wei Liu, Yangqing Jia, Pierre Sermanet, Scott Reed, Dragomir
  Anguelov, Dumitru Erhan, Vincent Vanhoucke, and Andrew Rabinovich.
\newblock Going deeper with convolutions.
\newblock In {\em Proceedings of the IEEE conference on computer vision and
  pattern recognition}, pages 1--9. IEEE, 2015.

\bibitem{he2016deep}
Kaiming He, Xiangyu Zhang, Shaoqing Ren, and Jian Sun.
\newblock Deep residual learning for image recognition.
\newblock In {\em Proceedings of the IEEE conference on computer vision and
  pattern recognition}, pages 770--778. IEEE, 2016.

\bibitem{krizhevsky2012imagenet}
Alex Krizhevsky, Ilya Sutskever, and Geoffrey~E Hinton.
\newblock Imagenet classification with deep convolutional neural networks.
\newblock In {\em Advances in neural information processing systems},
  volume~25, 2012.

\bibitem{gibbs2023conformal}
Isaac Gibbs, John~J Cherian, and Emmanuel~J Cand{\`e}s.
\newblock Conformal prediction with conditional guarantees.
\newblock {\em arXiv preprint arXiv:2305.12616}, 2023.

\bibitem{heiss2016discrete}
Florian Heiss.
\newblock Discrete choice methods with simulation.
\newblock {\em Econometric Reviews}, 35:1--5, 02 2016.

\bibitem{tversky1993context}
Amos Tversky and Itamar Simonson.
\newblock Context-dependent preferences.
\newblock {\em Management science}, 39(10):1179--1189, 1993.

\end{thebibliography}
}

\clearpage
\newpage

\appendix
% \section{Counterfactual Harm Control Algorithm}\label{app:algorithm}
% \input{notes/charm_control_cI}
% \input{neurips24/harm_control_cI}
% Alternative
% \section{Choice of Notation and definition by Richens~\cite{richens2022counterfactual}}
\section{Comparison to Pearl's notation}
\label{app:notation}

% In this section we first compare Peters' notation~\cite{peters2017elements}, which we adopt, against Pearl's notation~\cite{pearl2009causality} in the context of our definition of counterfactual harm. Next, we compare our definition of counterfactual harm against the definition used by Richens et al.~\cite{richens2022counterfactual}.

% \xhdr{Comparison to Pearl's notation} 
% \xhdr{Choice of notation}
Using Pearl's notation~\cite{pearl2009causality}, our definition of counterfactual harm reads as follows:
\begin{align*}
h_{\lambda}(x, y, \hat{y}) = \mathbb{E}[\max\{0, \mathbbm{1}\{\hat{y} = y \} - \mathbbm{1}\{\hat{Y}_{\lambda} = y\} \} | X=x, Y=y, \hat{Y}_{1}=\hat{y}],    
\end{align*}
where the subindices in $\hat{Y}_{\lambda}$ and $\hat{Y}_{1}$ denote (hard) interventions on the exogenous variable $\Lambda$, the random variables $x, y, \hat{y} \sim P(X, Y, \hat{Y}_{1})$, and the expectation is over the prediction $\hat{Y}_{\lambda}$ made by the expert using the system $\mathcal{C}_{\lambda}$. 
Using the above notation, the definition does not explicitly specify the (counterfactual) distribution used in the expectation, in contrast to the definition using the notation by Peters et al.~\cite{peters2017elements}, which explicitly specifies the counterfactual distribution used in the expectation, as restated below for clarity.
\begin{align*}
     h_{\lambda}(x, y, \hat{y}) = \EE_{\hat{Y} \sim P^{\Mcal ; \text{do}(\Lambda=\lambda) \given \hat{Y}=\hat{y}, X=x, Y=y}(\hat Y)}
        [\max\{0, \mathbbm{1}\{\hat{y} = y\} - \mathbbm{1}\{\hat{Y} = y \} \}].
\end{align*} 

\clearpage
\newpage

\section{Proofs}
\label{app:proofs}

\subsection{Proof of Proposition~\ref{prop:harm-counterfactual-monotonicity}}
Given an observation $(x, y, \hat{y})$ under no intervention, we have 
\begin{align}
     h_{\lambda}(x, y, \hat{y}) &= \EE_{\hat{Y} \sim P^{\Mcal ; \text{do}(\Lambda=\lambda) \given \hat{Y}=\hat{y}, X=x, Y=y}}
        [\max\{0, \mathbbm{1}\{\hat{y} = y\} - \mathbbm{1}\{\hat{Y} = y \} \}] \nonumber\\
    &=\EE_{\hat{Y} \sim P^{\Mcal ; \text{do}(\Lambda=\lambda) \given \hat{Y}=\hat{y}, X=x, Y=y}}
        [\max\{0, 1 - \mathbbm{1}\{\hat{Y} = y \}\}]\cdot\mathbbm{1}\{\hat{y} = y\}\nonumber\\
    &+\EE_{\hat{Y} \sim P^{\Mcal ; \text{do}(\Lambda=\lambda) \given \hat{Y}=\hat{y}, X=x, Y=y}}
        [\max\{0, 0 - \mathbbm{1}\{\hat{Y} = y \}\}]\cdot\mathbbm{1}\{\hat{y} \neq y\}\nonumber\\
    &=\EE_{\hat{Y} \sim P^{\Mcal ; \text{do}(\Lambda=\lambda) \given \hat{Y}=\hat{y}, X=x, Y=y}}
        [\max\{0, 1 - \mathbbm{1}\{\hat{Y} = y \}\}]\cdot\mathbbm{1}\{\hat{y} = y\} + 0\nonumber\\
    &=\EE_{\hat{Y} \sim P^{\Mcal ; \text{do}(\Lambda=\lambda) \given \hat{Y}=\hat{y}, X=x, Y=y}}
        [\max\{0, 1 - 1\}]\cdot\mathbbm{1}\{\hat{y} = y\}\cdot  \mathbbm{1}\{y \in \Ccal_{\lambda}(x)\} \nonumber\\
    &+\EE_{\hat{Y} \sim P^{\Mcal ; \text{do}(\Lambda=\lambda) \given \hat{Y}=\hat{y}, X=x, Y=y}}
        [\max\{0, 1 - 0\}]\cdot\mathbbm{1}\{\hat{y} = y\}\cdot \mathbbm{1}\{y \notin \Ccal_{\lambda}(x)\}~\label{eq:step-5}\\
    &=0\cdot\mathbbm{1}\{\hat{y} = y\}\cdot \mathbbm{1}\{y \in \Ccal_{\lambda}(x)\}\nonumber \\
    &+\EE_{\hat{Y} \sim P^{\Mcal ; \text{do}(\Lambda=\lambda) \given \hat{Y}=\hat{y}, X=x, Y=y}}
        [\max\{0, 1 - 0\}]\cdot\mathbbm{1}\{\hat{y} = y\}\cdot \mathbbm{1}\{y \notin \Ccal_{\lambda}(x)\}\nonumber\\
    &= 1 \cdot \mathbbm{1}\{\hat{y} = y\}\cdot \mathbbm{1}\{y \notin \Ccal_{\lambda}(x)\}\nonumber\\
    &=  \mathbbm{1}\{\hat{y} = y\}\cdot \mathbbm{1}\{y \notin \Ccal_{\lambda}(x)\}\nonumber\\
    &= \mathbbm{1}\{\hat{y} = y \wedge y \notin \Ccal_{\lambda}(x)\}\nonumber,
\end{align}
where, in Eq.~\ref{eq:step-5}, we used the counterfactual monotonicity property by setting $\mathbbm{1}\{\hat{Y} = y\} = 1$ if $y \in \Ccal_{\lambda}(x)$.
% \hfill \QED

\subsection{Proof of Proposition~\ref{prop:harm-bound-interventional-monotonicity}}

Under interventional monotonicity, by Lemma~\ref{lem:harm-bound-interventional-monotonicity}, we have
\begin{align*}
    h_{\lambda}(x, y, \hat{y}) &= \EE_{\hat{Y} \sim P^{\Mcal ; \text{do}(\Lambda=\lambda) \given \hat{Y}=\hat{y}, X=x, Y=y}}
        [ \mathbbm{1}\{\hat{Y} \neq y \}] \cdot \mathbbm{1}\{\hat{y} = y\} \cdot \mathbbm{1}\{y \in \Ccal_{\lambda}(x)\} \\& + \mathbbm{1}\{\hat{y} = y\} \cdot \mathbbm{1}\{y \notin \Ccal_{\lambda}(x)\}\\& \leq 
        \EE_{\hat{Y} \sim P^{\Mcal ; \text{do}(\Lambda=\lambda) \given \hat{Y}=\hat{y}, X=x, Y=y}}
        [ \mathbbm{1}\{\hat{Y} \neq y \}] \cdot \mathbbm{1}\{y \in \Ccal_{\lambda}(x)\} \\& + \mathbbm{1}\{\hat{y} = y\} \cdot \mathbbm{1}\{y \notin \Ccal_{\lambda}(x)\},
\end{align*}
where we used that  $\mathbbm{1}\{\hat{y} = y\} \leq 1$.

By taking the expectation on both sides over $\hat{y}$ under no intervention we have 
\begin{align}\label{eq:exp-hat-y}
    \EE_{\hat{Y}' \sim P^{\Mcal}}\left [h_{\lambda}(x, y, \hat{Y}')\right ] & \leq 
       \EE_{\hat{Y}' \sim P^{\Mcal}}\left[ \EE_{\hat{Y} \sim P^{\Mcal ; \text{do}(\Lambda=\lambda) \given \hat{Y}=\hat{Y}', X=x, Y=y}} 
        \left[ \mathbbm{1}\{\hat{Y} \neq y \} \right] 
        \cdot \mathbbm{1}\{y \in \Ccal_{\lambda}(x)\} \right] \\& + \EE_{\hat{Y}' \sim P^{\Mcal}}\left [\mathbbm{1}\{\hat{Y}' = y\} \cdot \mathbbm{1}\{y \notin \Ccal_{\lambda}(x)\} \right] \\
        &= \EE_{\hat{Y}' \sim P^{\Mcal}}\left[ \EE_{\hat{Y} \sim P^{\Mcal ; \text{do}(\Lambda=\lambda) \given \hat{Y}=\hat{Y}', X=x, Y=y}} 
        \left[ \mathbbm{1}\{\hat{Y} \neq y \} \right] \right]
        \cdot \mathbbm{1}\{y \in \Ccal_{\lambda}(x)\} \\
        &+ \EE_{\hat{Y}' \sim P^{\Mcal}}\left [\mathbbm{1}\{\hat{Y}' = y\} \cdot \mathbbm{1}\{y \notin \Ccal_{\lambda}(x)\} \right],
\end{align}
since $\mathbbm{1}\{y \in \Ccal_{\lambda}(x)\}$ is a constant with respect to the expectation over $\hat{Y}' \sim P^{\Mcal}$. 
% We will now use Lemma~\ref{lem:interventional-counterfactual}  

By Lemma~\ref{lem:interventional-counterfactual}, Eq.~\ref{eq:exp-hat-y} becomes
\begin{align}\label{eq:harm-interv-lambda}
    \EE_{\hat{Y}' \sim P^{\Mcal}}\left [h_{\lambda}(x, y, \hat{Y}')\right ] & \leq 
        \EE_{\hat{Y} \sim P^{\Mcal;\text{do}(\Lambda=\lambda)}}  
        \left[ \mathbbm{1}\{\hat{Y} \neq Y \} \given X=x, Y = y  \right]
        \cdot \mathbbm{1}\{y \in \Ccal_{\lambda}(x)\} \nonumber\\
        &+ \EE_{\hat{Y}' \sim P^{\Mcal}}\left [\mathbbm{1}\{\hat{Y}' = y\} \cdot \mathbbm{1}\{y \notin \Ccal_{\lambda}(x)\} \right].
\end{align}
By Assumption~\ref{asm:cond-interventional-monotonicity} for $\lambda' = 1$ and any $\lambda$ such that $y \in \Ccal_{\lambda}(x)$, we have that
\begin{equation*}
    P^{\Mcal ; \text{do}(\Lambda = \lambda)}(\hat{Y} = Y \given X = x, Y=y) \geq P^{\Mcal}(\hat{Y} = Y \given X = x ,Y=y),
\end{equation*}
which we can rewrite as 
\begin{equation*}
    \EE_{\hat{Y}\sim P^{\Mcal ; \text{do}(\Lambda = \lambda)}}[\mathbbm{1}\{\hat{Y} = Y\} \given X = x, Y=y] \geq \EE_{\hat{Y}\sim P^{\Mcal}}[\mathbbm{1}\{\hat{Y} = Y \}\given X = x ,Y=y],
\end{equation*}
and is equivalent to 
\begin{equation*}
    \EE_{\hat{Y}\sim P^{\Mcal ; \text{do}(\Lambda = \lambda)}}[\mathbbm{1}\{\hat{Y} \neq Y\} \given X = x, Y=y] \leq \EE_{\hat{Y}\sim P^{\Mcal}}[\mathbbm{1}\{\hat{Y} \neq Y \}\given X = x ,Y=y].
\end{equation*}
Using the above result, Eq.~\ref{eq:harm-interv-lambda} becomes
\begin{align*}
    \EE_{\hat{Y}' \sim P^{\Mcal}}\left [h_{\lambda}(x, y, \hat{Y}')\right ] & \leq 
        \EE_{\hat{Y} \sim P^{\Mcal} } 
        \left[ \mathbbm{1}\{\hat{Y} \neq Y \} \given X=x, Y = y  \right]
        \cdot \mathbbm{1}\{y \in \Ccal_{\lambda}(x)\} \\
        &+ \EE_{\hat{Y}' \sim P^{\Mcal}}\left [\mathbbm{1}\{\hat{Y}' = y\} \cdot \mathbbm{1}\{y \notin \Ccal_{\lambda}(x)\} \right].
\end{align*}

If we now take the expectation in both sides over $X, Y \sim P^{\Mcal}$ and use that $\hat{Y}'$ is equal in distribution with $\hat{Y}$, we have
 \begin{multline*}
 \EE_{X, Y, \hat{Y} \sim P^{\Mcal}}\left [h_{\lambda}(X, Y, \hat{Y}) \right]  \leq   
 \EE_{X, Y, \hat{Y} \sim P^{\Mcal}} \left [\mathbbm{1}\{\hat{Y} \neq Y \} \cdot \mathbbm{1}\{Y \in \Ccal_{\lambda}(X)\} \right ]\\ + \EE_{X, Y, \hat{Y} \sim P^{\Mcal}} \left [\mathbbm{1}\{\hat{Y} = Y \} \cdot \mathbbm{1}\{Y \notin \Ccal_{\lambda}(X)\} \right ],
 \end{multline*}
 that is equivalent to 
 \begin{multline*}
 \EE_{X, Y, \hat{Y} \sim P^{\Mcal}}\left [h_{\lambda}(X, Y, \hat{Y}) \right] \\ \leq 
 \EE_{X, Y, \hat{Y} \sim P^{\Mcal}} \left [\mathbbm{1}\{\hat{Y} \neq Y \} \cdot \mathbbm{1}\{Y \in \Ccal_{\lambda}(X)\}  +\mathbbm{1}\{\hat{Y} = Y \} \cdot \mathbbm{1}\{Y \notin \Ccal_{\lambda}(X)\} \right ],
 \end{multline*}
which we can write as 
\begin{multline*}
 \EE_{X, Y, \hat{Y} \sim P^{\Mcal}}\left [h_{\lambda}(X, Y, \hat{Y}) \right] \\ \leq 
 \EE_{X, Y, \hat{Y} \sim P^{\Mcal}} \left [\mathbbm{1}\{\hat{Y} \neq Y \wedge Y \in \Ccal_{\lambda}(X)\}  +\mathbbm{1}\{\hat{Y} = Y \wedge Y \notin \Ccal_{\lambda}(X)\} \right ].
 \end{multline*}
 For the lower-bound, by Lemma~\ref{lem:harm-bound-interventional-monotonicity}, we have 
 \begin{align*}
    h_{\lambda}(x, y, \hat{y}) &= \EE_{\hat{Y} \sim P^{\Mcal ; \text{do}(\Lambda=\lambda) \given \hat{Y}=\hat{y}, X=x, Y=y}}
        [ \mathbbm{1}\{\hat{Y} \neq y \}] \cdot \mathbbm{1}\{\hat{y} = y\} \cdot \mathbbm{1}\{y \in \Ccal_{\lambda}(x)\} \\& + \mathbbm{1}\{\hat{y} = y\} \cdot \mathbbm{1}\{y \notin \Ccal_{\lambda}(x)\}\\ &\geq 
         \mathbbm{1}\{\hat{y} = y\} \cdot \mathbbm{1}\{y \notin \Ccal_{\lambda}(x)\},
\end{align*}
where we used that  $ \EE_{\hat{Y} \sim P^{\Mcal ; \text{do}(\Lambda=\lambda) \given \hat{Y}=\hat{y}, X=x, Y=y}}
        [ \mathbbm{1}\{\hat{Y} \neq y \}] \cdot \mathbbm{1}\{\hat{y} = y\}\cdot \mathbbm{1}\{y \in \Ccal_{\lambda}(x)\} \geq 0$.
By taking the expectation in both sides over $X, Y, \hat{Y} \sim P^{\Mcal}$, we have
\begin{align*}
     \EE_{X, Y, \hat{Y} \sim P^{\Mcal}}\left [\mathbbm{1}\{\hat{Y} = Y\} \cdot \mathbbm{1}\{Y \notin \Ccal_{\lambda}(X)\}\right] \leq \EE_{X, Y, \hat{Y} \sim P^{\Mcal}}\left [h_{\lambda}(X, Y, \hat{Y}) \right] , 
\end{align*}
that is equivalent to 
\begin{align*}
     \EE_{X, Y, \hat{Y} \sim P^{\Mcal}}\left [\mathbbm{1}\{\hat{Y} = Y \wedge Y \notin \Ccal_{\lambda}(X)\}\right] \leq \EE_{X, Y, \hat{Y} \sim P^{\Mcal}}\left [h_{\lambda}(X, Y, \hat{Y}) \right] , 
\end{align*}
thus we conclude the proof.
% where if we set $h^{\text{CF}(X, Y, \hat{Y})} =\mathbbm{1}\{\hat{Y} \neq Y \} \cdot \mathbbm{1}\{Y \in \Ccal_{\lambda}(X)\}$ and $g_{\lambda}(X, Y, \hat{Y}) = \mathbbm{1}\{\hat{Y} = Y \} \cdot \mathbbm{1}\{Y \notin \Ccal_{\lambda}(X)\} $ we conclude the proof.
%
% \hfill \QED

\subsection{Proof of Theorem~\ref{thm:harm-control-counterfactual-mon}}
% 
% For the proof we will use the following:
% 
We will show that $\mathbbm{1}\{\hat{Y}_i = Y_i \wedge Y_i \notin \Ccal_{\lambda}(X_i)\}$ are exchangeable and that they satisfy the requirements of Theorem~1 in~\citep{angelopoulos2024conformal} for each $X_i, Y_i, \hat{Y}_i$. 
Since $X_i,Y_i, \hat{Y}_i$ are i.i.d., $\mathbbm{1}\{\hat{Y}_i = Y_i \wedge Y_i \notin \Ccal_{\lambda}(X_i)\}$ are also i.i.d thus exchangeable.
Considering each $\mathbbm{1}\{\hat{Y}_i = Y_i \wedge Y_i \notin \Ccal_{\lambda}(X_i)\}$ for each $X_i, Y_i, \hat{Y}_i \in \Dcal$ as a function of $\lambda$, it holds by  Lemma~\ref{lem:non-increasing} that for each $X_i, Y_i, \hat{Y}_i$, $\mathbbm{1}\{\hat{Y}_i = Y_i \wedge Y_i \notin \Ccal_{\lambda}(X_i)\}$ is non-increasing in $\lambda$ and right-continuous. 
In addition, it holds that $\mathbbm{1}\{\hat{Y}_i = Y_i \wedge Y_i \notin \Ccal_{1}(X_i)\} = 0 \leq \alpha$ and that $\sup_{\lambda} \mathbbm{1}\{\hat{Y}_i = Y_i \wedge Y_i \notin \Ccal_{\lambda}(X_i)\} \leq 1 $ surely by definition.  

% \hfill \QED

% \subsection{Corollary~\ref{cor:counterfactual}}

% TBD.

\subsection{Proof of Theorem~\ref{thm:conformal-cf-benefit-control-cI}}
For the proof we will follow a similar procedure as in~\cite{angelopoulos2024conformal}.
% and use the following:
Since $ \mathbbm{1}\{\hat{Y} \neq Y \wedge Y \in \Ccal_{\lambda}(X)\} \leq 1$, $\forall X, Y, \hat{Y}, \lambda$, it will hold that
\begin{equation}
\label{eq:g-ub}
    \hat{G}_{n+1}(\lambda) = \frac{n\hat{G}_{n}(\lambda) +  \mathbbm{1}\{\hat{Y}_{n+1} \neq Y_{n+1} \wedge Y_{n+1} \in \Ccal_{\lambda}(X_{n+1})\}}{n+1} \leq \frac{n}{n+1}\hat{G}_{n}(\lambda) + \frac{1}{n+1},
\end{equation}
for every $\lambda \in \Lcal$.
Now, let
\begin{align}
\label{eq:lambda-hat-prime}
    \hat{\lambda}' = \sup \{ \lambda \in \Lcal : \hat{G}_{n+1} (\lambda) \leq \alpha\}.
\end{align}
If such $\hat{\lambda}'$ does not exist, by Eq.~\ref{eq:lambda-hat-prime} $\hat{\lambda}$ will also not exist. If $\hat{\lambda}'$ exists, 
% Given that $\hat{\lambda}' \geq 0 = \min_{\lambda \in \Lcal} \lambda$, as $\hat{G}_{n+1}(0) = 0 \leq \alpha$, $\hat{\lambda}'$ is well-defined. 

we will  show that $\hat{\lambda} \leq \hat{\lambda}'$.
If $\hat{\lambda} = 0$, then it will hold that $\hat{\lambda}' \geq \hat{\lambda} = 0 = \min_{\lambda \in \Lcal} \lambda$. If $\hat{\lambda} > 0$, then by Eq.~\ref{eq:lambda-hat-prime}, $\hat{\lambda}$ must satisfy 
\begin{equation}
\label{eq:B-hat-bound}
    \hat{G}_{n+1}(\hat{\lambda})  \leq \frac{n}{n+1}\hat{G}_{n}(\hat{\lambda}) + \frac{1}{n+1} \leq \alpha,
\end{equation}
which means that $\hat{\lambda} \leq \hat{\lambda}'$. Therefore, in any case we have that $\hat{\lambda} \leq \hat{\lambda}'$.

From Lemma~\ref{lem:non-decreasing}, the random functions $\mathbbm{1}\{\hat{Y}_i \neq Y_i \wedge Y_i \in \Ccal_{\lambda}(X_i)\}$ are non-decreasing for each $i \in [n+1]$. As a result, since $\hat{\lambda} \leq \hat{\lambda}'$, we have that 
\begin{multline}
    \EE_{X_i, Y_i, \hat{Y}_i \sim P^{\Mcal}}[ \mathbbm{1}\{\hat{Y}_{n+1} \neq Y_{n+1} \wedge Y_{n+1} \in \Ccal_{\hat{\lambda}}(X_{n+1})\}] \\ \leq \EE_{X_i, Y_i, \hat{Y}_i \sim P^{\Mcal}}[\mathbbm{1}\{\hat{Y}_{n+1} \neq Y_{n+1} \wedge Y_{n+1} \in \Ccal_{\hat{\lambda}'}(X_{n+1})\}],
\end{multline}
where $i=\{1, \dots, n+1\}$.
% the expectations are over $X_{1},..,X_{n+1}, Y_{1}, ..., Y_{n+1}, \hat{Y}_{1},.., \hat{Y}_{n+1}$. 
% 
Given that the data samples $(X_i, Y_i, \hat{Y}_i)$ for $ i \in [1, n+1]$ are exchangeable, 
the values 
$\mathbbm{1}\{\hat{Y}_{i} \neq Y_{i} \wedge Y_{i} \in \Ccal_{\hat{\lambda}'}(X_{i})\}$ for $i \in [n+1]$ are also exchangeable.
Therefore, due to exchangeability and given that $\hat{\lambda}'$ is fixed given $(X_i, Y_i, \hat{Y}_i)$ for $ i \in [1, n+1]$,
$\mathbbm{1}\{\hat{Y}_{n+1} \neq Y_{n+1} \wedge Y_{n+1} \in \Ccal_{\hat{\lambda}'}(X_{n+1})\}$ has the same probability of taking any value in $\Dcal_{\hat{\lambda}'} = \{\mathbbm{1}\{\hat{Y}_{i} \neq Y_{i} \wedge Y_{i} \in \Ccal_{\hat{\lambda}'}(X_{i})\}\}_{i=1}^{n+1}$, \ie, the value of $\mathbbm{1}\{\hat{Y}_{i} \neq Y_{i} \wedge Y_{i} \in \Ccal_{\hat{\lambda}'}(X_{i})\}$ for the $n+1$-th sample in $\{(X_i, Y_i, \hat{Y}_i)\}_{i=1}^{n+1}$ follows a uniform distribution over all the possible values in $\Dcal_{\hat{\lambda}'}$.
As a result, from the law of total expectation it holds for $i=\{1, \dots, n+1\}$ that $ \EE_{X_i, Y_i, \hat{Y}_i \sim P^{\Mcal}}[\mathbbm{1}\{\hat{Y}_{n+1} \neq Y_{n+1} \wedge Y_{n+1} \in \Ccal_{\hat{\lambda}'}(X_{n+1})\}] = \EE_{\Dcal_{\hat{\lambda}'}}[\mathbbm{1}\{\hat{Y}_{n+1} \neq Y_{n+1} \wedge Y_{n+1} \in \Ccal_{\hat{\lambda}'}(X_{n+1})\} \given \Dcal_{\hat{\lambda}'}] = \frac{\sum_{i=1}^{n+1}\mathbbm{1}\{\hat{Y}_{i} \neq Y_{i} \wedge Y_{i} \in \Ccal_{\hat{\lambda}'}(X_{i})\}}{n+1} = \hat{G}_{n+1}(\hat{\lambda}')$. 
Finally, from Eq.~\ref{eq:B-hat-bound} and the above we have
\begin{multline}
    % \EE[g_{\hat{\lambda}}(X_{n+1}, Y_{n+1}, \hat{Y}_{n+1})] \leq \EE[g_{\hat{\lambda}'}(X_{n+1}, Y_{n+1}, \hat{Y}_{n+1})]
    \EE_{X_i, Y_i, \hat{Y}_i \sim P^{\Mcal}}[ \mathbbm{1}\{\hat{Y}_{n+1} \neq Y_{n+1} \wedge Y_{n+1} \in \Ccal_{\hat{\lambda}}(X_{n+1})\}] \\ \leq \EE_{X_i, Y_i, \hat{Y}_i \sim P^{\Mcal}}[\mathbbm{1}\{\hat{Y}_{n+1} \neq Y_{n+1} \wedge Y_{n+1} \in \Ccal_{\hat{\lambda}'}(X_{n+1})\}]
    = \hat{G}_{n+1}(\hat{\lambda}') 
    \leq \alpha,
\end{multline}
where $i=\{1, \dots, n+1\}$.
% \hfill \QED
\vspace{1mm}
\subsection{Additional Lemmas}\label{app:lemmas}

\begin{lemma}\label{lem:harm-bound-interventional-monotonicity}
 Under the interventional monotonicity assumption, for any $x, y, \hat{y} \sim P^{\Mcal}$, the counterfactual harm that a decision support system $\Ccal_{\lambda}$ would have caused, if deployed, is given by
\begin{multline}
\label{eq:harm-cond-interventional-mon}
    h_{\lambda}(x, y, \hat{y}) = \EE_{\hat{Y} \sim P^{\Mcal ; \text{do}(\Lambda=\lambda) \given \hat{Y}=\hat{y}, X=x, Y=y}}
        [ \mathbbm{1}\{\hat{Y} \neq y \}] \cdot \mathbbm{1}\{\hat{y} = y\} \cdot \mathbbm{1}\{y \in \Ccal_{\lambda}(x)\} \\
        + \mathbbm{1}\{\hat{y} = y\}\cdot \mathbbm{1}\{y \notin \Ccal_{\lambda}(x)\}.
\end{multline}
\end{lemma}
\begin{proof}
Given an observation $(x, y, \hat{y})$ under no intervention, we have 
% assuming that Assumption~\ref{asm:cond-interventional-monotonicity} holds.
\begin{align*}
    h_{\lambda}(x, y, \hat{y}) &= \EE_{\hat{Y} \sim P^{\Mcal ; \text{do}(\Lambda=\lambda) \given \hat{Y}=\hat{y}, X=x, Y=y}}
        [\max\{0, \mathbbm{1}\{\hat{y} = y\} - \mathbbm{1}\{\hat{Y} = y \} \}] \nonumber \\
        &= \EE_{\hat{Y} \sim P^{\Mcal ; \text{do}(\Lambda=\lambda) \given \hat{Y}=\hat{y}, X=x, Y=y}}
        [\max\{0, 1 - \mathbbm{1}\{\hat{Y} = y \} \}] \cdot \mathbbm{1}\{\hat{y} = y\}  \nonumber \\
        &+ \EE_{\hat{Y} \sim P^{\Mcal ; \text{do}(\Lambda=\lambda) \given \hat{Y}=\hat{y}, X=x, Y=y}}
        [\max\{0, 0 - \mathbbm{1}\{\hat{Y} = y \} \}] \cdot \mathbbm{1}\{\hat{y} \neq y\} \nonumber \\
        & = \EE_{\hat{Y} \sim P^{\Mcal ; \text{do}(\Lambda=\lambda) \given \hat{Y}=\hat{y}, X=x, Y=y}}
        [\max\{0, 1 - \mathbbm{1}\{\hat{Y} = y \} \}] \cdot \mathbbm{1}\{\hat{y} = y\} + 0 \nonumber \\
        &= \EE_{\hat{Y} \sim P^{\Mcal ; \text{do}(\Lambda=\lambda) \given \hat{Y}=\hat{y}, X=x, Y=y}}
        [\max\{0, 1 - \mathbbm{1}\{\hat{Y} = y \} \}] \cdot \mathbbm{1}\{\hat{y} = y\} \cdot \mathbbm{1}\{y \in \Ccal_{\lambda}(x)\} \nonumber \\
        &+ \EE_{\hat{Y} \sim P^{\Mcal ; \text{do}(\Lambda=\lambda) \given \hat{Y}=\hat{y}, X=x, Y=y}}
        [\max\{0, 1 - \mathbbm{1}\{\hat{Y} = y \} \}] \cdot \mathbbm{1}\{\hat{y} = y\} \cdot \mathbbm{1}\{y \notin \Ccal_{\lambda}(x)\} \nonumber \\
        &= \EE_{\hat{Y} \sim P^{\Mcal ; \text{do}(\Lambda=\lambda) \given \hat{Y}=\hat{y}, X=x, Y=y}}
        [\max\{0, \mathbbm{1}\{\hat{Y} \neq y \} \}] \cdot \mathbbm{1}\{\hat{y} = y\} \cdot \mathbbm{1}\{y \in \Ccal_{\lambda}(x)\} \nonumber \\
        &+ \EE_{\hat{Y} \sim P^{\Mcal ; \text{do}(\Lambda=\lambda) \given \hat{Y}=\hat{y}, X=x, Y=y}}
        [\max\{0, 1 - 0 \}] \cdot \mathbbm{1}\{\hat{y} = y\} \cdot \mathbbm{1}\{y \notin \Ccal_{\lambda}(x)\} \nonumber \\
        &= \EE_{\hat{Y} \sim P^{\Mcal ; \text{do}(\Lambda=\lambda) \given \hat{Y}=\hat{y}, X=x, Y=y}}
        [\max\{0, \mathbbm{1}\{\hat{Y} \neq y \} \}] \cdot \mathbbm{1}\{\hat{y} = y\} \cdot \mathbbm{1}\{y \in \Ccal_{\lambda}(x)\}  \nonumber\\
        &+ 1 \cdot  \mathbbm{1}\{\hat{y} = y\} \cdot \mathbbm{1}\{y \notin \Ccal_{\lambda}(x)\} \nonumber \\
        & = \EE_{\hat{Y} \sim P^{\Mcal ; \text{do}(\Lambda=\lambda) \given \hat{Y}=\hat{y}, X=x, Y=y}}
        [ \mathbbm{1}\{\hat{Y} \neq y \}] \cdot \mathbbm{1}\{\hat{y} = y\} \cdot \mathbbm{1}\{y \in \Ccal_{\lambda}(x)\}  + \mathbbm{1}\{\hat{y} = y\} \cdot \mathbbm{1}\{y \notin \Ccal_{\lambda}(x)\}
\end{align*}
\end{proof}

\begin{lemma}\label{lem:interventional-counterfactual}
    Under the structural causal model $\Mcal$ satisfying Eq.~\ref{eq:scm}, given an observation $(x, y)$ it holds
    \begin{multline}
        \EE_{\hat{Y}' \sim P^{\Mcal}}\left [\EE_{\hat{Y} \sim P^{\Mcal;\text{do}(\Lambda=\lambda)\given X = x, Y = y, \hat{Y} = \hat{Y}'}} \left [\mathbbm{1}\{ \hat{Y} \neq Y\} \right ] \right]\\ = \EE_{\hat{Y}\sim P^{\Mcal;\text{do}(\Lambda = \lambda)}}[ \mathbbm{1}\{\hat{Y} \neq Y \}\given X=x, Y = y].
    \end{multline}
\end{lemma}

% \subsection{Lemma~\ref{lem:interventional-counterfactual}}
\begin{proof}
    Given an observation $(x, y, \hat{y})$ under no intervention, for the expectation of the prediction of the human $\hat{Y}\sim P^{\Mcal;\text{do}(\Lambda = \lambda) \given X=x, Y=y, \hat{Y} = \hat{y}}$ it holds
\begin{multline*}\label{eq:counterfactual-exp}
    \EE_{\hat{Y}\sim P^{\Mcal;\text{do}(\Lambda = \lambda) \given X=x, Y=y, \hat{Y}=\hat{y}}} \left [ \mathbbm{1}\{\hat{Y} \neq Y \}\right] \\ =  \sum_{u,v} P(U=u, V=v \given X=x, Y=y, \hat{Y}=\hat{y})\cdot\mathbbm{1}\{f_{\hat{Y}}(u, v, \Ccal_{\lambda}(x)) \neq y\},
\end{multline*}
where $P(U=u, V=v \given X=x, Y=y, \hat{Y}=\hat{y})$ is the posterior distribution of the exogenous variables $U, V$ conditional on the observation $(x, y, \hat{y})$.

By taking the expectation of the above with respect to the human prediction $\hat{y}$ under no intervention we have
\begin{align*}
    &\EE_{\hat{Y}'\sim P^{\Mcal}} \left [\EE_{\hat{Y}\sim P^{\Mcal;\text{do}(\Lambda = \lambda) \given X=x, Y=y, \hat{Y}=\hat{Y}'}} \left[ \mathbbm{1}\{\hat{Y} \neq Y \} \right] \right]  \\ & =
    \sum_{\hat{y}} P^{\Mcal}(\hat{Y}' = \hat{y} \given X=x, Y=y) \\& \cdot \sum_{u,v} P(U=u, V=v \given X=x, Y=y, \hat{Y}=\hat{y})\cdot\mathbbm{1}\{f_{\hat{Y}}(u, v, \Ccal_{\lambda}(x)) \neq y\} \\
    &= \sum_{u,v} \sum_{\hat{y}} P(U=u, V=v \given X=x, Y=y, \hat{Y}=\hat{y}) \cdot P^{\Mcal}(\hat{Y}' = \hat{y} \given X=x, Y=y) \\
    & \cdot\mathbbm{1}\{f_{\hat{Y}}(u, v, \Ccal_{\lambda}(x)) \neq y\}. 
\end{align*}
Since $\hat{Y}'$ is equal in distribution to $\hat{Y}$ we can rewrite the above as 
\begin{align*}
    &\EE_{\hat{Y}'\sim P^{\Mcal}} \left [\EE_{\hat{Y}\sim P^{\Mcal;\text{do}(\Lambda = \lambda) \given X=x, Y=y, \hat{Y}=\hat{Y}'}} \left[ \mathbbm{1}\{\hat{Y} \neq Y \} \right] \right]  \\ 
     &= \sum_{u,v} \sum_{\hat{y}} P(U=u, V=v \given X=x, Y=y, \hat{Y}=\hat{y})\cdot P^{\Mcal}(\hat{Y} = \hat{y} \given X=x, Y=y) \\
    & \cdot\mathbbm{1}\{f_{\hat{Y}}(u, v, \Ccal_{\lambda}(x)) \neq y\}
\end{align*}

and since $ P(U=u, V=v \given X=x, Y=y) = \sum_{\hat{y}} P(U=u, V=v \given X=x, Y=y, \hat{Y}=\hat{y})\cdot P^{\Mcal}(\hat{Y} = \hat{y} \given X=x, Y=y)$, we have
\begin{align*}
     &\EE_{\hat{Y}'\sim P^{\Mcal}} \left [\EE_{\hat{Y}\sim P^{\Mcal;\text{do}(\Lambda = \lambda) \given X=x, Y=y, \hat{Y}=\hat{Y}'}} \left[ \mathbbm{1}\{\hat{Y} \neq Y \} \right] \right]  \\
      &= \sum_{u,v} P(U=u, V=v \given X=x, Y=y)
     \cdot\mathbbm{1}\{f_{\hat{Y}}(u, v, \Ccal_{\lambda}(x)) \neq y\}\\
     &=  \sum_{u} P(U=u \given X=x, Y=y) \cdot \sum_{v} P(V=v \given X=x, Y=y)
     \cdot\mathbbm{1}\{f_{\hat{Y}}(u, v, \Ccal_{\lambda}(x)) \neq y\}\\
     &=  \EE_{\hat{Y}\sim P^{\Mcal;\text{do}(\Lambda = \lambda)}}[ \mathbbm{1}\{\hat{Y} \neq Y \}\given X=x, Y = y],
\end{align*}
where we used that the exogenous variables $U$ and $V$ are independent.
\end{proof}

\begin{lemma}
\label{lem:non-increasing}
    Under counterfactual monotonicity, for a given $x,y, \hat{y}$ the counterfactual harm $h_{\lambda}(x,y,\hat{y})$ is non-increasing as a function of $\lambda$ and right-continuous.
\end{lemma}
\begin{proof}    
        If $y \in \Ccal_{\lambda}(x)$ $\forall  \lambda \in \Lcal $ then $h_{\lambda}(x,y,\hat{y})= 0, \forall \lambda \in \Lambda$ given Eq.~\ref{eq:harm-counterfactual-mon}.
        If $\exists \Tilde{\lambda} \in \Lambda$ such that $\Tilde{\lambda} = \inf\{\lambda \in \Lambda: y \in \Ccal_{\lambda}(x)\}$, then $\forall \lambda \geq \Tilde{\lambda}$, it holds that $y \in \Ccal_{\lambda}(x)$, and as a result $h_{\lambda}(x,y,\hat{y}) = 0$. For $\forall \lambda' < \Tilde{\lambda}$, it holds that $y \notin \Ccal_{\lambda'}(x)$ and as a result Eq.~\ref{eq:harm-counterfactual-mon} becomes
        \begin{equation}
            h_{\lambda'}(x,y,\hat{y}) = \mathbbm{1}\{\hat{y} = y\}\cdot 1 = \mathbbm{1}\{\hat{y} = y\} \geq 0 = h_{\lambda}(x,y,\hat{y}),
        \end{equation}
    where $\mathbbm{1}\{\hat{y} = y\}$ is constant $\forall \lambda' < \Tilde{\lambda}$. 
\end{proof}

\begin{lemma}
    \label{lem:non-decreasing}
    For a given data sample $x$ with label $y$, let $G: \Lcal \rightarrow \{0,1\}$ be the function $G(\lambda) = \mathbbm{1}\{\hat{Y} \neq y \wedge y \in \Ccal_{\lambda}(x)\}$. Under interventional monotonicity, $G$ is non-decreasing in $\lambda$.
\end{lemma}
\begin{proof}
    
 Let $\lambda, \lambda' \in \Lcal$  such that $\lambda \leq \lambda'$. Then it will always hold that $\Ccal_{\lambda}(x) \subseteq \Ccal_{\lambda'}(x)$. We distinguish the following cases for $\Ccal_{\lambda}(x), \Ccal_{\lambda'}(x)$:
    \begin{itemize}
        \item $y \notin \Ccal_{\lambda'}(x)$. Then $G(\lambda) = G(\lambda') = 0$.
        \item $y \notin \Ccal_{\lambda}(x)$ and $y \in \Ccal_{\lambda'}(x)$. Then, $G(\lambda) = 0$ and $G(\lambda') =  \mathbbm{1}\{\hat{Y} \neq y \} \geq 0 = G(\lambda)$
        \item $y \in \Ccal_{\lambda}(x) \subseteq \Ccal_{\lambda'}(x)$. Then, $G(\lambda) = G(\lambda') = \mathbbm{1}\{\hat{Y} \neq y \}$. 
    \end{itemize}
\end{proof}

% \subsection{Corollary~\ref{cor:interventional}}
% TBD.

\clearpage
\newpage

\section{Additional experimental details}\label{app:additional-details}
In this section, we describe the experimental details that are omitted from the main paper due to space constraints.

\xhdr{Mixture of multinomial logit models}
Multinomial logit models (MNLs)~\cite{heiss2016discrete} are among the most popular discrete choice models used to predict the probability that a human chooses an alternative, given a set of alternatives within a particular \textit{context}~\cite{tversky1993context}.
Straitouri et al.~\cite{straitouri2023improving} used MNLs to predict the probability that a human expert predicts a label (\eg, the ground truth label),  given a prediction set, while considering as context the ground truth label and the level of difficulty of the image.
To this end, for images with the same level of difficulty, they estimate the confusion matrix of the expert predictions on their own and use it to parameterize an MNL model.
Therefore, their mixture of MNLs comprises of a different MNL model for each level of difficulty. 
To distinguish images based on their level of difficulty,  
they consider different quantiles of the experts'{} average accuracy values of all images.
In our experiments, we consider two levels of difficulty,
where in the first level we consider images for which the average accuracy of experts is smaller than the $0.5$ quantile, and consider the rest of the images in the second level. 

\xhdr{Labels of the ImageNet16H dataset} The ImageNet16H~\cite{steyvers2022bayesian} was created using $1{,}200$ unique images labeled into $207$ different fine-grained categories from the ILSRVR 2012 dataset~\cite{russakovsky2015imagenet}. 
Steyvers et al.~\cite{steyvers2022bayesian} mapped each of these fine-grained categories into one out of $16$ coarse-grained categories---namely airplane, bear, bicycle, bird, boat, bottle, car, cat, chair, clock, dog, elephant, keyboard, knife, oven, and truck---that serve as the ground truth labels. 
This mapping essentially eliminated any label disagreement that could potentially occur between annotators in the ILSRVR 2012 dataset.   

\xhdr{Implementation details and licenses}
We implement our methods and execute our experiments using {\tt Python $3.10.9$}, along with the open-source libraries {\tt NumPy 1.26.4} (BSD License), and {\tt Pandas 2.2.1} (BSD 3-Clause License).
For reproducibility, we used a different fixed random seed for each random sampling of the test and calibration sets. Both datasets used in our experiments are distributed under the Creative Commons Attribution License 4.0 (CC BY 4.0).

\clearpage
\newpage

\section{Additional experimental results on counterfactual monotonicity}\label{app:additional-results-counterfactual}
In this section, we evaluate our methodology on additional strata of images and on larger calibration sets, and study the relationship between average counterfactual harm, average prediction set size and empirical coverage. 

\xhdr{Other strata of images} Figures~\ref{fig:acc_vs_harm_mnl_noise80} and~\ref{fig:acc_vs_harm_mnl_noise95}  show the average accuracy $A(\lambda)$ achieved by a human expert using $\Ccal_{\lambda}$, as predicted by the mixture of MNLs, against the average counterfactual harm $H(\lambda)$ caused by $\Ccal_{\lambda}$ for $\alpha \in \{0.01, 0.05\}$ on the strata of images with $\omega = 80$ and $\omega = 95$, respectively.
The meaning and coloring of each point are the same as in Figure~\ref{fig:acc_vs_harm_mnl}.
The results on both strata of images are consistent with the results on the stratum of images with $\omega  = 110$---the decision support systems $\Ccal_{\lambda}$ always cause some amount of counterfactual harm, our framework successfully identifies harm-controlling sets and there is a trade-off between accuracy and counterfactual harm. 
However, for the strata with $\omega \in \{80, 95\}$, where the classification task has lower difficulty compared to the stratum with $\omega=110$, we observe that the trade-off between accuracy and counterfactual is minimal.
\begin{figure}[ht]
    \centering
    \subfloat[$\alpha = 0.01$]{
    % \label{fig:80class-alpha0.01}
    \includegraphics[width=.45\linewidth, valign=t]{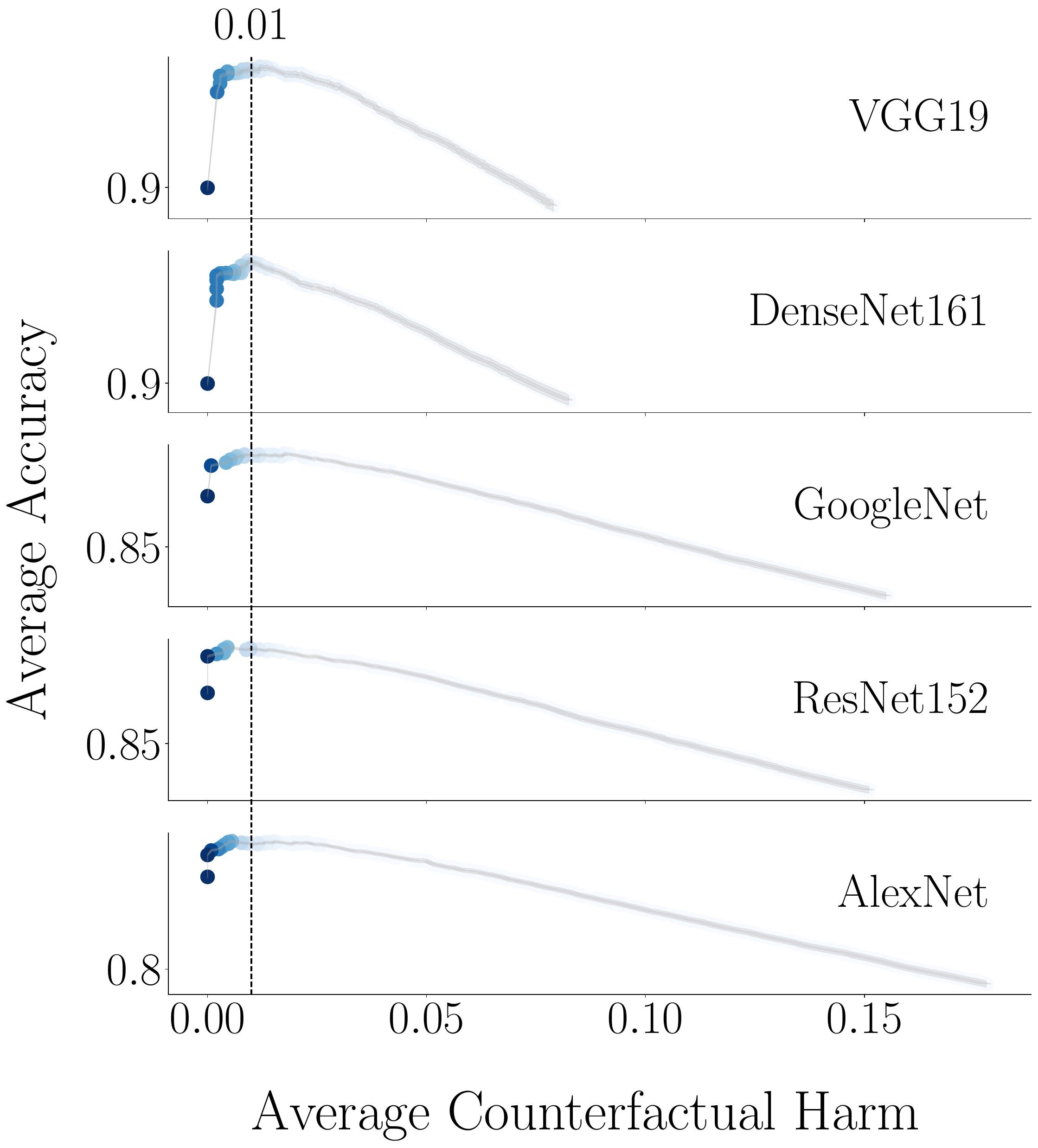}
    }
    \hspace{-0.5mm}
    \subfloat[$\alpha=0.05$]{
    % \label{fig:80class-alpha0.05}
    \includegraphics[width=.45\linewidth, valign=t]{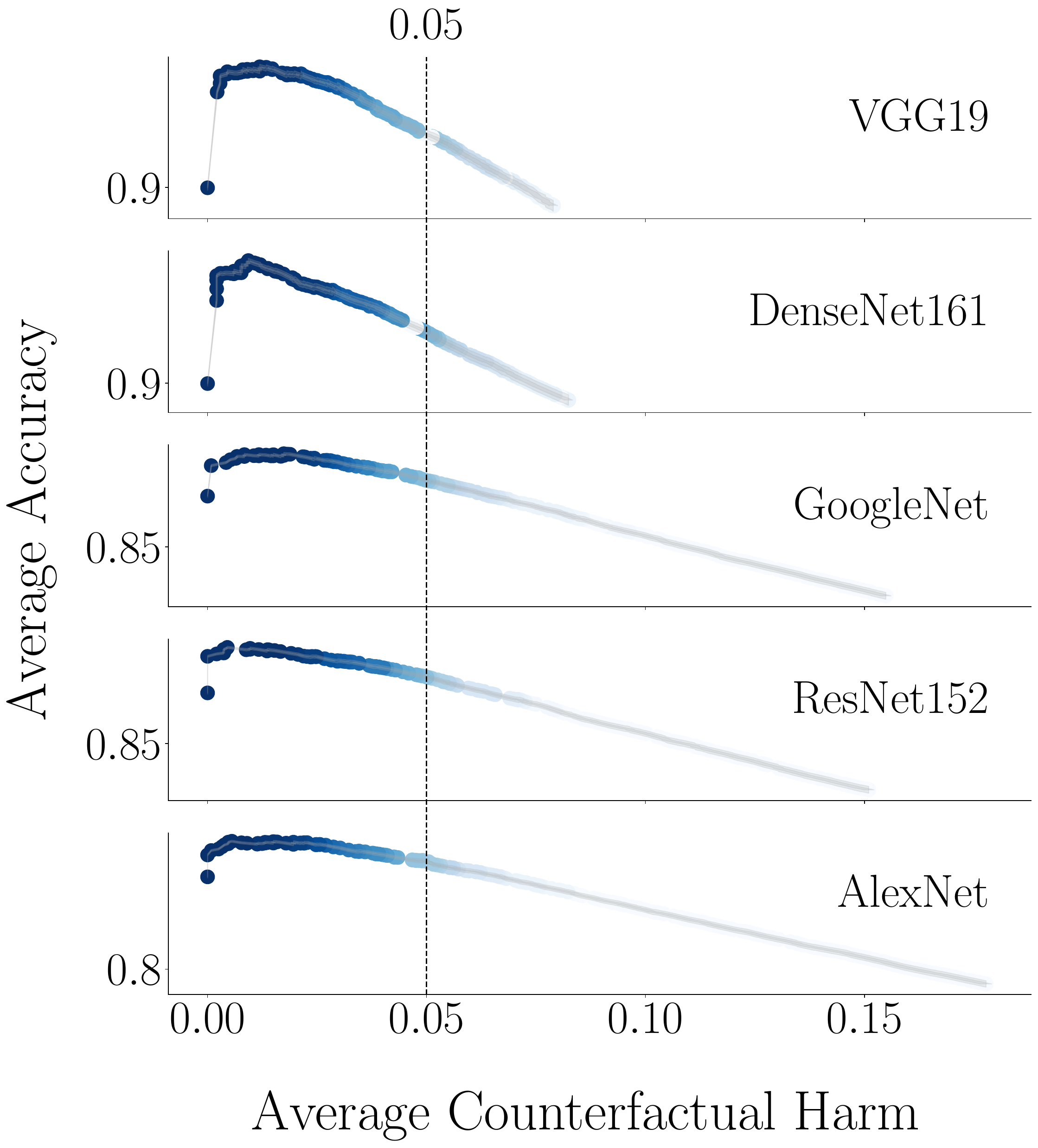}}
    \includegraphics[scale=0.203, valign=t]{figures/colorbar_big.pdf}
     \caption{Average accuracy estimated by the mixture of MNLs against the average counterfactual harm for images with $\omega = 80$.
    Each point corresponds to a $\lambda$ value from $0$ to $1$ with step $0.001$ and the coloring indicates the relative frequency with which each $\lambda$ value is in $\Lambda(\alpha)$ across random samplings of the calibration set.
    Each row corresponds to decision support systems $\Ccal_{\lambda}$ with a different pre-trained classifier with average accuracies   $0.891$~(VGG19),  $0.892$~(DenseNet161),  $0.802$~(GoogleNet), $0.804$~(ResNet152), and $0.784$~(AlexNet).
    The average accuracy achieved by human experts on their own is $0.9$.
    The results are averaged across $50$ random samplings of the test and calibration set.
    In both panels, $95\%$ confidence intervals have width always below $0.02$ and are represented using shaded areas.
    }
    \label{fig:acc_vs_harm_mnl_noise80}
    \vspace{-3mm}
\end{figure}
\clearpage
\newpage
\begin{figure}[t]
    \centering
    \subfloat[$\alpha = 0.01$]{
    % \label{fig:class-alpha0.01}
    \includegraphics[width=.45\linewidth, valign=t]{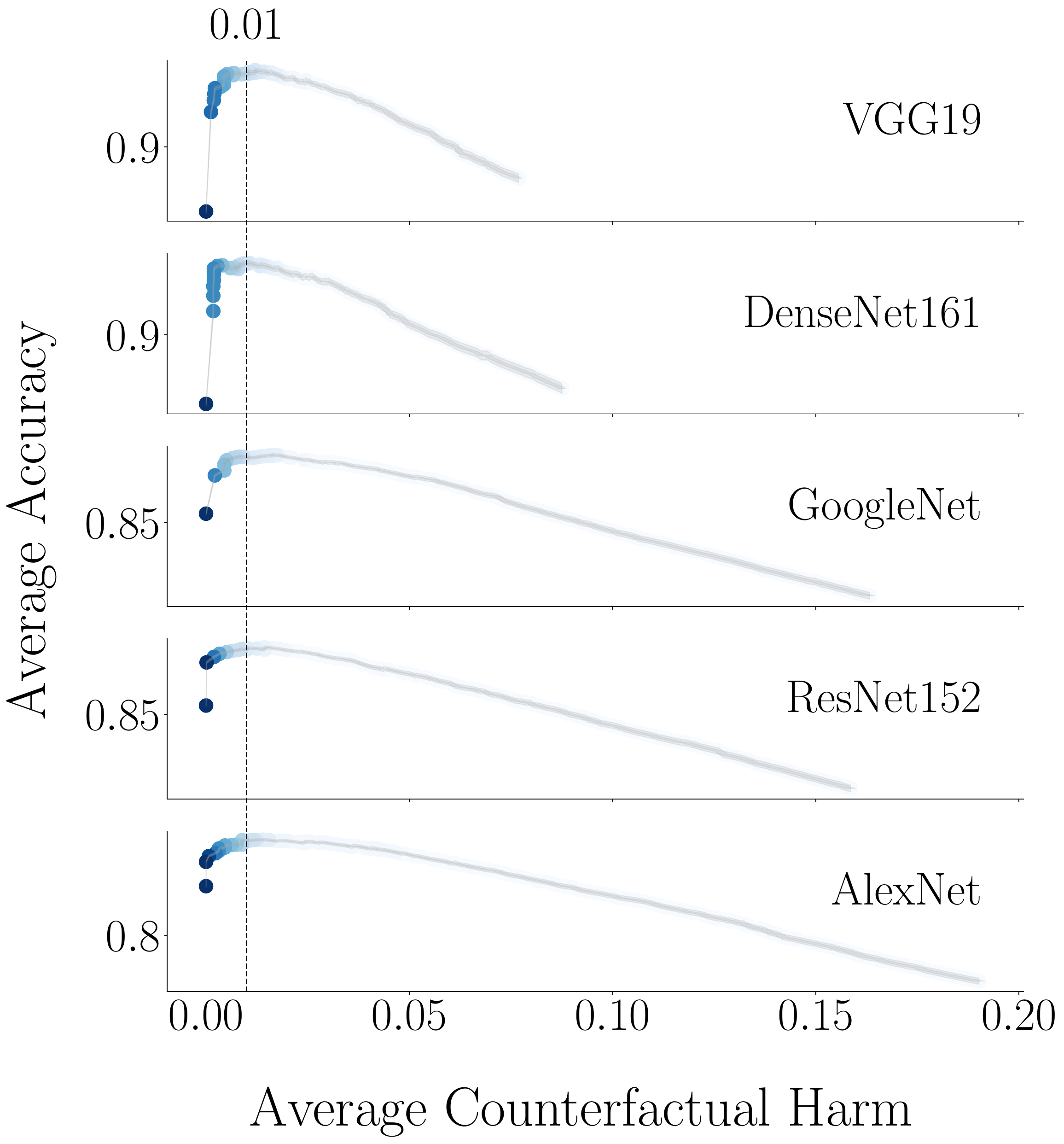}
    }
    \hspace{-0.5mm}
    \subfloat[$\alpha=0.05$]{
    % \label{fig:class-alpha0.05}
    \includegraphics[width=.45\linewidth, valign=t]{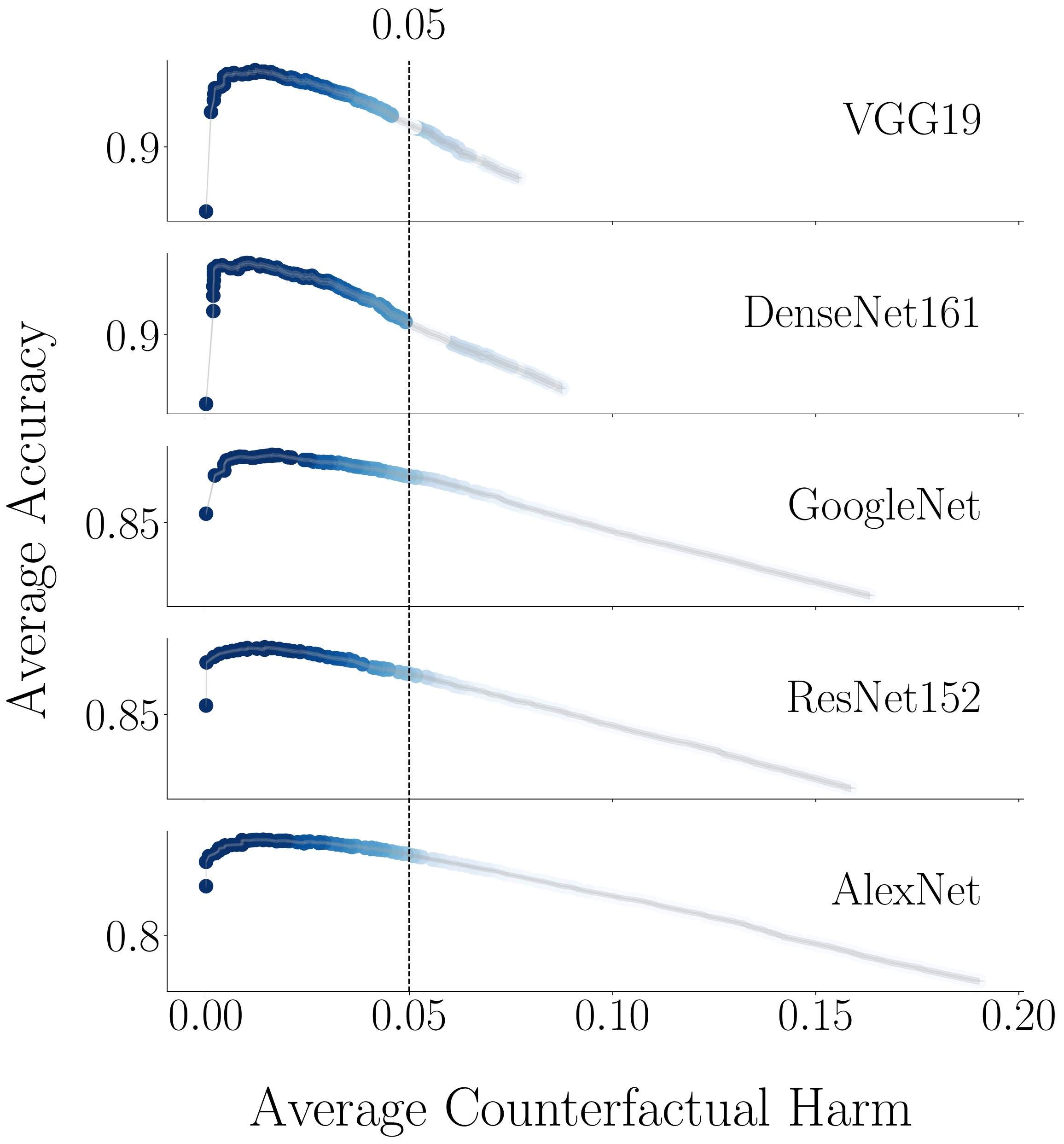}}
    \includegraphics[scale=0.197, valign=t]{figures/colorbar_big.pdf}
 \caption{Average accuracy estimated by the mixture of MNLs against the average counterfactual harm for images with $\omega = 95$.
    Each point corresponds to a $\lambda$ value from $0$ to $1$ with step $0.001$ and the coloring indicates the relative frequency with which each $\lambda$ value is in $\Lambda(\alpha)$ across random samplings of the calibration set.
    Each row corresponds to decision support systems $\Ccal_{\lambda}$ with a different pre-trained classifier with average accuracies   $0.88$~(VGG19),  $0.868$~(DenseNet161),  $0.775$~(GoogleNet), $0.773$~(ResNet152), and $0.745$~(AlexNet).
    The average accuracy achieved by human experts on their own is $0.86$.
    The results are averaged across $50$ random samplings of the test and calibration set.
    In both panels, $95\%$ confidence intervals have width always below $0.02$ and are represented using shaded areas.
    }
    \label{fig:acc_vs_harm_mnl_noise95}
    \vspace{-3mm}
\end{figure}

% \clearpage
% \newpage

% \section{Experimental results on counterfactual monotonicity using larger calibration sets} \label{app:calibration_set_size}
\xhdr{Larger calibration sets} Figure~\ref{fig:acc_vs_harm_mnl_cal_sizes_110} shows the average accuracy $A(\lambda)$ achieved by a human expert using $\Ccal_{\lambda}$, as predicted by the mixture of MNLs, against the average counterfactual harm $H(\lambda)$ caused by $\Ccal_{\lambda}$ for $\alpha = 0.01$ on the stratum of images with
% with $\omega=80$ in Figure~\ref{fig:acc_vs_harm_mnl_cal_sizes_80}, with $\omega=95$ in Figure~\ref{fig:acc_vs_harm_mnl_cal_sizes_95} and with 
$\omega=110$. 
% In all Figures, 
% In panel (a), we use $n = 360$ calibration samples ($30\%$ of the dataset) and, in panel (b), $n = 600$ calibration samples ($50\%$ of the dataset). %Consistently across all strata of images, 
The results show that the larger the calibration set, the more often $\lambda$ values with average counterfactual harm $H(\lambda)$ close to $\alpha$\footnote{Here, note that for $\alpha = 0.01$, these $\lambda$ values happen to achieve the highest average accuracy $A(\lambda)$.} are included in the harm-controlling set $\Lambda(\alpha)$.  
This is because, the larger the number of samples in the calibration set ($n$), the lower the estimation error in the empirical estimate of the average counterfactual harm $\hat{H}_{n}$. 
We find qualitatively similar results for the other strata of images.

\begin{figure}[ht]
    \centering
    \subfloat[Calibration samples $n = 360$ $(30\%)$]{
    \label{fig:0.3cal_set_110}
    \includegraphics[width=.45\linewidth, valign=t]{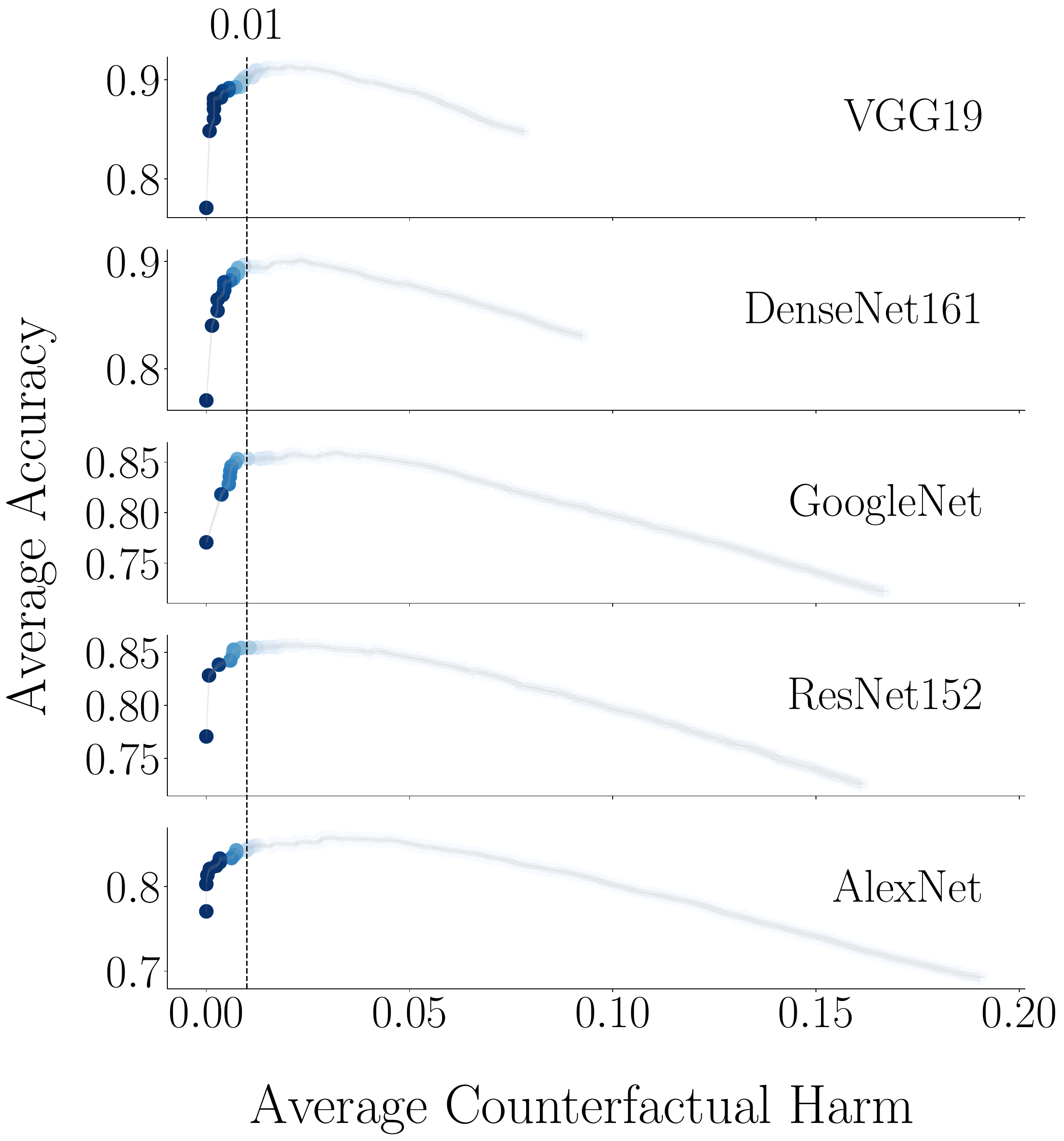}
    }
    % \hspace{-0.5mm}
    \subfloat[Calibration samples $n = 600$ $(50\%)$]{
    \label{fig:0.5cal_set_110}
    \includegraphics[width=.45\linewidth, valign=t]{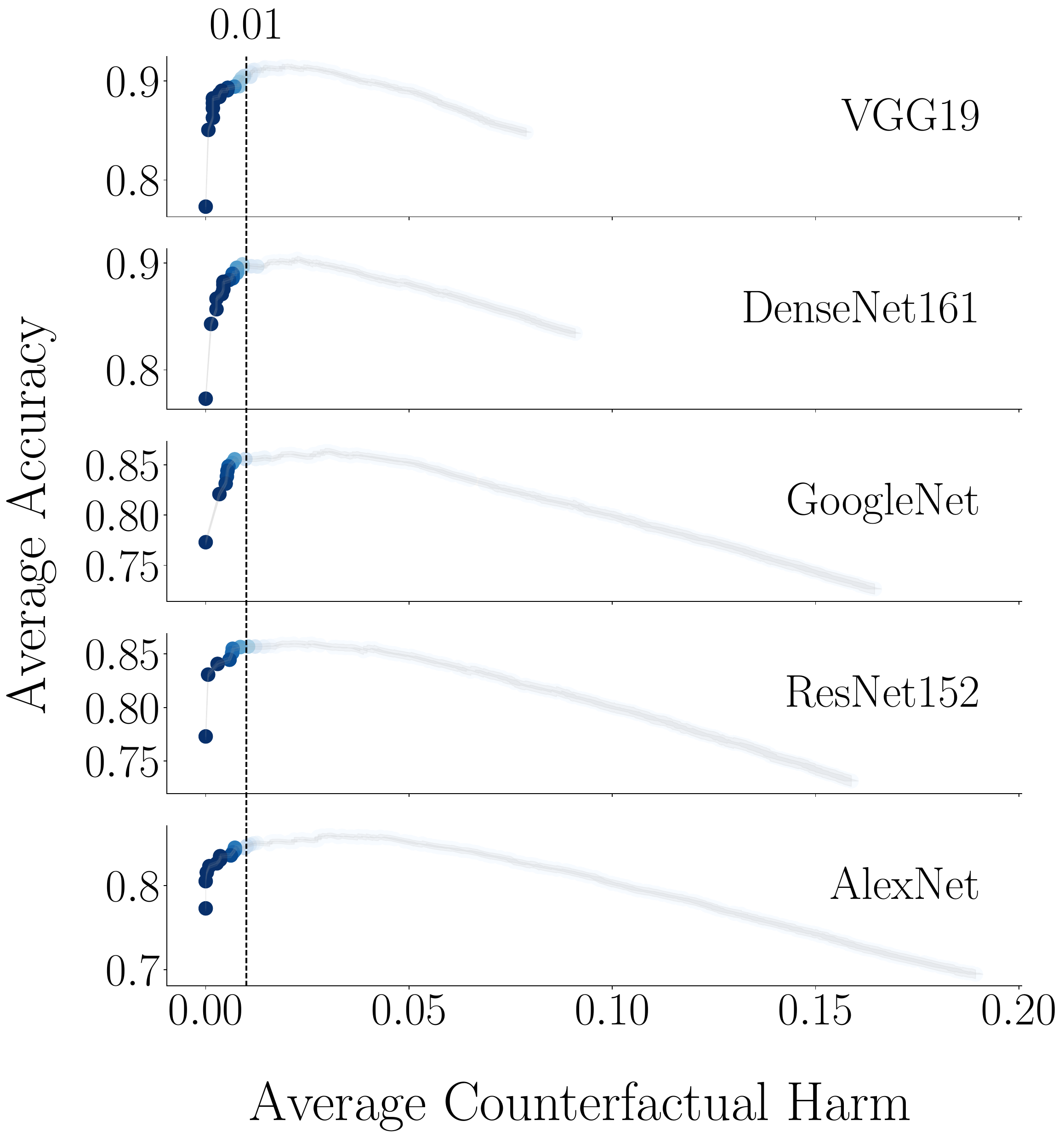}}
    \includegraphics[scale=0.196, valign=t]{figures/colorbar_big.pdf}
    \caption{Average accuracy estimated by the mixture of MNLs against the average counterfactual harm for images with $\omega = 110$ and different sizes of the calibration set.
    Each point corresponds to a $\lambda$ value from $0$ to $1$ with step $0.001$ and the coloring indicates the relative frequency with which each $\lambda$ value is in $\Lambda(\alpha)$ across random samplings of the calibration set.
    Each row corresponds to decision support systems $\Ccal_{\lambda}$ with a different pre-trained classifier with average accuracies   \vgg~(VGG19),  \densenet~(DenseNet),  \googlenet~(GoogleNet), \resnet~(ResNet152), and \alexnet~(AlexNet).
    The average accuracy achieved by the simulated human experts on their own is $0.771$.
    The results are averaged across $50$ random samplings of the test and calibration set.
    In both panels, $95\%$ confidence intervals are represented using shaded areas and always have width below $0.02$.
    }
    \label{fig:acc_vs_harm_mnl_cal_sizes_110}
    \vspace{-3mm}
\end{figure}

\xhdr{Average prediction set-size and empirical coverage}
Figure~\ref{fig:set-size-coverage-vs-harm-110}(a) shows the average prediction set-size of each decision support system $\Ccal_{\lambda}$ against the average counterfactual harm $H(\lambda)$ caused by $\Ccal_{\lambda}$ on the stratum of images with $\omega=110$.
Figure~\ref{fig:set-size-coverage-vs-harm-110}(b) shows the empirical coverage, \ie, the fraction of samples $x$ in the test set, for which  ground truth label $y \in \Ccal_{\lambda}(x)$, for each decision support system $\Ccal_{\lambda}$ against the average counterfactual harm $H(\lambda)$ caused by $\Ccal_{\lambda}$ on the same stratum of images. 
The results show that systems with higher coverage construct larger prediction sets in average and cause less counterfactual harm. We find similar results for the other strata of images.      

\begin{figure}[ht]
    \centering
    \subfloat[Average Set-Size]{
    \label{fig:set-size-110}
    \includegraphics[width=.45\linewidth, valign=t]{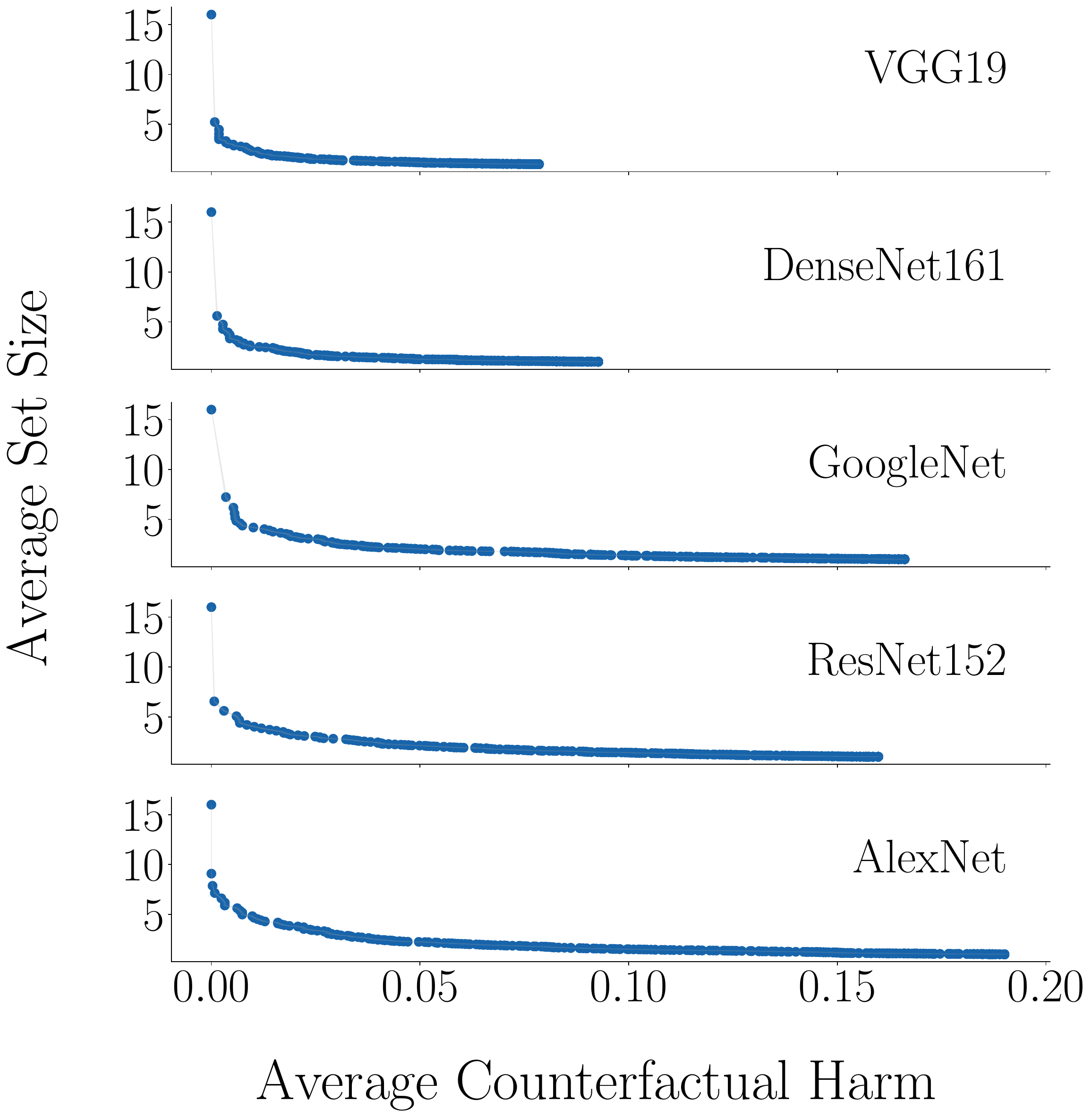}
    }
    % \hspace{-0.5mm}
    \subfloat[Empirical Coverage]{
    \label{fig:coverage-110}
    \includegraphics[width=.45\linewidth, valign=t]{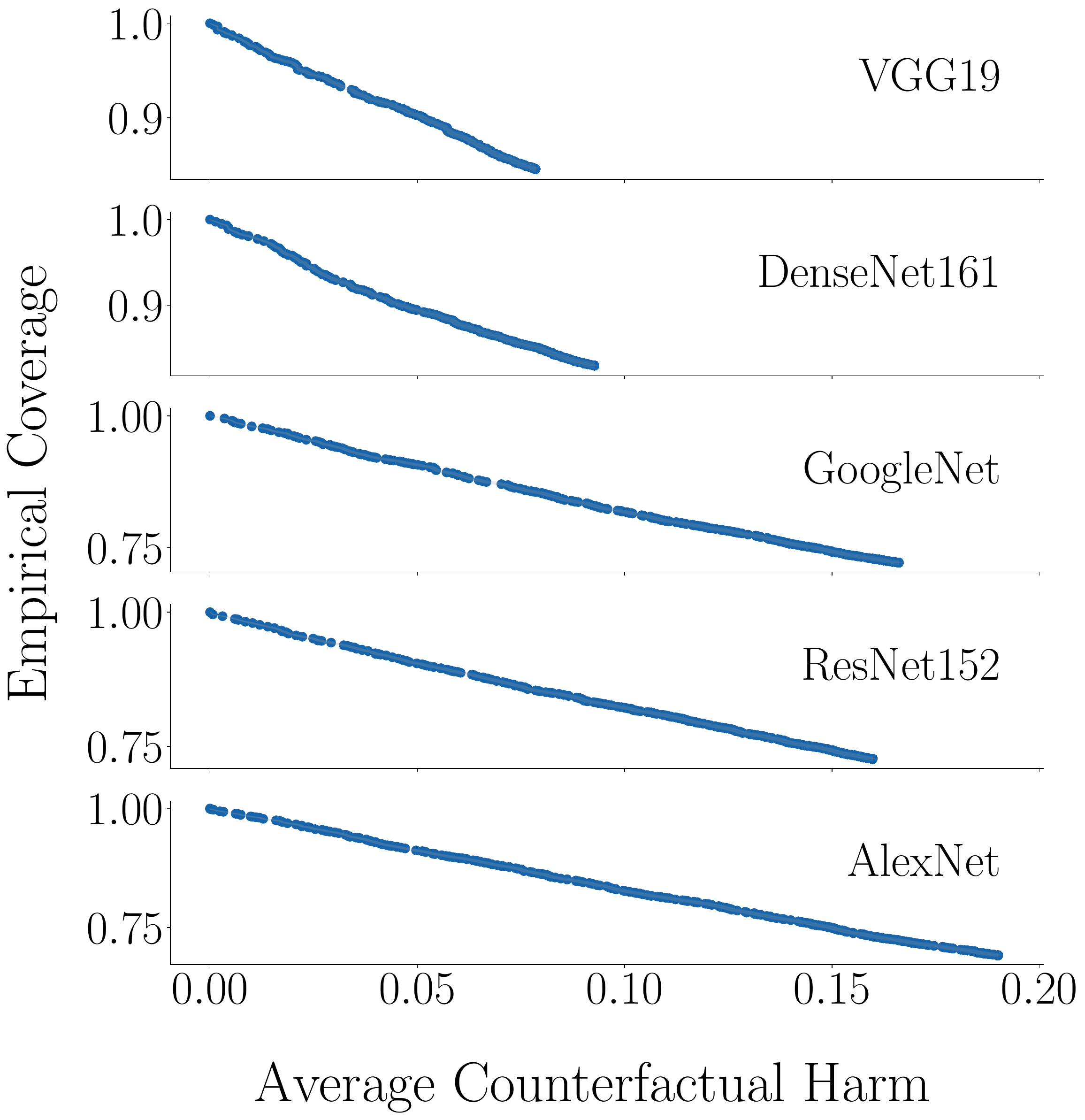}}
    \caption{Average set size and empirical coverage against the average counterfactual harm for images with $\omega = 110$.
    Each point corresponds to a $\lambda$ value from $0$ to $1$ with step $0.001$.
    Each row corresponds to decision support systems $\Ccal_{\lambda}$ with a different pre-trained classifier with average accuracies as in Figure~\ref{fig:acc_vs_harm_mnl_cal_sizes_110} above.
    The average accuracy achieved by the simulated human experts on their own is $0.771$.
    The results are averaged across $50$ random samplings of the test and calibration set.
    In both panels, $95\%$ confidence intervals are represented using shaded areas and always have width below $0.025$.
    }
    \label{fig:set-size-coverage-vs-harm-110}
    \vspace{-3mm}
\end{figure}

\clearpage
\newpage

\section{Experimental results on counterfactual monotonicity with a more complex set-valued predictor} \label{app:saps}
In this section, we experiment with a system that constructs the prediction set $\Ccal(x)$ using a more complex set-valued predictor~\cite{huang2023conformal}. 
This set-valued predictor first computes the following nonconformity score for each label value $y \in \Ycal$ using the softmax output of a pre-trained classifier $m_{y}(x) \in [0,1]$:
\begin{align}
    s(y, x, u; w) = \left\{ 
    \begin{array}{cc}
         u \cdot m_{y_{(1)}} (x),&  \text{if } y = y_{(1)}\\
         m_{y_{(1)}} + (i-2+u)\cdot w, & \text{if } y = y_{(i)} 
    \end{array}
    \right. , 
\end{align}
where  $y_{(i)}$ denotes the label value with the $i$-th largest classifier output, $u \sim U(0,1)$ is a uniform random variable and $w$ is a hyperparameter representing the weight of the ranking of the label value. 
Then, given a user-specified threshold $\lambda \in \RR$, it constructs the prediction sets $\Ccal(x) = \Ccal_{\lambda}(x)$ as follows: 
\begin{equation}
\label{eq:saps-set-valued-predictor}
    \Ccal_{\lambda}(x) = \{y_{(L-i)}\}_{i=0}^{k-1},\text{ with }k = 1 + \sum_{j=1}^{L-1} \mathbbm{1}\{s(y_{(j)}, x, u; w) \leq \lambda\},
\end{equation}
where, with a slight abuse of notation, $y_{(i)}$ is the label value with the $i$-th largest $s(y,x, u; w)$ value. Here, we should note that may exist data samples with $s(y,x, u; w) > 1$. Therefore, considering $\lambda \in [0,1]$ may not include every set-valued predictor given by Eq.~\ref{eq:saps-set-valued-predictor}. For this reason, in what follows, we consider $\lambda \in [0, \max s(y, x, u, w)]$, where the maximum is essentially over the classifier output $m_{y}(x)$, the uniform random variable $u$ and the hyperparameter $w$.  

\xhdr{Experimental setup} We use the ImageNet16H dataset and follow the same setup as described in Section~\ref{sec:experiments}, with only difference that we split the images into a calibration set (10\%), which we use to find $\Lambda(\alpha)$, a validation set $(10\%)$, which we use to select the value of the hyperparameter $w$ following a similar procedure as in Huang et al.~\cite{huang2023conformal}, and a test set (80\%). In each iteration of the experiment, given a $\lambda$ value, we select the value of $w \in \{0.02, 0.05, 0.1, 0.15, 0.2, 0.25, 0.3, 0.35\}$ for which the set-valued predictor achieves the smallest average prediction set-size over the data samples in the validation set. 

\xhdr{Results} 
%Figures~\ref{fig:saps_acc_vs_harm_mnl_noise80},~\ref{fig:saps_acc_vs_harm_mnl_noise95} and~\ref{fig:saps_acc_vs_harm_mnl} show the average accuracy $A(\lambda)$ achieved by a human expert using $\Ccal_{\lambda}$, as predicted by the mixture of MNLs, against the average counterfactual harm $H(\lambda)$ caused by $\Ccal_{\lambda}$ for $\alpha \in \{0.01, 0.05\}$ on the strata of images with $\omega = 80$, $\omega = 95$ and $\omega = 110$, respectively.
%
The meaning and coloring of each point are the same as in Figure~\ref{fig:acc_vs_harm_mnl}.
The results on all strata of images are consistent with the results using the set-valued predictor given by Eq.~\ref{eq:set-valued-predictor}---the decision support systems $\Ccal_{\lambda}$ always cause some amount of counterfactual harm, our framework successfully identifies harm-controlling sets and there is a trade-off between accuracy and counterfactual harm. 

Figure~\ref{fig:saps-set-size-coverage-vs-harm-110}(a) shows the average prediction set-size of each decision support system $\Ccal_{\lambda}$ against the average counterfactual harm $H(\lambda)$ caused by $\Ccal_{\lambda}$ on the stratum of images with $\omega=110$. 
Figure~\ref{fig:saps-set-size-coverage-vs-harm-110}(b) shows the empirical coverage for each decision support system $\Ccal_{\lambda}$ against the average counterfactual harm $H(\lambda)$ caused by $\Ccal_{\lambda}$ on the stratum of images with $\omega=110$.
Our results are consistent with the results for the set-valued predictor given by Eq.~\ref{eq:set-valued-predictor} in Figure~\ref{fig:set-size-coverage-vs-harm-110}---systems $\Ccal_{\lambda}$ that achieve higher coverage, construct larger prediction sets on average and cause less harm.

\begin{figure}[ht]
    \centering
    \subfloat[$\alpha = 0.01$]{
    % \label{fig:80class-alpha0.01}
    \includegraphics[width=.45\linewidth, valign=t]{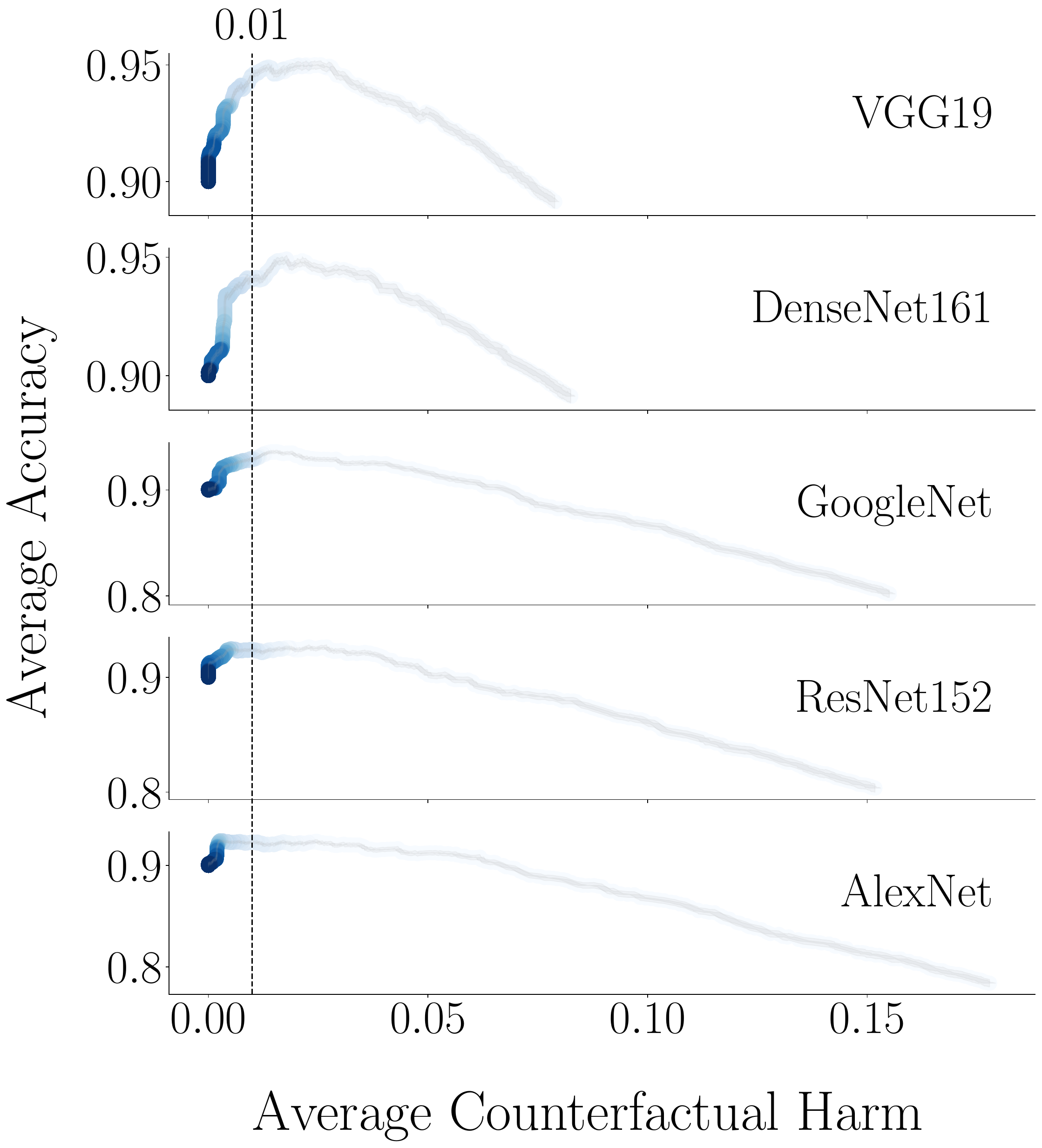}
    }
    \hspace{-0.5mm}
    \subfloat[$\alpha=0.05$]{
    % \label{fig:80class-alpha0.05}
    \includegraphics[width=.45\linewidth, valign=t]{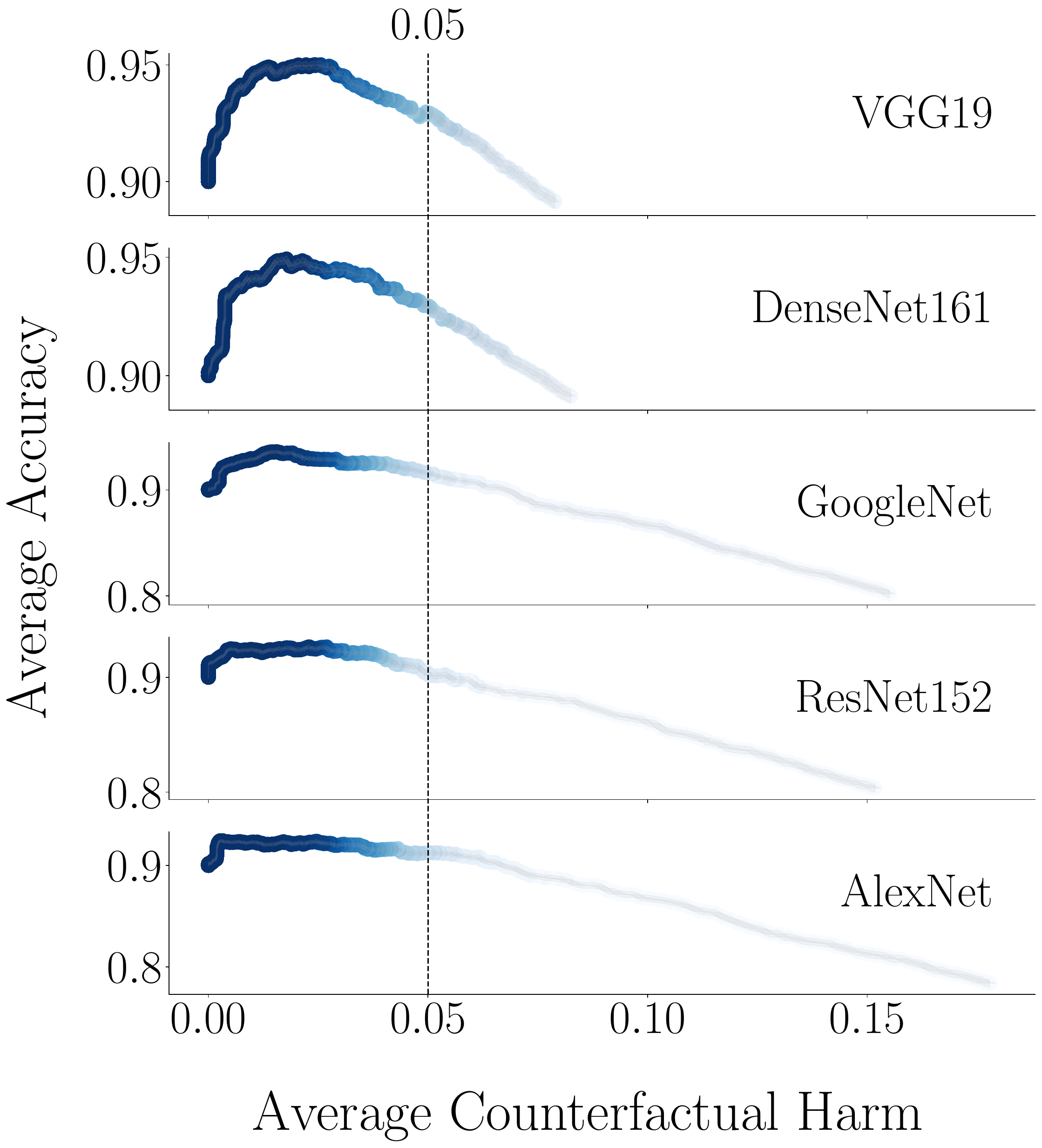}}
    \includegraphics[scale=0.202, valign=t]{figures/colorbar_big.pdf}
     \caption{Average accuracy estimated by the mixture of MNLs against the average counterfactual harm for images with $\omega = 80$.
    Each point corresponds to a $\lambda$ value from $0$ to $6.25$ with step $0.00625$ and the coloring indicates the relative frequency with which each $\lambda$ value is in $\Lambda(\alpha)$ across random samplings of the calibration set.
    Each row corresponds to decision support systems $\Ccal_{\lambda}$ with a different pre-trained classifier with average accuracies   $0.891$~(VGG19),  $0.892$~(DenseNet161),  $0.802$~(GoogleNet), $0.804$~(ResNet152), and $0.784$~(AlexNet).
    The average accuracy of human experts on their own is $0.9$.
    The results are averaged across $50$ random samplings of the test and calibration set.
    In both panels, $95\%$ confidence intervals have width always below $0.02$ and are represented using shaded areas.
    }
    \label{fig:saps_acc_vs_harm_mnl_noise80}
    \vspace{-3mm}
\end{figure}
\begin{figure}[t!]
    \centering
    \subfloat[$\alpha = 0.01$]{
    % \label{fig:class-alpha0.01}
    \includegraphics[width=.45\linewidth, valign=t]{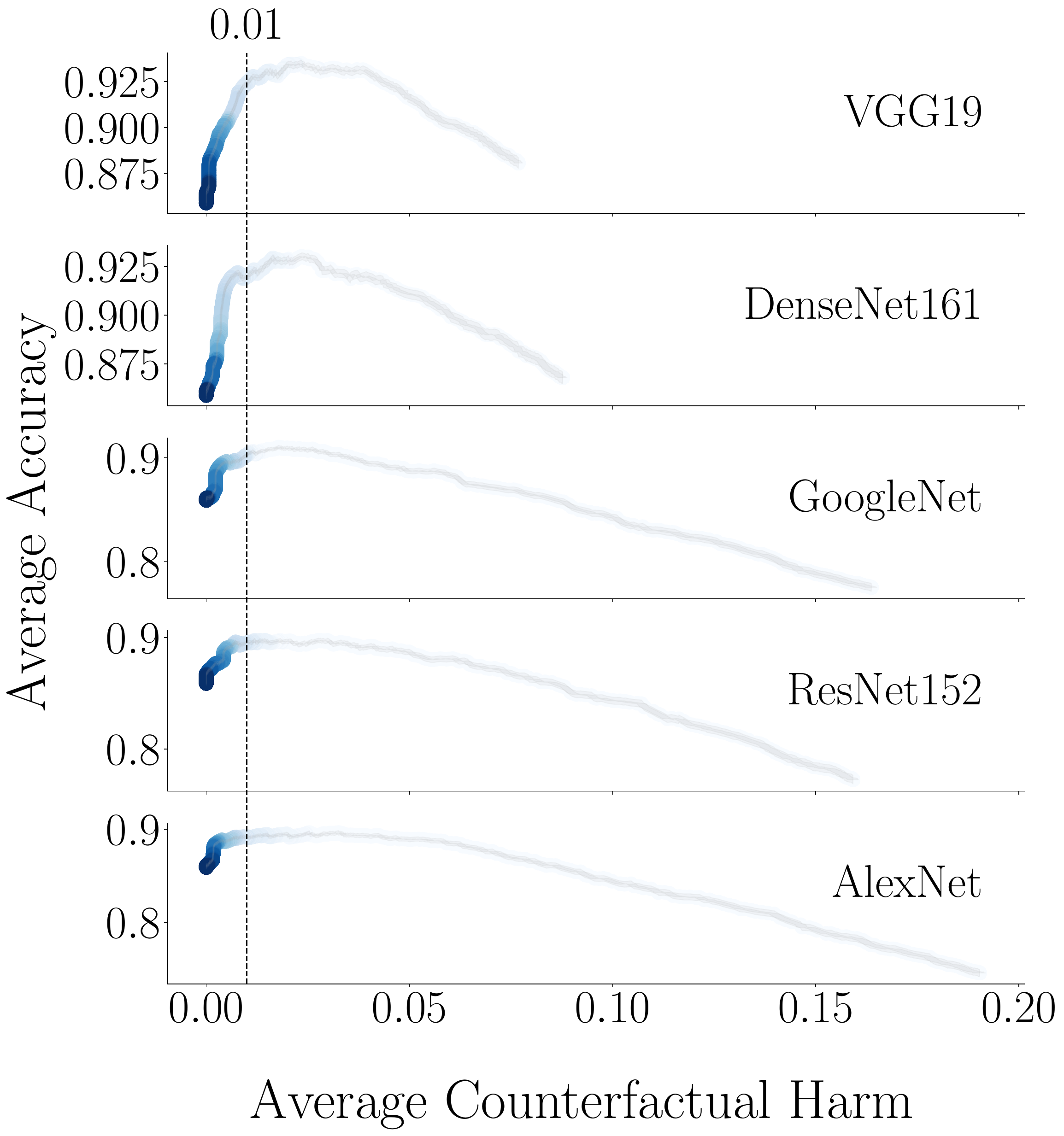}
    }
    \hspace{-0.5mm}
    \subfloat[$\alpha=0.05$]{
    % \label{fig:class-alpha0.05}
    \includegraphics[width=.45\linewidth, valign=t]{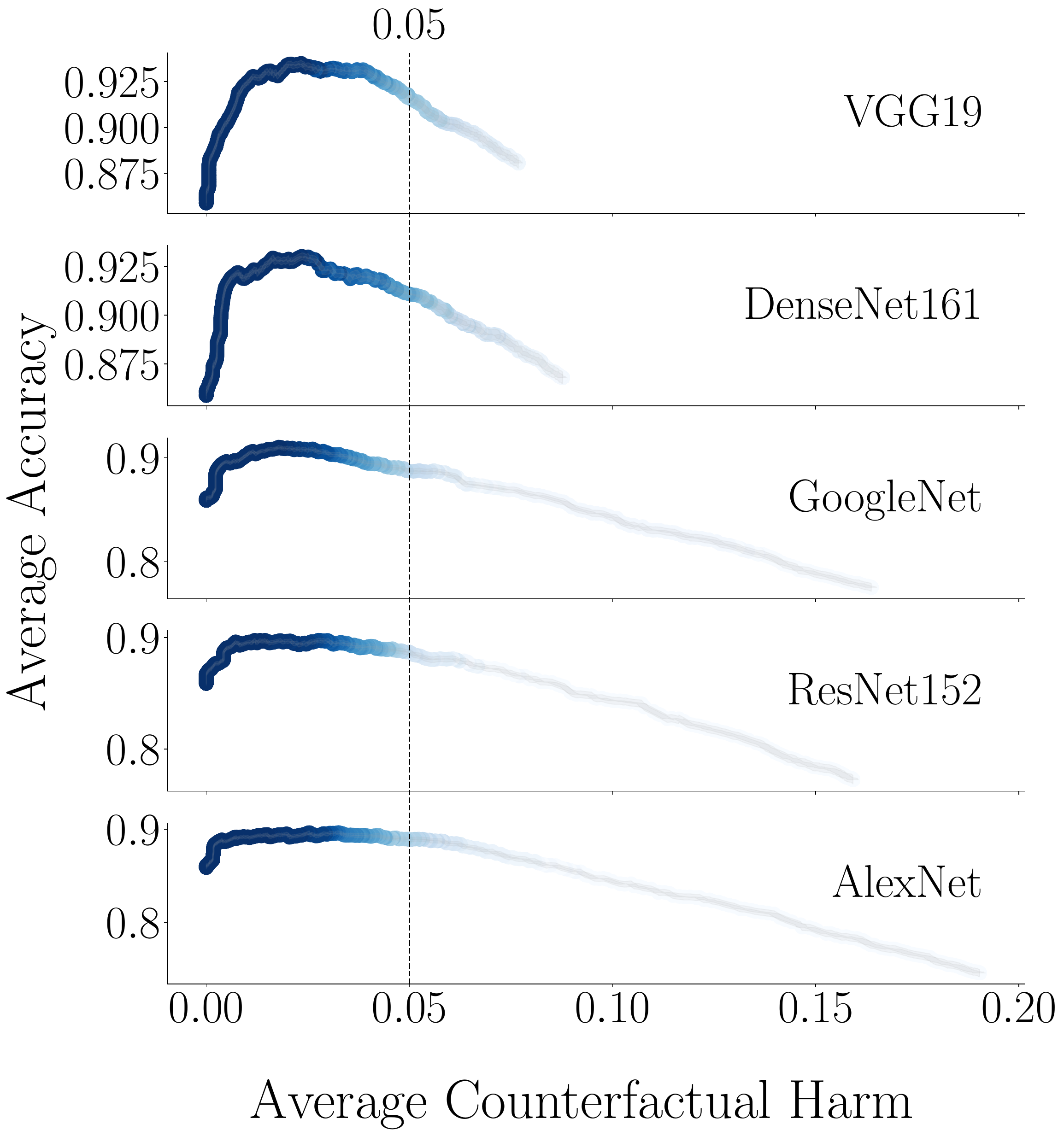}}
    \includegraphics[scale=0.195, valign=t]{figures/colorbar_big.pdf}
 \caption{Average accuracy estimated by the mixture of MNLs against the average counterfactual harm for images with $\omega = 95$.
    Each point corresponds to a $\lambda$ value from $0$ to $6.25$ with step $0.00625$ and the coloring indicates the relative frequency with which each $\lambda$ value is in $\Lambda(\alpha)$ across random samplings of the calibration set.
    Each row corresponds to decision support systems $\Ccal_{\lambda}$ with a different pre-trained classifier with average accuracies   $0.88$~(VGG19),  $0.868$~(DenseNet161),  $0.775$~(GoogleNet), $0.773$~(ResNet152), and $0.745$~(AlexNet).
    The average accuracy of human experts on their own is $0.86$.
    The results are averaged across $50$ random samplings of the test and calibration set.
    In both panels, $95\%$ confidence intervals have width always below $0.02$ and are represented using shaded areas.
    }
    \label{fig:saps_acc_vs_harm_mnl_noise95}
    \vspace{-3mm}
\end{figure}

\begin{figure}[t]
    \centering
    \subfloat[$\alpha = 0.01$]{
    \label{fig:saps_class-alpha0.01}
    \includegraphics[width=.45\linewidth, valign=t]{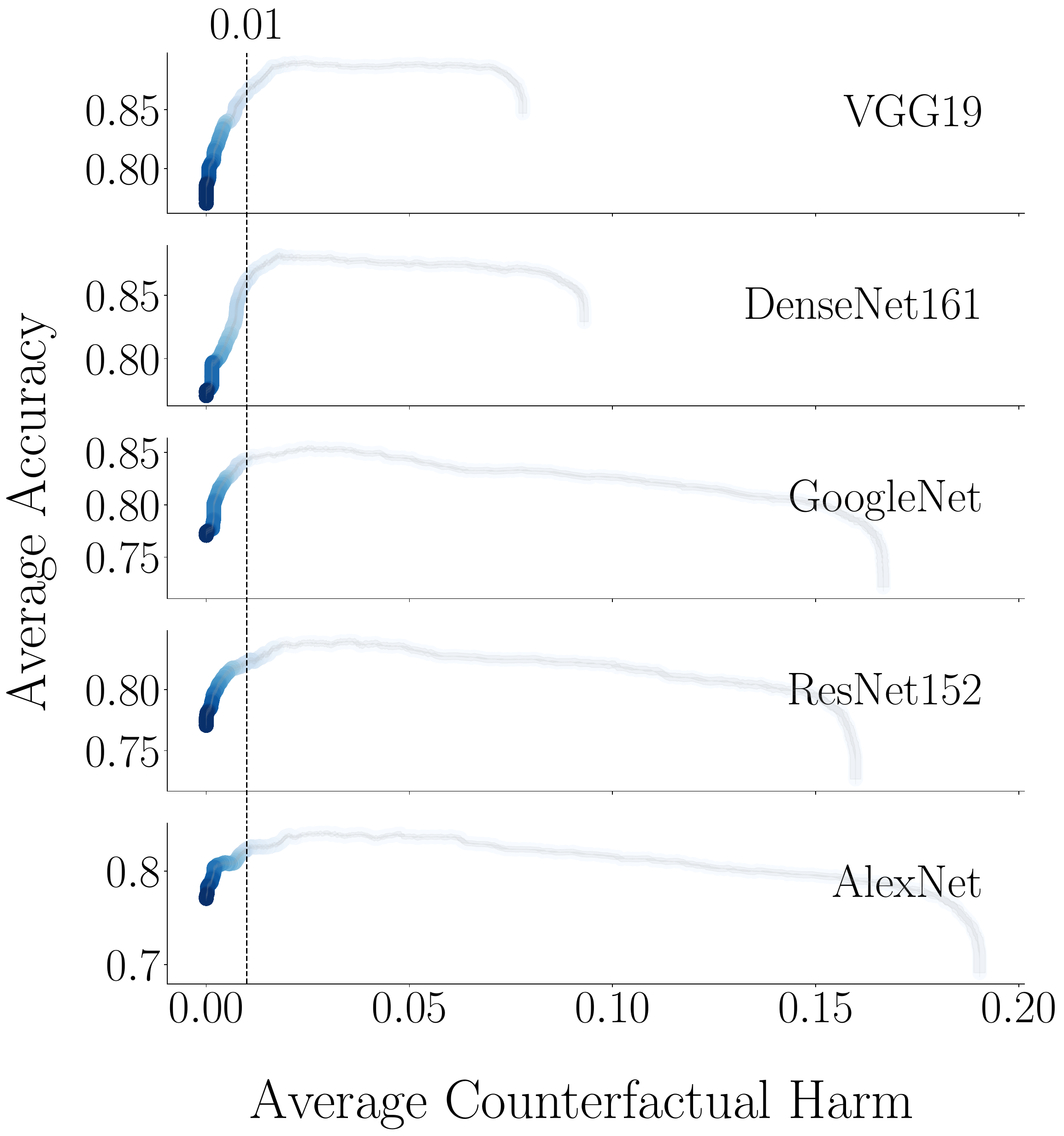}
    }
    \hspace{-0.5mm}
    \subfloat[$\alpha=0.05$]{
    \label{fig:saps_class-alpha0.05}
    \includegraphics[width=.45\linewidth, valign=t]{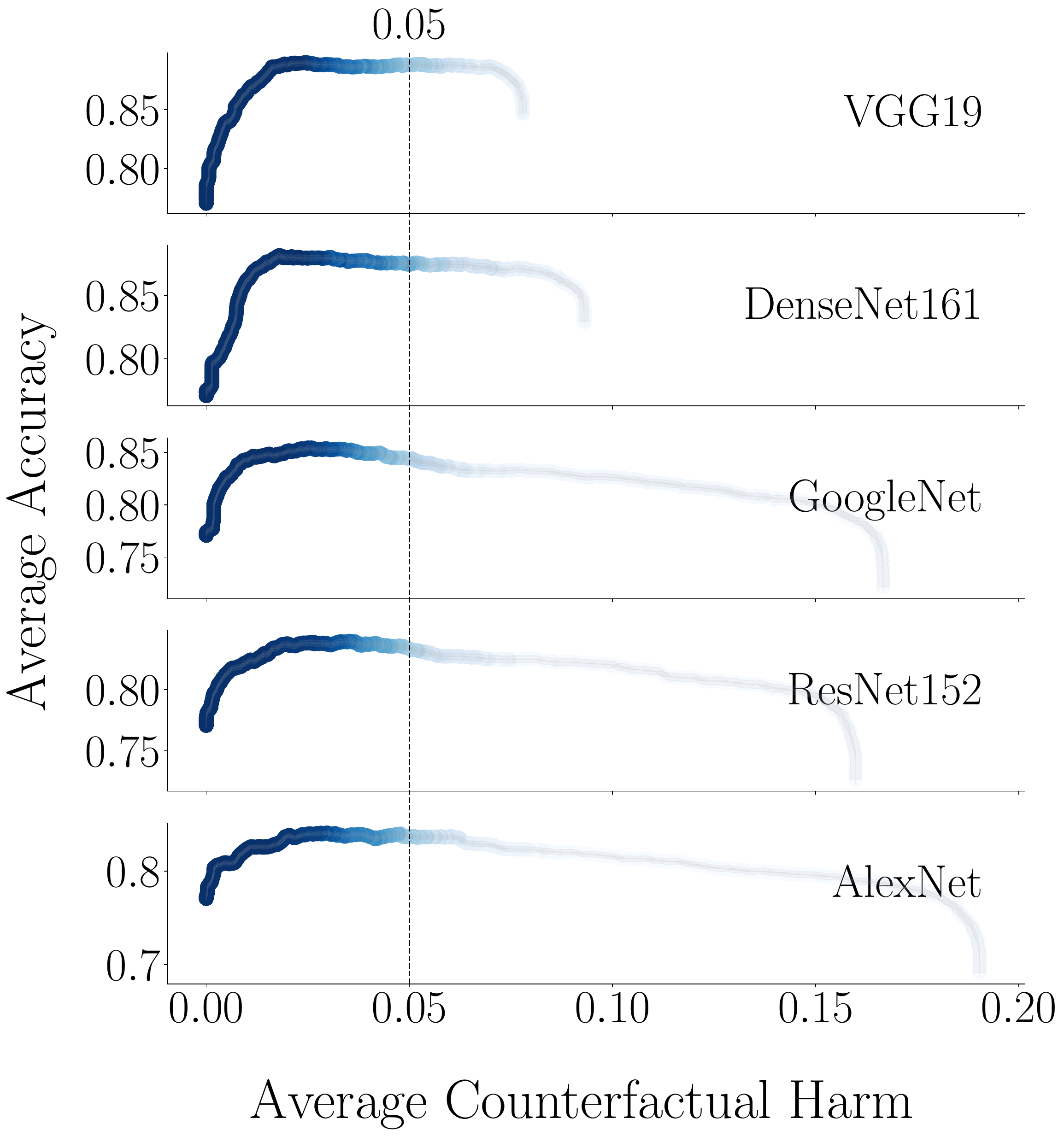}}
    \includegraphics[scale=0.195, valign=t]{figures/colorbar_big.pdf}
    \caption{Average accuracy estimated by the mixture of MNLs against the average counterfactual harm for images with $\omega = 110$.
    Each point corresponds to a $\lambda$ value from $0$ to $6.25$ with step $0.00625$ and the coloring indicates the relative frequency with which each $\lambda$ value is in $\Lambda(\alpha)$ across random samplings of the calibration set.
    Each row corresponds to decision support systems $\Ccal_{\lambda}$ with a different pre-trained classifier with average accuracies   \vgg~(VGG19),  \densenet~(DenseNet),  \googlenet~(GoogleNet), \resnet~(ResNet152), and \alexnet~(AlexNet).
    The average accuracy achieved by the simulated human experts on their own is $0.771$.
    The results are averaged across $50$ random samplings of the test and calibration set.
    In both panels, $95\%$ confidence intervals are represented using shaded areas and always have width below $0.02$.
    }
    \label{fig:saps_acc_vs_harm_mnl}
    \vspace{-3mm}
\end{figure}
\begin{figure}[ht]
    \centering
    \subfloat[Average Set-Size]{
    \label{fig:saps-set-size-110}
    \includegraphics[width=.45\linewidth, valign=t]{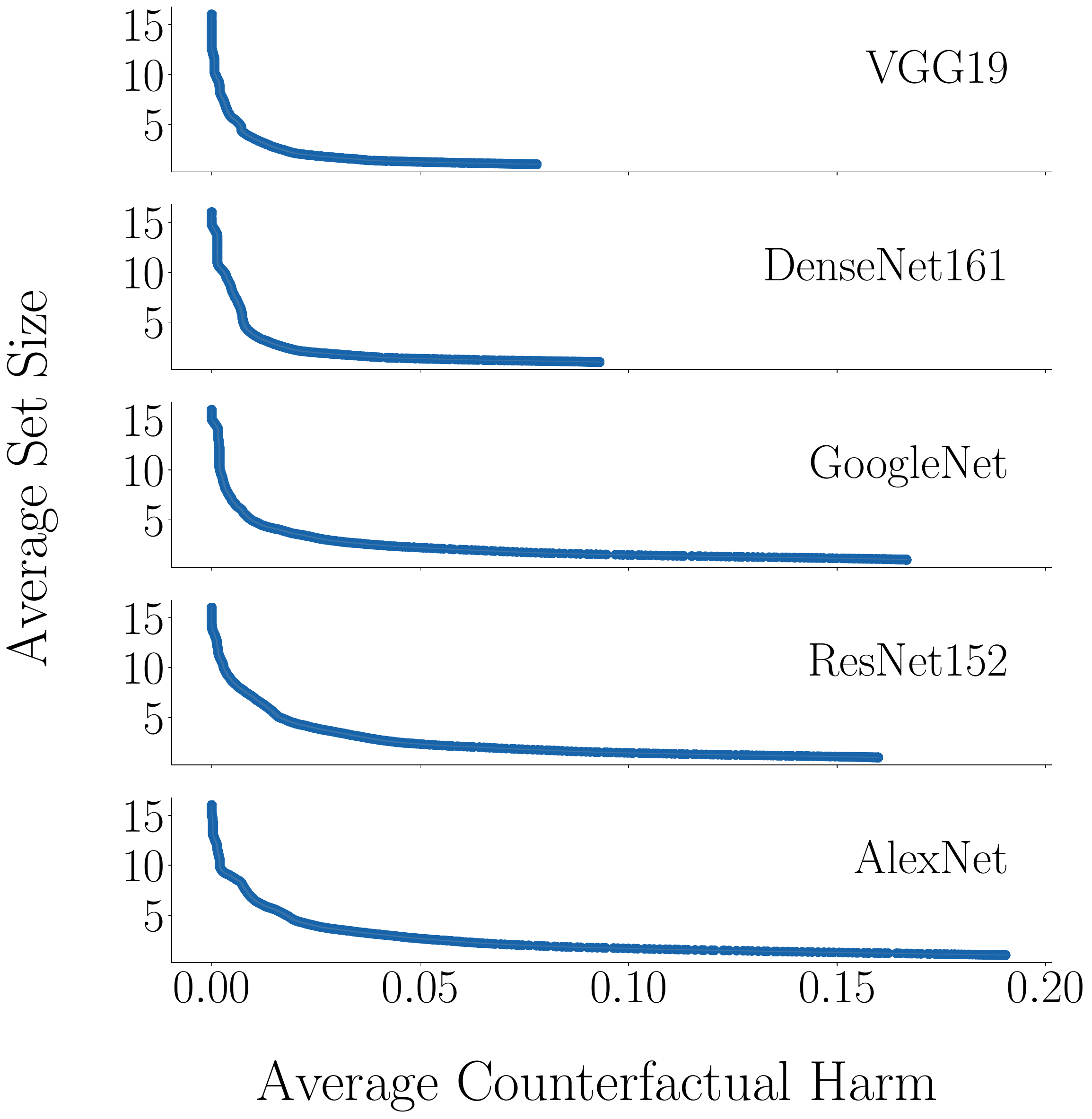}
    }
    % \hspace{-0.5mm}
    \subfloat[Empirical Coverage]{
    \label{fig:saps-coverage-110}
    \includegraphics[width=.45\linewidth, valign=t]{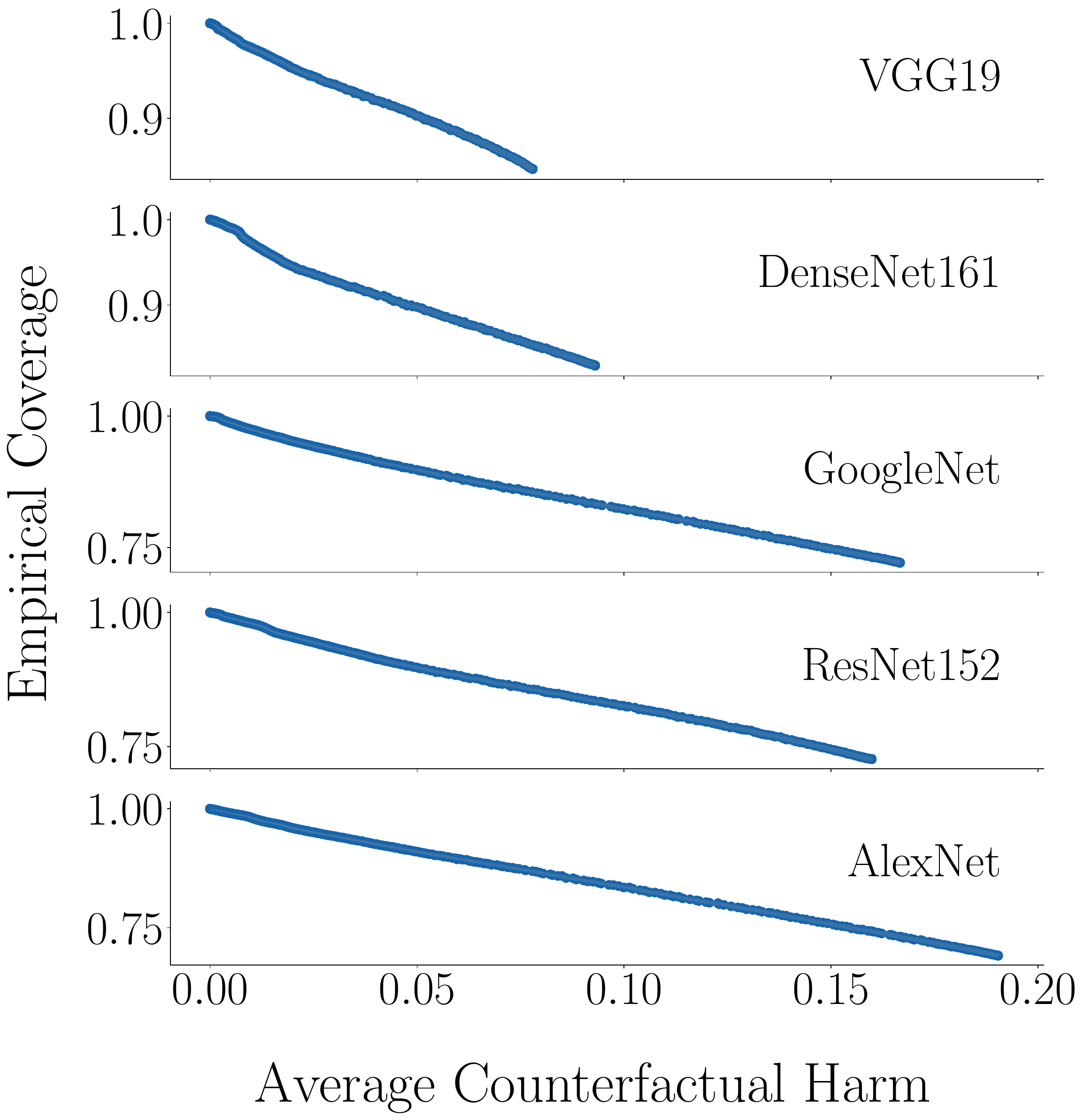}}
    \caption{Average set size and empirical coverage against the average counterfactual harm for images with $\omega = 110$.
    Each point corresponds to a $\lambda$ value from $0$ to $6.25$ with step $0.00625$.
    Each row corresponds to decision support systems $\Ccal_{\lambda}$ with a different pre-trained classifier with average accuracies as in Figure~\ref{fig:saps_acc_vs_harm_mnl} above.
    The average accuracy achieved by the simulated human experts on their own is $0.771$.
    The results are averaged across $50$ random samplings of the test and calibration set.
    In both panels, $95\%$ confidence intervals are represented using shaded areas and always have width below $0.025$.
    }
    \label{fig:saps-set-size-coverage-vs-harm-110}
    \vspace{-3mm}
\end{figure}

\clearpage
\newpage

\section{Average accuracy vs. prediction set size}\label{app:interventional}

In this section, we use the ImageNet16H-PS~\cite{straitouri2024designing} dataset to verify the interventional monotonicity assumption, \ie, whether human experts achieve higher average accuracy
%whether human experts have a higher probability of success at predicting the ground truth label 
when predicting from smaller prediction sets that include the ground truth label.
To this end, we follow a procedure, similar to the one used by Straitouri et al.~\cite{straitouri2024designing} to verify the interventional monotonicity assumption unconditionally of the value of the ground truth label.

Straitouri et al. estimate the average accuracy per prediction set size on images with similar difficulty, averaged across all experts and across experts with the same level of competence.   
They consider images with similar  difficulty instead of single images, as the dataset ImageNet16H-PS does not include enough human predictions per image to faithfully estimate the average accuracy per image. 
They stratify the images into groups of similar difficulty, following the same procedure used in Straitouri et al.~\cite{straitouri2023improving} based on different quantiles of the average accuracy values of all images.   
% To distinguish the images of similar difficulty, they follow the same procedure  where the level of difficulty corresponds to the quantile of the empirical success probabilities achieved by experts on all images.
% 
In our work, we stratify images with the same ground truth label into  the following four groups:
\begin{itemize}
    \item[---] High difficulty: images with average accuracy within the $0.25$ quantile of the average accuracy values of all images with the same ground truth label. %, as images with the highest difficulty.
    \item[---] Medium to high difficulty: images with average accuracy within the $0.5$ quantile and outside the $0.25$ quantile of the average accuracy values of all images with the same ground truth label.
     \item[---] Medium to low difficulty: images with average accuracy within the $0.75$ quantile and outside the $0.5$ quantile the of the average accuracy values of all images with the same ground truth label. % , as images with medium  difficulty.
     \item[---] Low difficulty: images with average accuracy outside the $0.75$ quantile of the average accuracy values of all images with the same ground truth label. %, as images 
\end{itemize}
% [label=(\roman*),leftmargin=0.8cm,noitemsep,nolistsep]
%     \item[---] High difficulty: images with empirical success probability within the $0.25$ quantile of the empirical success probabilities of all images. %, as images with the highest difficulty.
%     \item[---] Medium to high difficulty: images with empirical success probability within the $0.5$ quantile and outside the $0.25$ qunatile of the empirical success probabilities of all images.
%      \item[---] Medium to low difficulty: images with empirical success probability within the $0.75$ quantile and outside the $0.5$ quantile the of the empirical success probabilities of all images. % , as images with medium  difficulty.
%      \item[---] Low difficulty: images with empirical success probability outside the $0.75$ quantile of the empirical success probabilities of all images. %, as images with the lowest difficulty.
% \end{enumerate}
% 
% For example the group of the most difficult images contains images for which the empirical success probability is within the $20\%$ percentile of all the empirical success probabilities, the group of the second most difficult images contains images for which the empirical success probability is above the $20\%$ percentile and within the $40\%$ percentile, and so on. 
% 
We follow a similar method to stratify experts based on their level of competence into two groups for each ground truth label.
To measure the level of competence of an expert for a ground truth label, we use the average accuracy across all the predictions that she made on images with this ground truth label.
% As a result 
For each ground truth label,
we consider the $50\%$ of experts with the highest average accuracy as the experts with high level of competence and the rest $50\%$ as experts with low level of competence.

Figure~\ref{fig:acc-set-size-across-experts} shows the average accuracy per prediction set size---for prediction sets that include the ground truth label---across images of various levels of difficulty and experts with different levels of competence for some of the ground truth labels. We find qualitatively similar results for the rest of the ground truth labels. 
The results are consistent with the results on the unconditional interventional monotonicity by Straitouri et al.~\cite{straitouri2024designing}, showing that as long as the classification task is not too easy, the experts  achieve higher average accuracy when predicting from smaller prediction sets, \ie, the interventional monotonicity holds.  
% images with  different  for different ground truth labels

\begin{figure}[ht]
    \centering
    \subfloat[All experts, ground truth label \texttt{bottle}]
    {
    \begin{tabular}{@{}ccc@{}}
    \includegraphics[width=.3\linewidth]{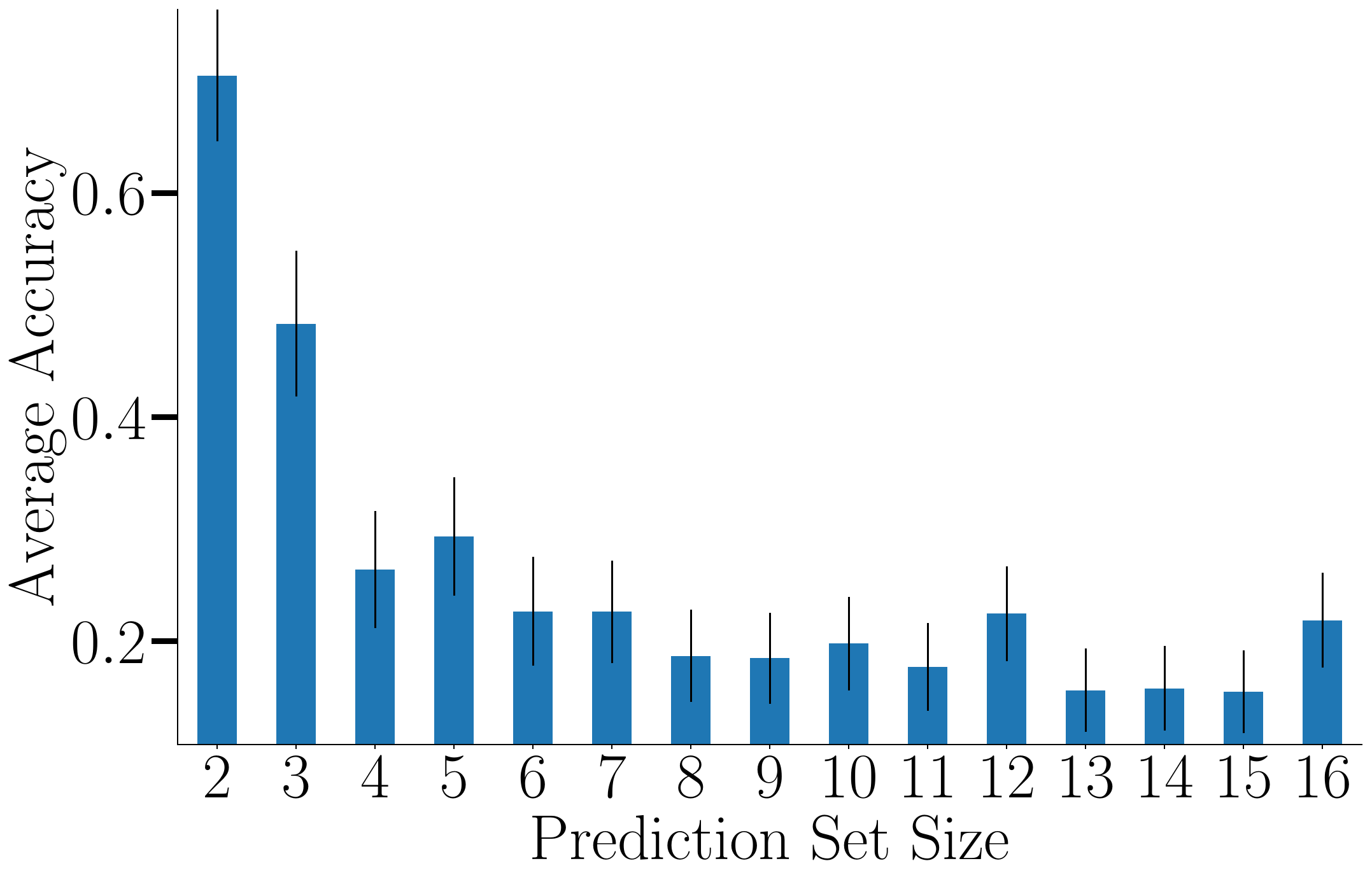}
    &
    \includegraphics[width=.3\linewidth]{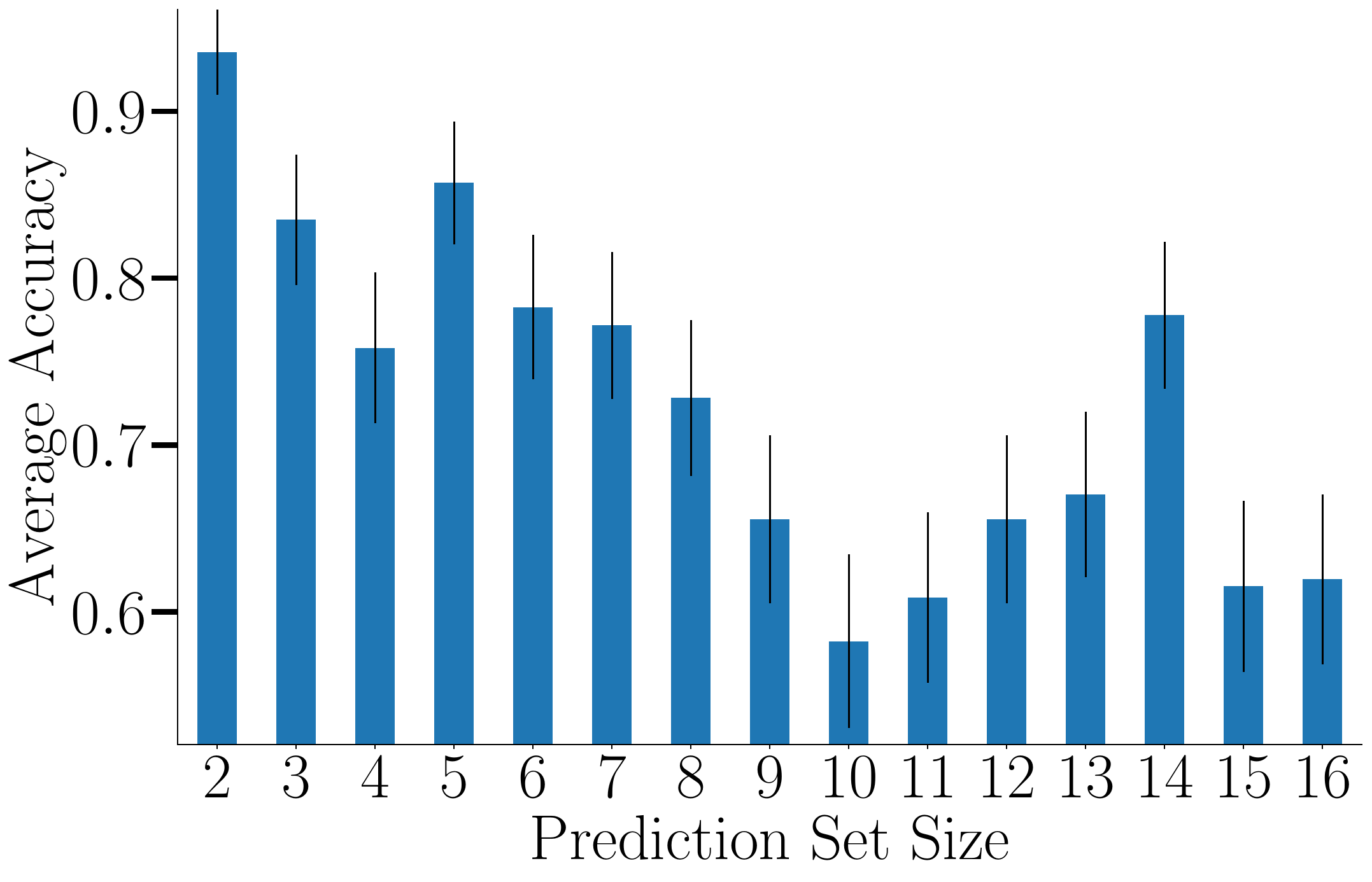}
    & 
    \includegraphics[width=.3\linewidth]{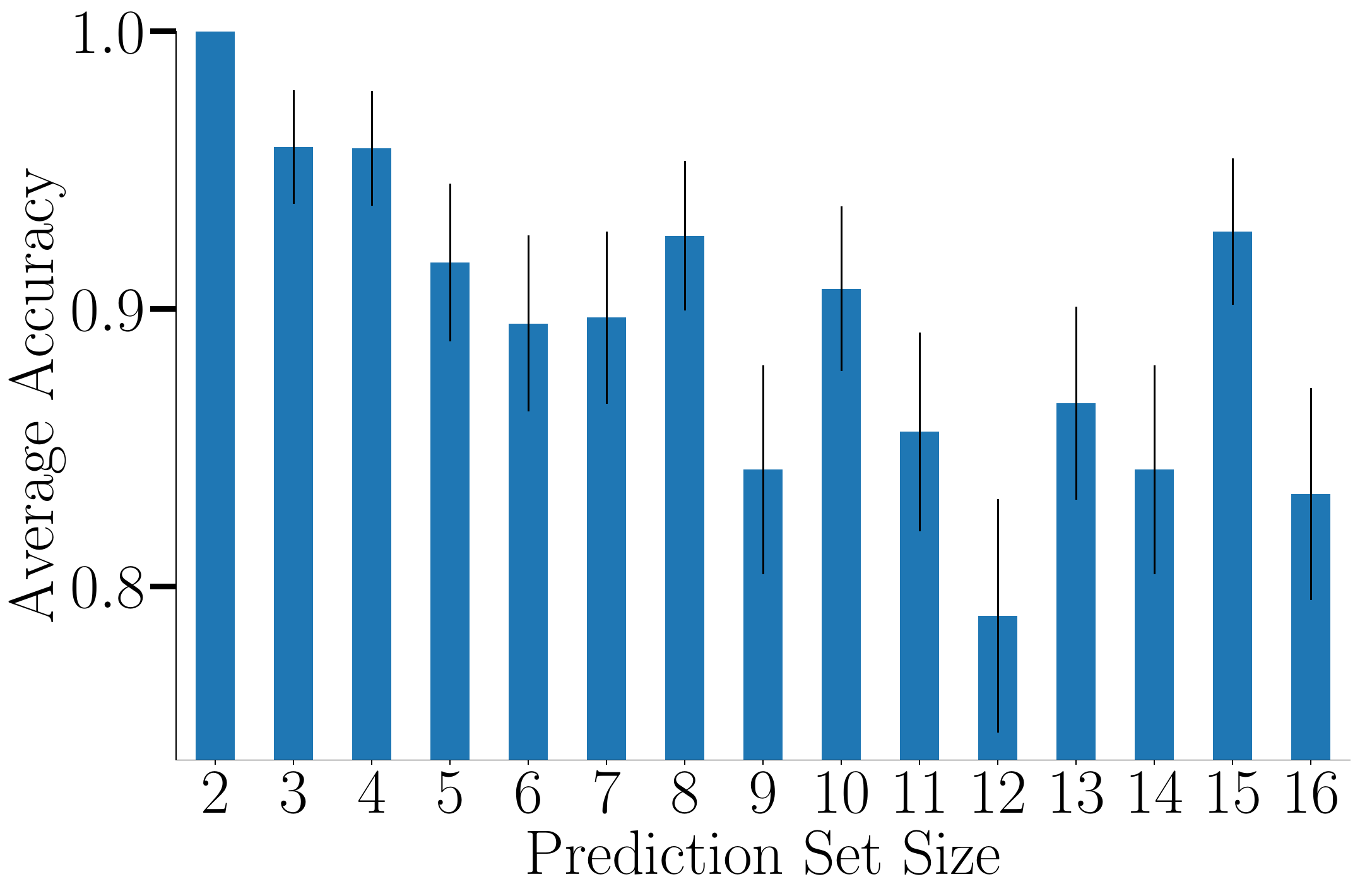}
    \\
    \small High difficulty& \small Medium to high difficulty & \small Medium to low difficulty \\
    \end{tabular}
    } \\
    \subfloat[All experts, ground truth label \texttt{oven}]
    {
    \begin{tabular}{@{}ccc@{}}
    \includegraphics[width=.3\linewidth]{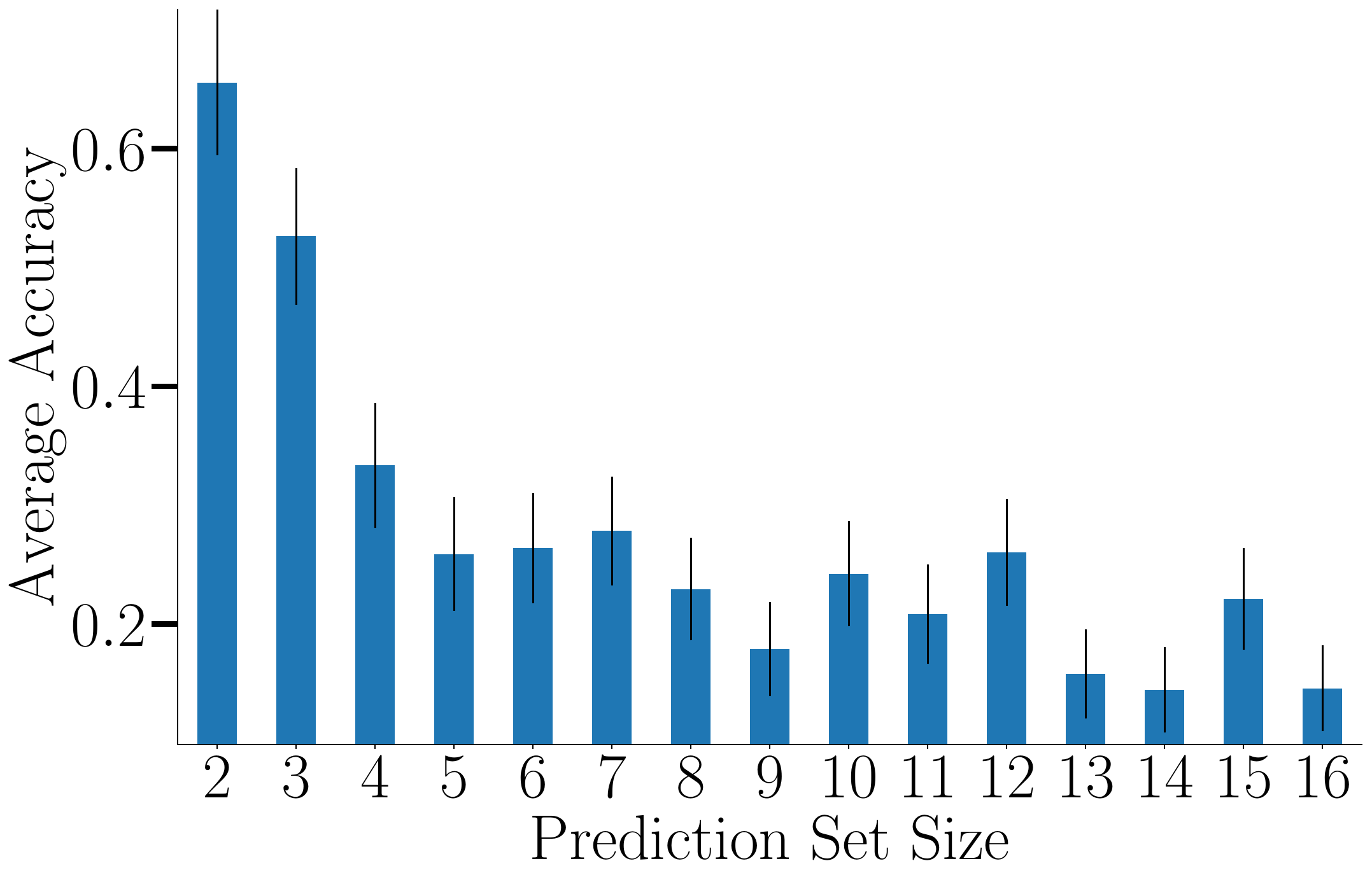}
    &
    \includegraphics[width=.3\linewidth]{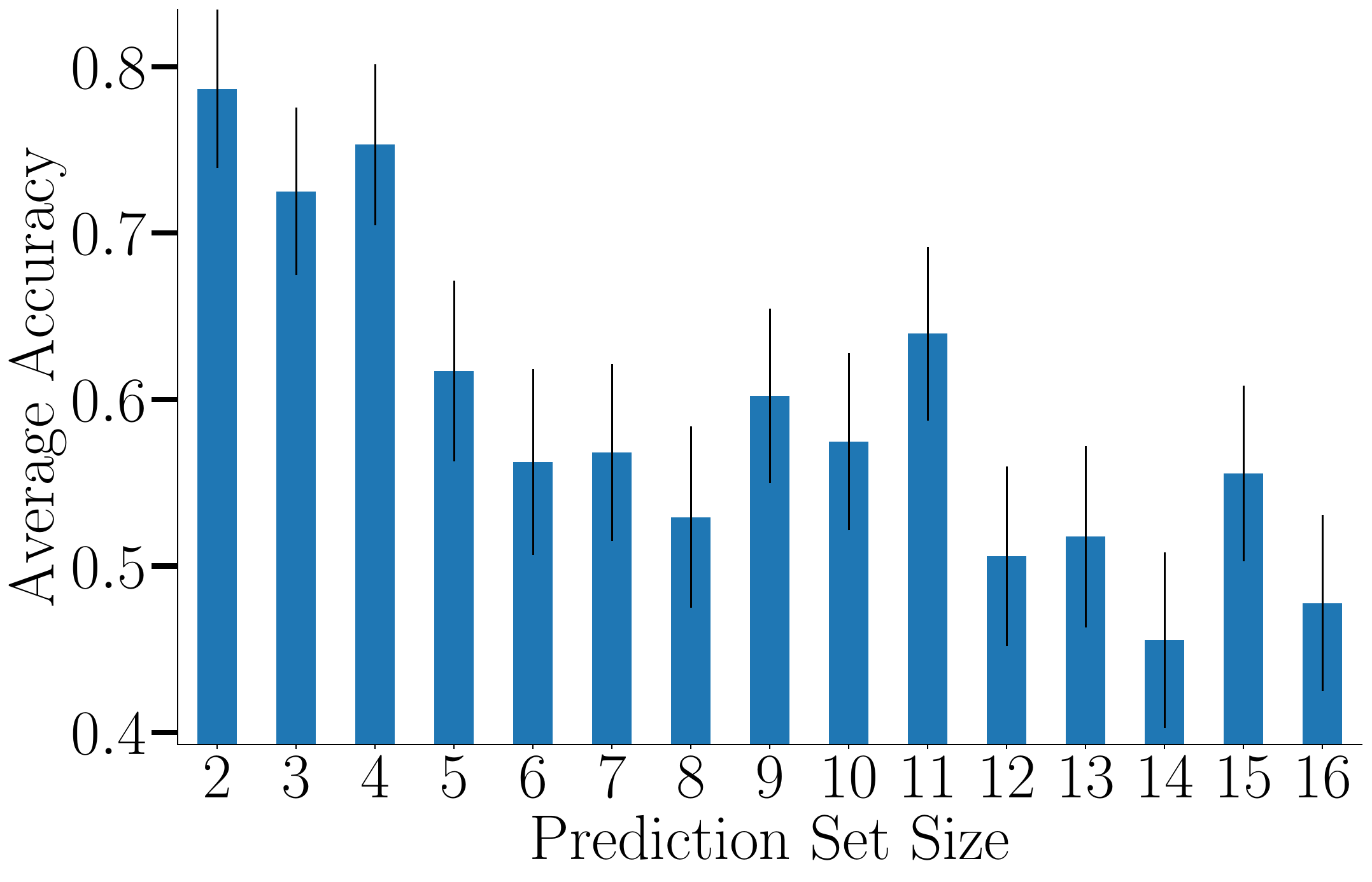}
    & 
    \includegraphics[width=.3\linewidth]{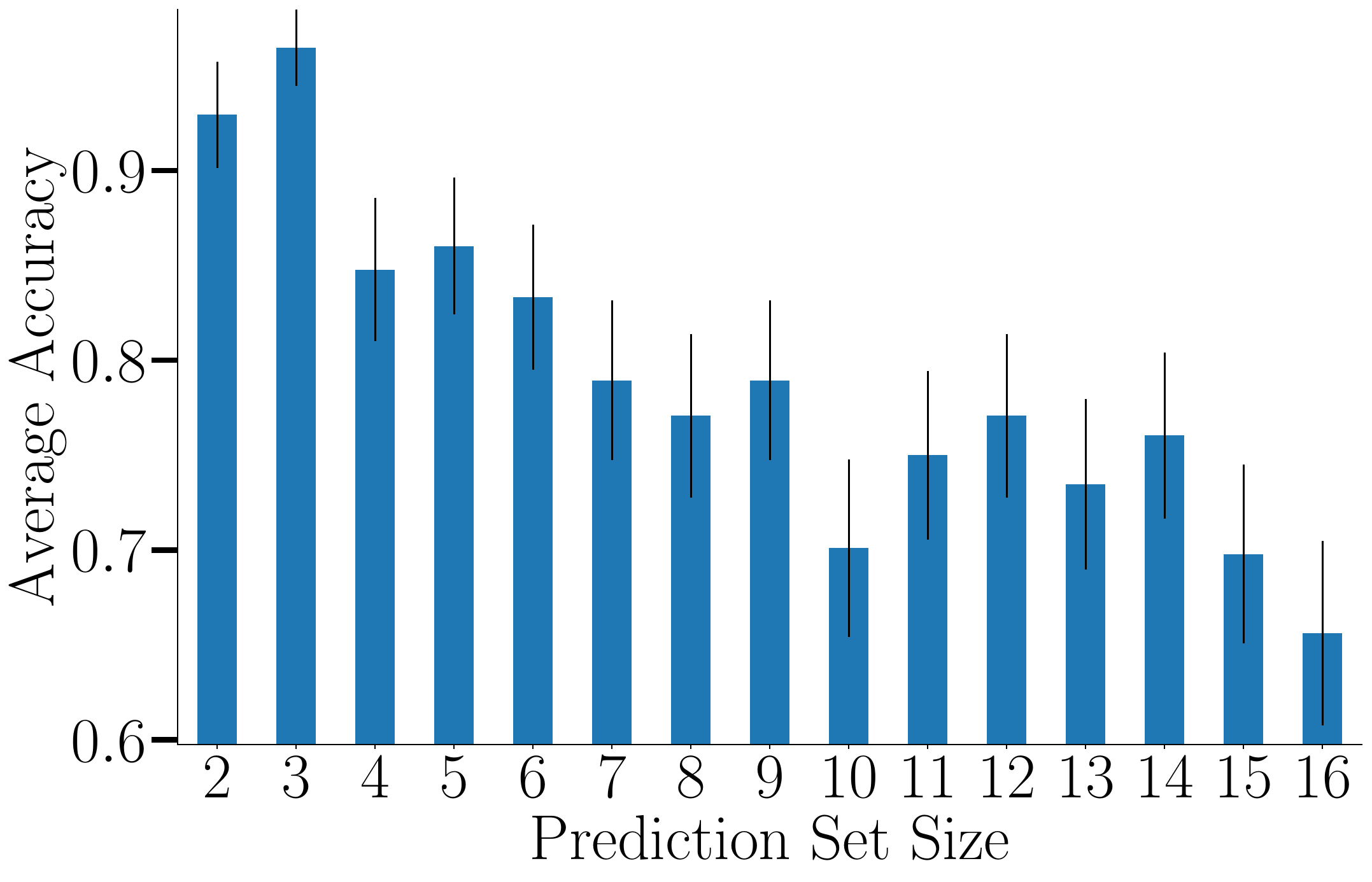}
    \\
    \small High difficulty& \small Medium to high difficulty & \small Medium to low difficulty \\
    \end{tabular}
    }\\
    \subfloat[Experts with low level of competence, ground truth label \texttt{bottle}]
    {
    \begin{tabular}{@{}ccc@{}}
    \includegraphics[width=.3\linewidth]{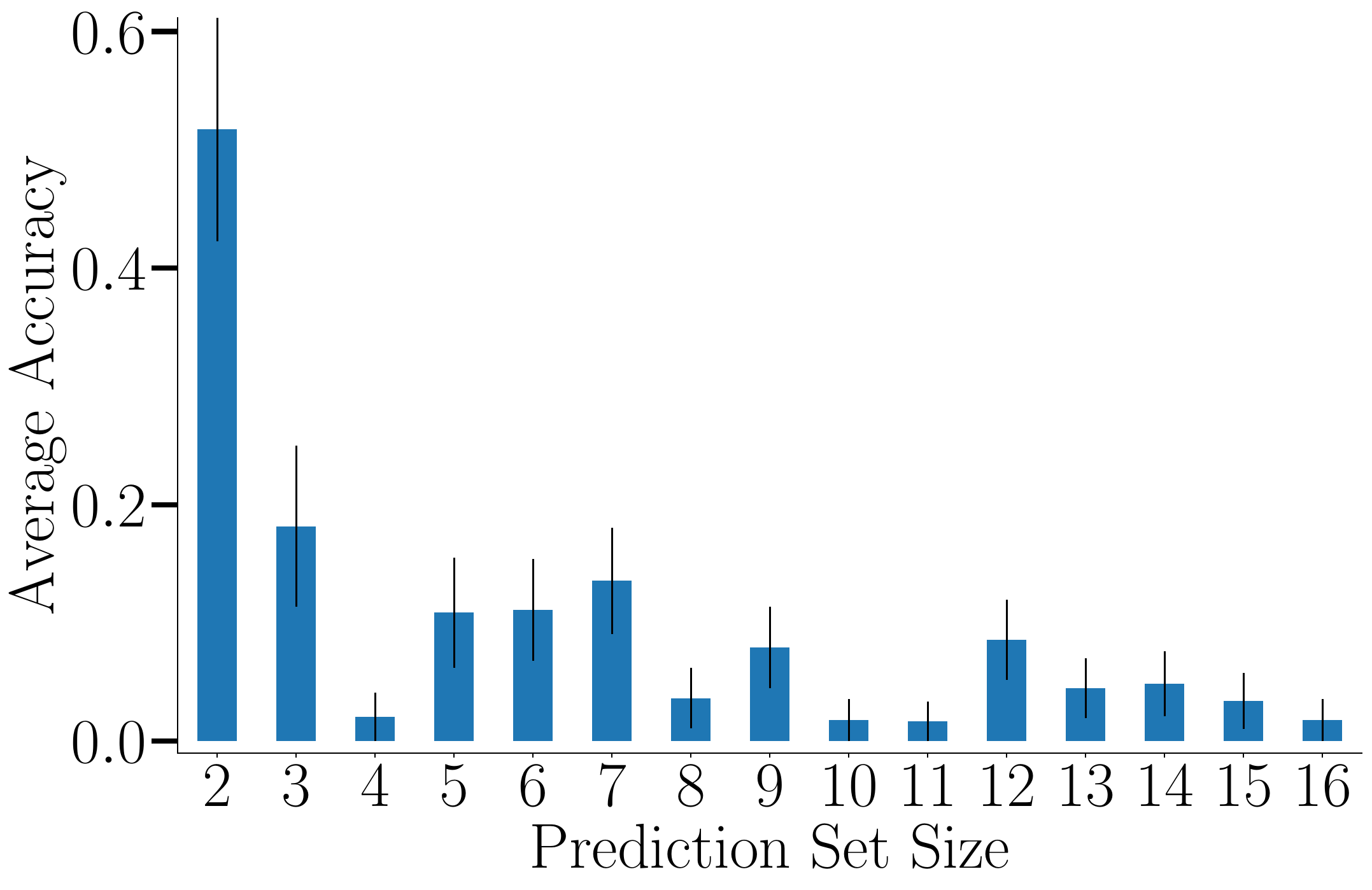}
    &
    \includegraphics[width=.3\linewidth]{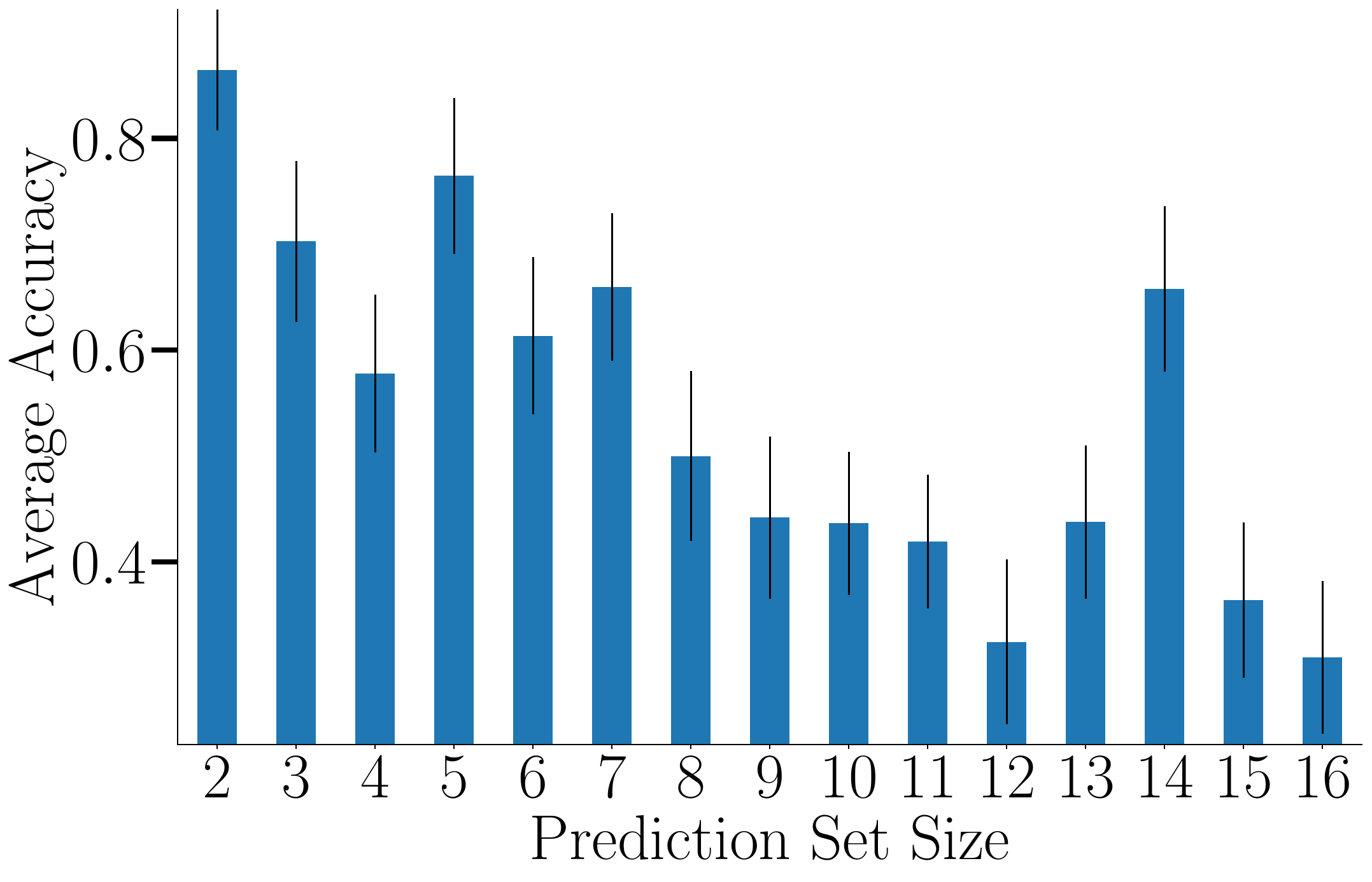}
    & 
    \includegraphics[width=.3\linewidth]{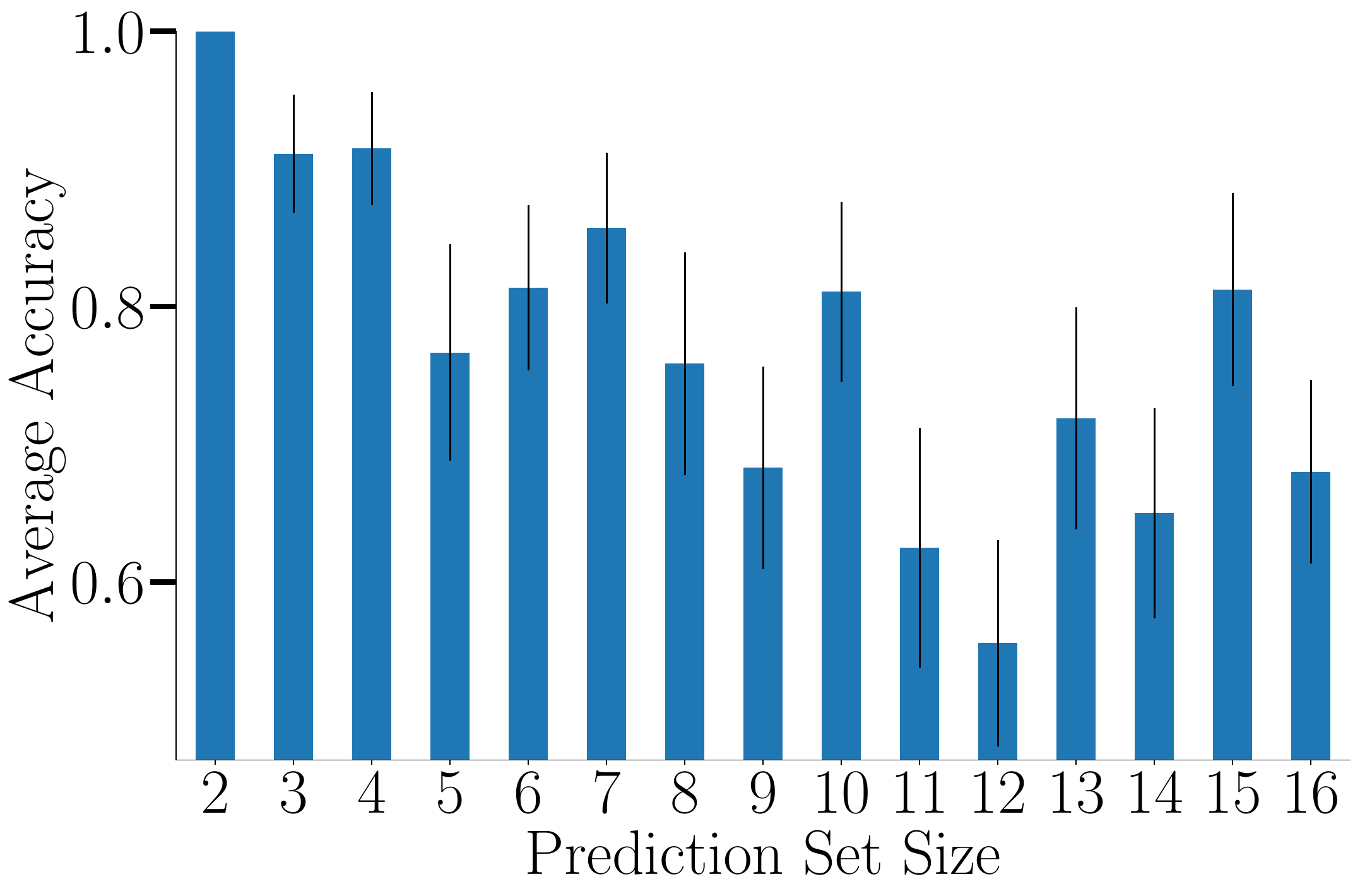}
    \\
    \small High difficulty & \small Medium to high difficulty & \small Medium to low difficulty \\
    \end{tabular}
    } \\
    \subfloat[Experts with high level of competence, ground truth label \texttt{bottle}]
    {
    \begin{tabular}{@{}ccc@{}}
    \includegraphics[width=.3\linewidth]{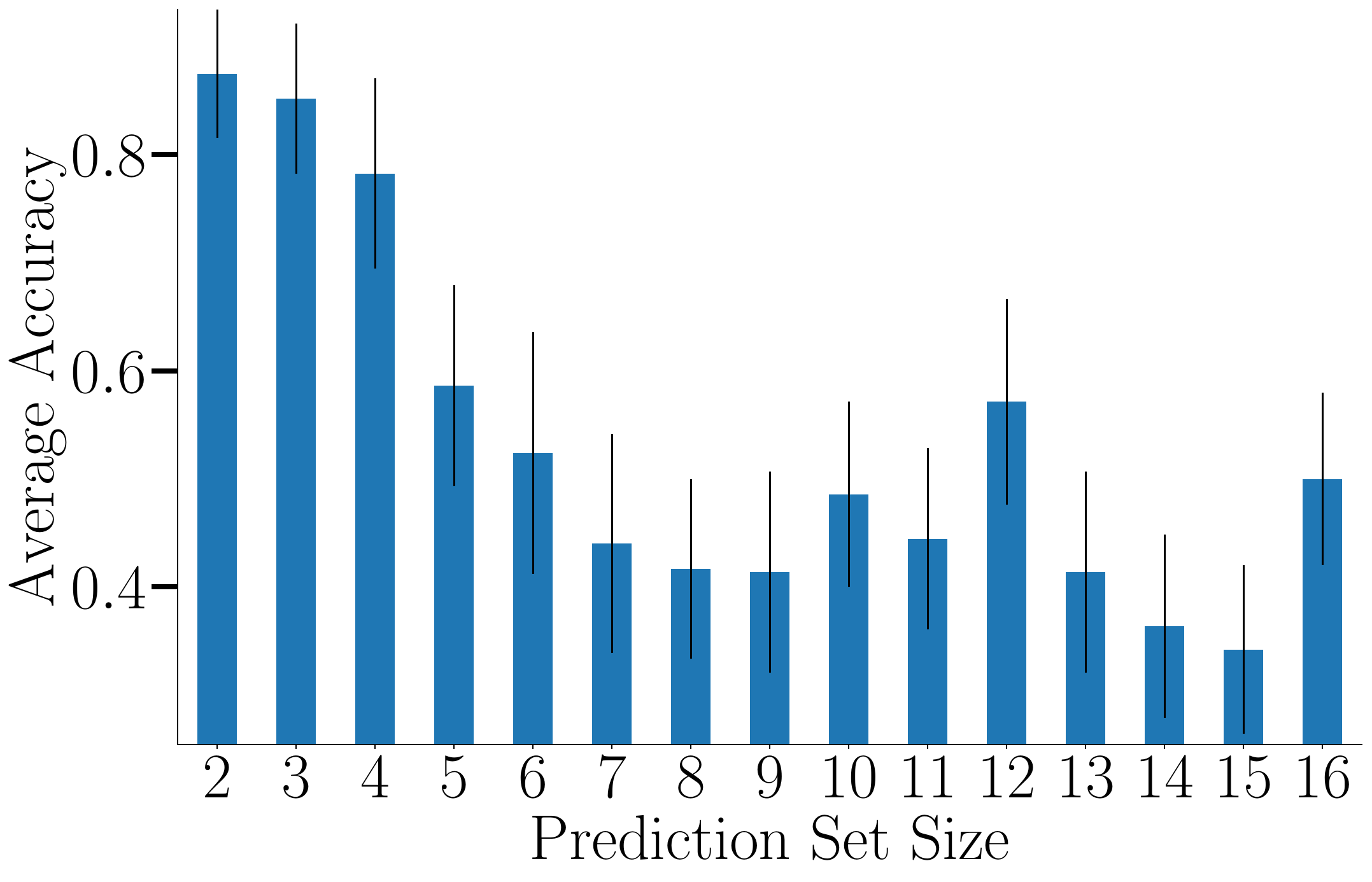}
    &
    \includegraphics[width=.3\linewidth]{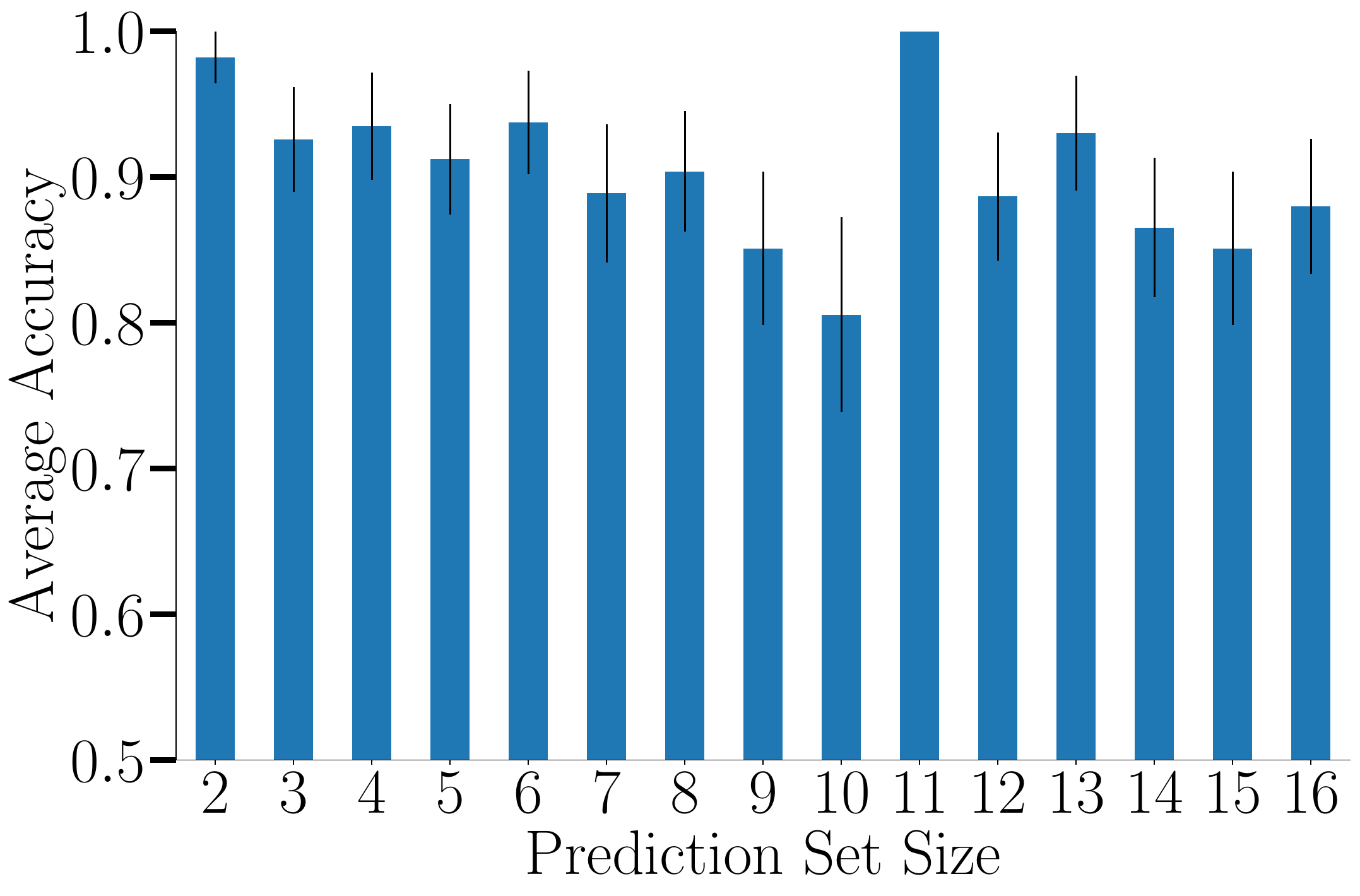}
    & 
    \includegraphics[width=.3\linewidth]{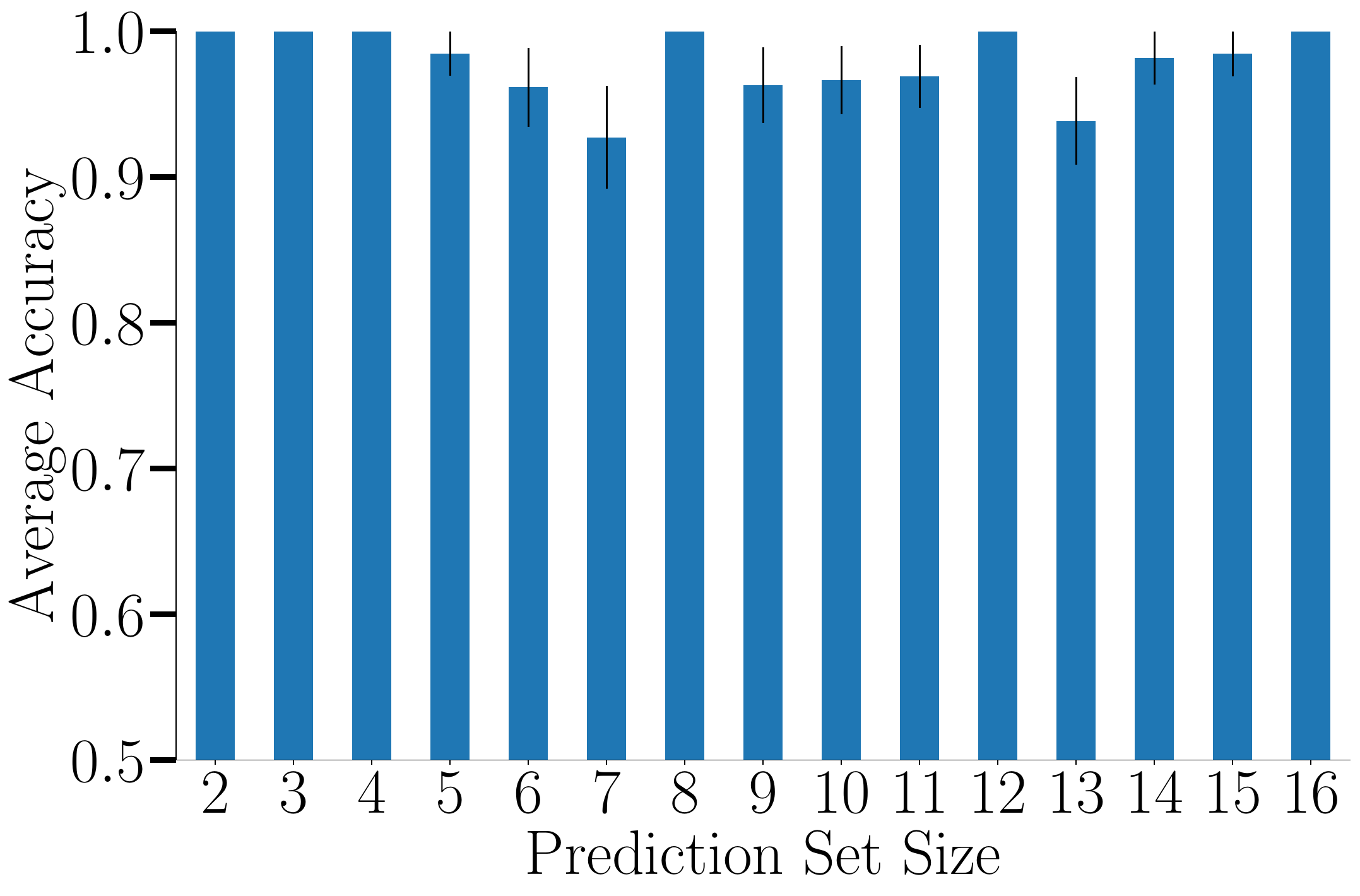}
    \\
    \small High difficulty & \small Medium to high difficulty & \small Medium to low difficulty \\
    \end{tabular}
    } \\
    % \subfloat[Experts with high level of competence]
    % {
    % \begin{tabular}{@{}ccc@{}}
    % \includegraphics[width=.3\linewidth]{figures/monotonicity/workers2/images5/workers_2_images_1_nosingleton.pdf}
    % &
    % \includegraphics[width=.3\linewidth]{figures/monotonicity/workers2/images5/workers_2_images_2_nosingleton.pdf}
    % & 
    % \includegraphics[width=.3\linewidth]{figures/monotonicity/workers2/images5/workers_2_images_3_nosingleton.pdf}
    % \\
    % \small High difficulty & \small Medium to high difficulty & \small Medium difficulty \\
    % \end{tabular}
    % }
    \caption{Empirical success probability per prediction set size averaged across (a) all experts for images with ground truth label \texttt{bottle}, (b) all experts for images with ground truth label \texttt{oven}, (c) experts with low level of competence for images with ground truth label \texttt{bottle}, 
    and (d) experts with high level of competence for images with ground truth label \texttt{bottle} with high, medium to high and medium to low difficulty.
    In all panels, we have only considered prediction sets that included the ground truth label and thus have omitted showing the empirical success probability for singletons, as it is always $1$. 
    Error bars denote standard error.
    }
    \label{fig:acc-set-size-across-experts}
\end{figure}

\clearpage
\newpage

\section{Experiments under the interventional monotonicity assumption }\label{app:additional-results}
In this section, we 
show  the average accuracy $A(\lambda)$ achieved by a human expert using $\Ccal_{\lambda}$, as
predicted by the mixture of MNLs, against the average counterfactual harm upper-bound caused by $\Ccal_{\lambda}$ for several $\alpha$ values 
on the strata of images with $\omega = 80$ in Figure~\ref{fig:acc_vs_harm_interv_noise80}, with $\omega = 95$ in Figure~\ref{fig:acc_vs_harm_interv_noise95} and with $\omega = 110$ in Figure~\ref{fig:acc_vs_harm_interv_noise110}. 
Here, each point corresponds to a different $\lambda$ value and its coloring indicates the empirical probability that $\lambda$ is included in the harm-controlling set $\Lambda'(\alpha)$ given by Corollary~\ref{cor:interventional}, for a fixed $\alpha'$ value across the random samplings of the test and calibration sets.
For each classifier and stratum of images, we select the $\alpha'$ value maximizing the expected size of the harm-controlling set $\Lambda'(\alpha)$. 
% 
% The meaning and coloring of each point is the same as in Figure~\ref{fig:acc_vs_harm_mnl}. 
% 
The results on all strata of images support the findings derived from the experiments assuming that the counterfactual monotonicity assumption holds. 
In addition, we observe that the gap between the lower-bound of the average counterfactual harm (that is equal to the average counterfactual harm under the counterfactual monotonicity assumption) and the upper-bound is often large.
% 
% are consistent with the results on , thus, supporting similar findings---the decision support systems $\Ccal_{\lambda}$ always cause some amount of counterfactual harm, our framework successfully identifies harm-controlling sets and there is a trade-off between accuracy and counterfactual harm. 

\begin{figure}[ht]
    \centering
    \subfloat[$\alpha = 0.12$]{
    % \label{fig:80class-alpha0.01}
    \includegraphics[width=.45\linewidth, valign=t]{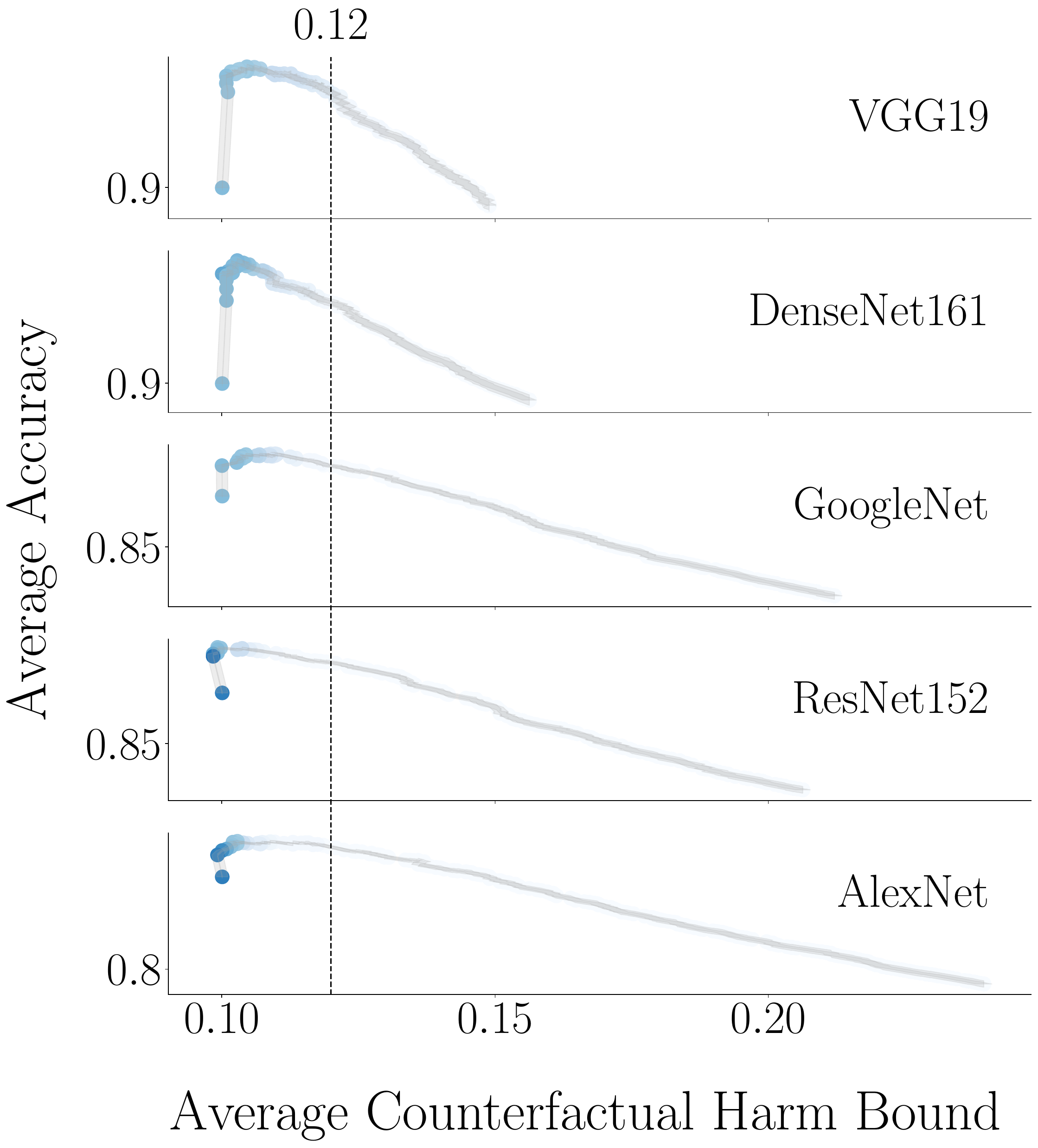}
    }
    \hspace{-0.5mm}
    \subfloat[$\alpha=0.14$]{
    % \label{fig:80class-alpha0.05}
    \includegraphics[width=.45\linewidth, valign=t]{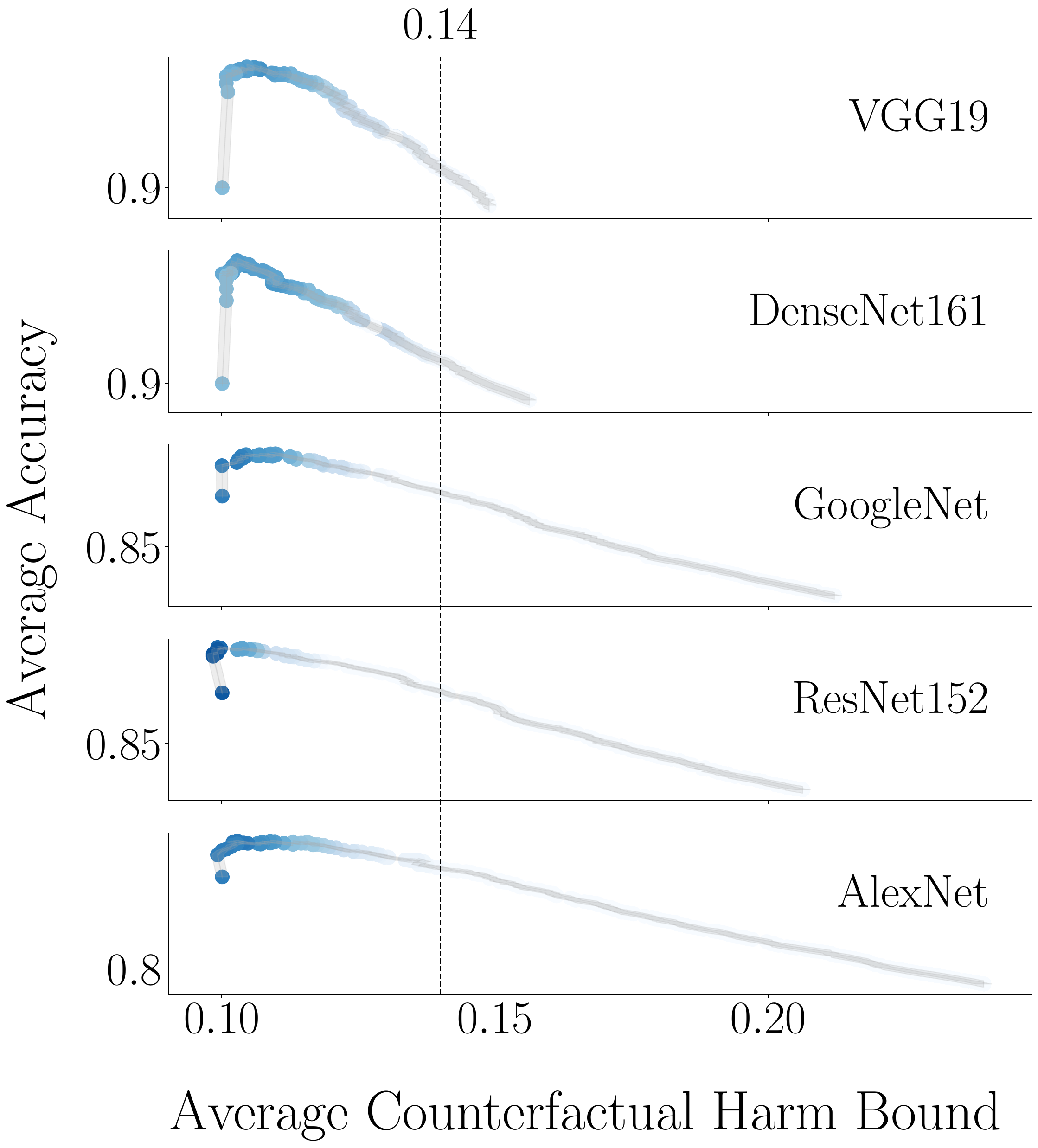}}
    \includegraphics[scale=0.202, valign=t]{figures/colorbar_big.pdf}
     \caption{Average accuracy estimated by the mixture of MNLs against the average counterfactual harm upper bound for images with $\omega = 80$.
    % % 
    Each point corresponds to a $\lambda$ value from $0$ to $1$ with step $0.001$ and the coloring indicates the relative frequency with which each $\lambda$ value is in $\Lambda'(\alpha)$ across random samplings of the calibration set for a fixed $\alpha'$.
    Each row corresponds to decision support systems with a different pre-trained classifiers as in Figure~\ref{fig:acc_vs_harm_mnl_noise80}. 
    % 
    % For each classifier, we choose the $\alpha'$   
    % $\Ccal_{\lambda}$ with a different pre-trained classifier with average accuracies   $0.891$~(VGG19),  $0.892$~(DenseNet),  $0.802$~(GoogleNet), $0.804$~(ResNet152), and $0.784$~(AlexNet).
    % % 
    % The average accuracy of human experts on their own is $0.9$.
    % % 
    The results are averaged across $50$ random samplings of the test and calibration set.
    In both panels, $95\%$ confidence intervals have width always below $0.02$ and are represented using shaded areas.
    }
    % \caption{TBD.}
    \label{fig:acc_vs_harm_interv_noise80}
    \vspace{-3mm}
\end{figure}

\begin{figure}[t]
    \centering
    \subfloat[$\alpha = 0.15$]{
    % \label{fig:80class-alpha0.01}
    \includegraphics[width=.45\linewidth, valign=t]{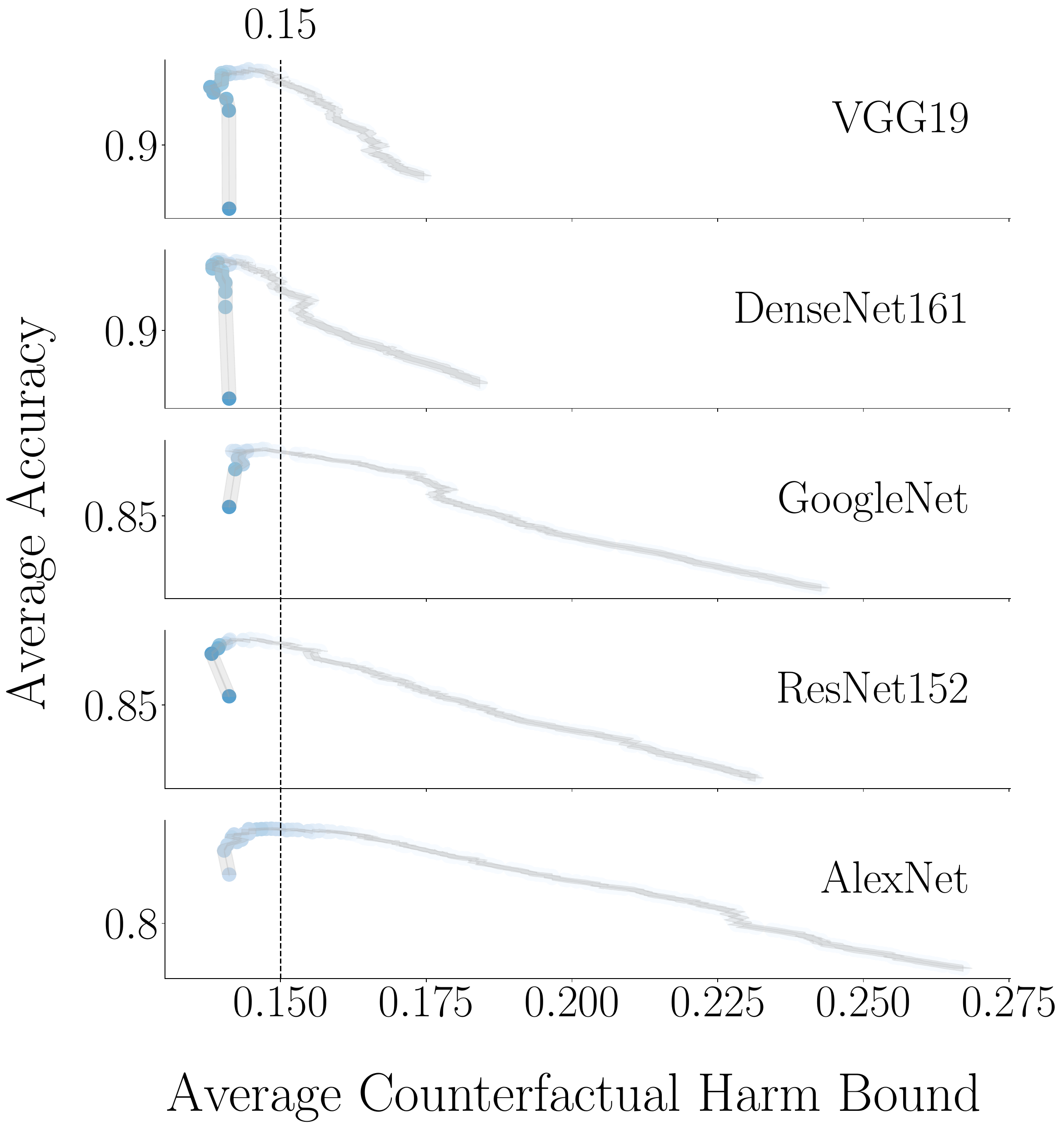}
    }
    \hspace{-0.5mm}
    \subfloat[$\alpha=0.17$]{
    % \label{fig:80class-alpha0.05}
    \includegraphics[width=.45\linewidth, valign=t]{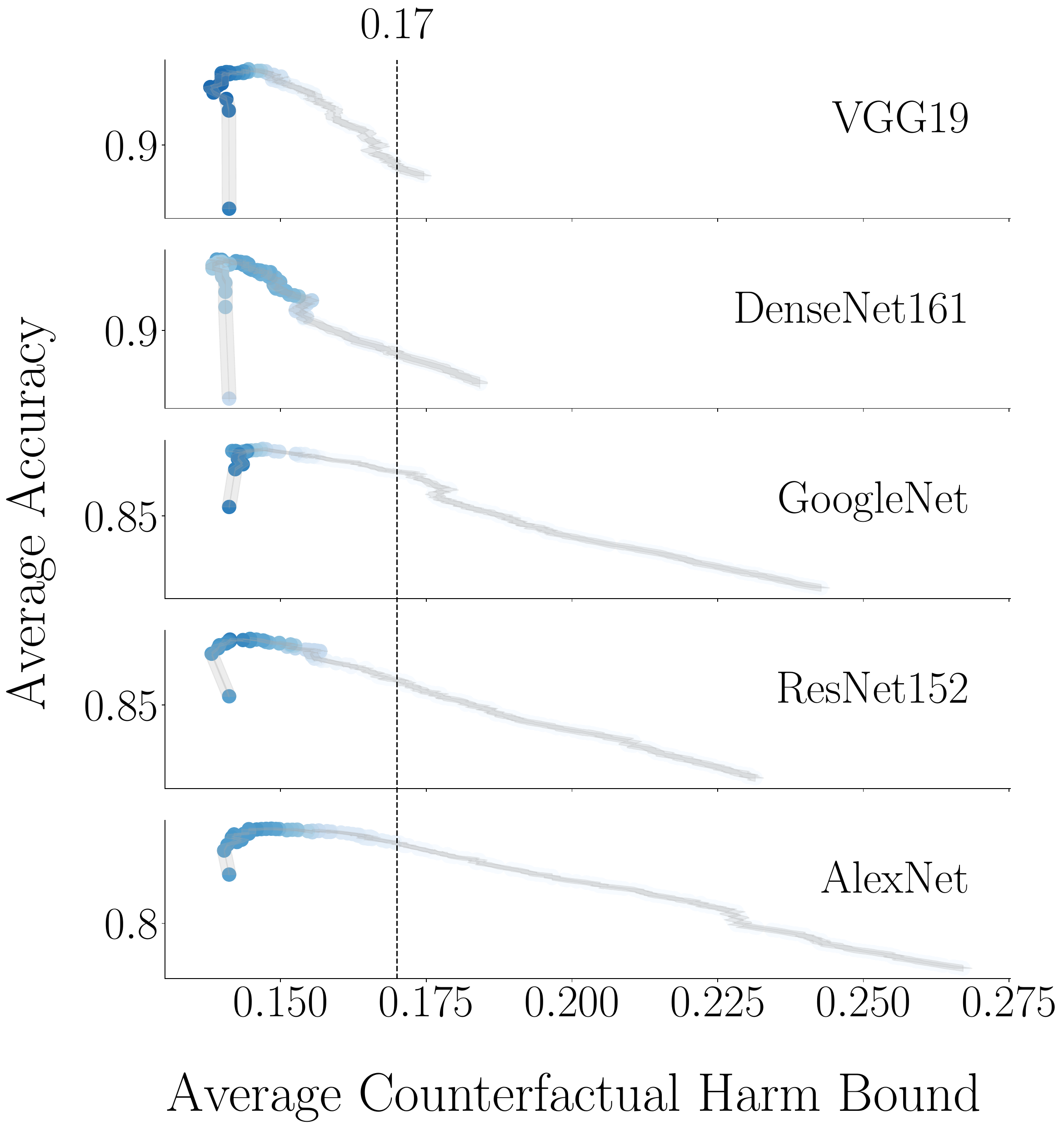}}
    \includegraphics[scale=0.195, valign=t]{figures/colorbar_big.pdf}
     \caption{Average accuracy estimated by the mixture of MNLs against the average counterfactual harm upper bound for images with $\omega = 95$.
    % % 
    Each point corresponds to a $\lambda$ value from $0$ to $1$ with step $0.001$ and the coloring indicates the relative frequency with which each $\lambda$ value is in $\Lambda'(\alpha)$ across random samplings of the calibration set for a fixed $\alpha'$.
    Each row corresponds to decision support systems with a different pre-trained classifiers as in Figure~\ref{fig:acc_vs_harm_mnl_noise95}. 
    % 
    % For each classifier, we choose the $\alpha'$   
    % $\Ccal_{\lambda}$ with a different pre-trained classifier with average accuracies   $0.891$~(VGG19),  $0.892$~(DenseNet),  $0.802$~(GoogleNet), $0.804$~(ResNet152), and $0.784$~(AlexNet).
    % % 
    % The average accuracy of human experts on their own is $0.9$.
    % % 
    The results are averaged across $50$ random samplings of the test and calibration set.
    In both panels, $95\%$ confidence intervals have width always below $0.02$ and are represented using shaded areas.
    }
    \label{fig:acc_vs_harm_interv_noise95}
    \vspace{-3mm}
\end{figure}

\begin{figure}[t]
    \centering
    \subfloat[$\alpha = 0.23$]{
    % \label{fig:80class-alpha0.01}
    \includegraphics[width=.45\linewidth, valign=t]{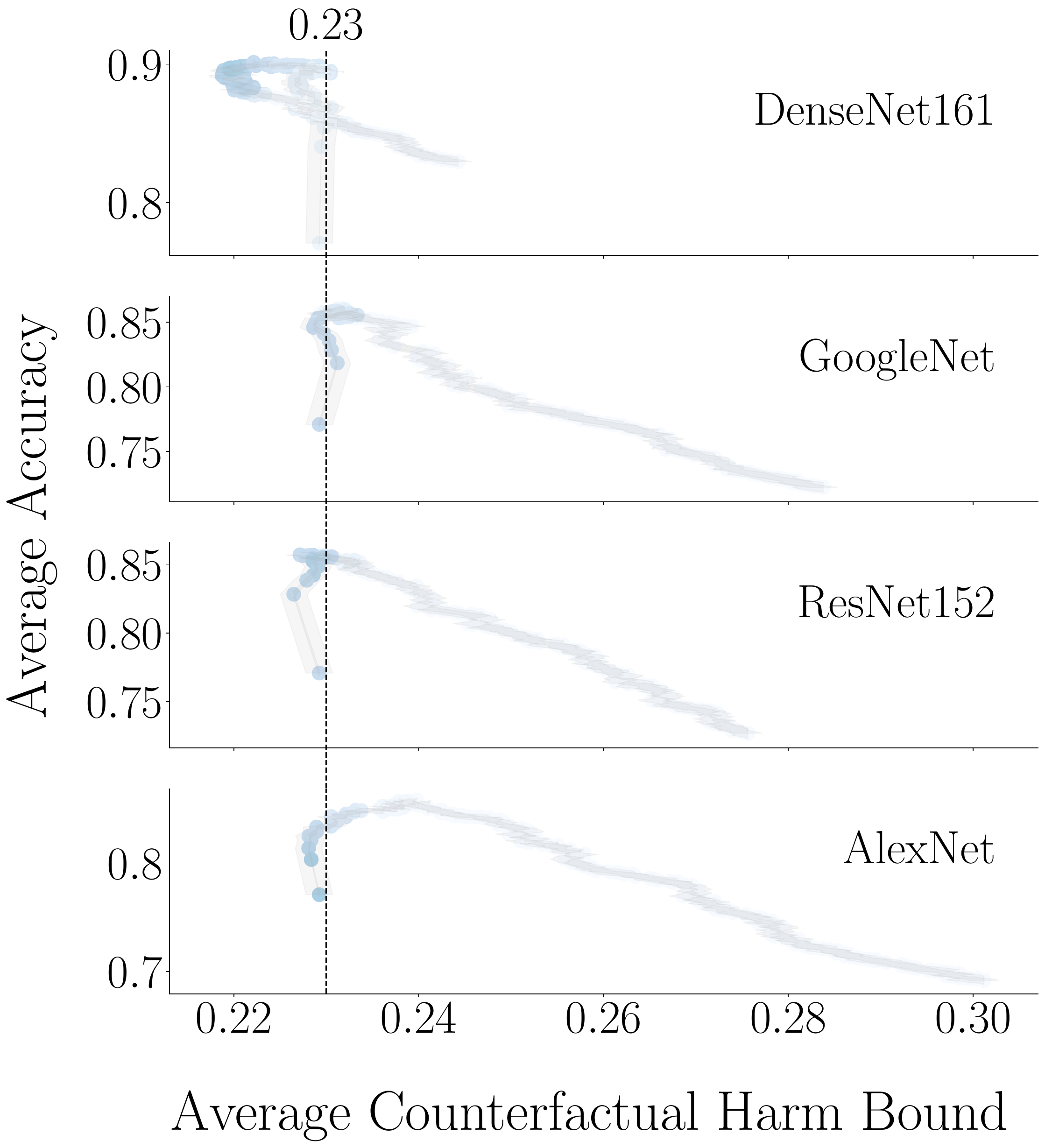}
    }
    \hspace{-0.5mm}
    \subfloat[$\alpha=0.24$]{
    % \label{fig:80class-alpha0.05}
    \includegraphics[width=.45\linewidth, valign=t]{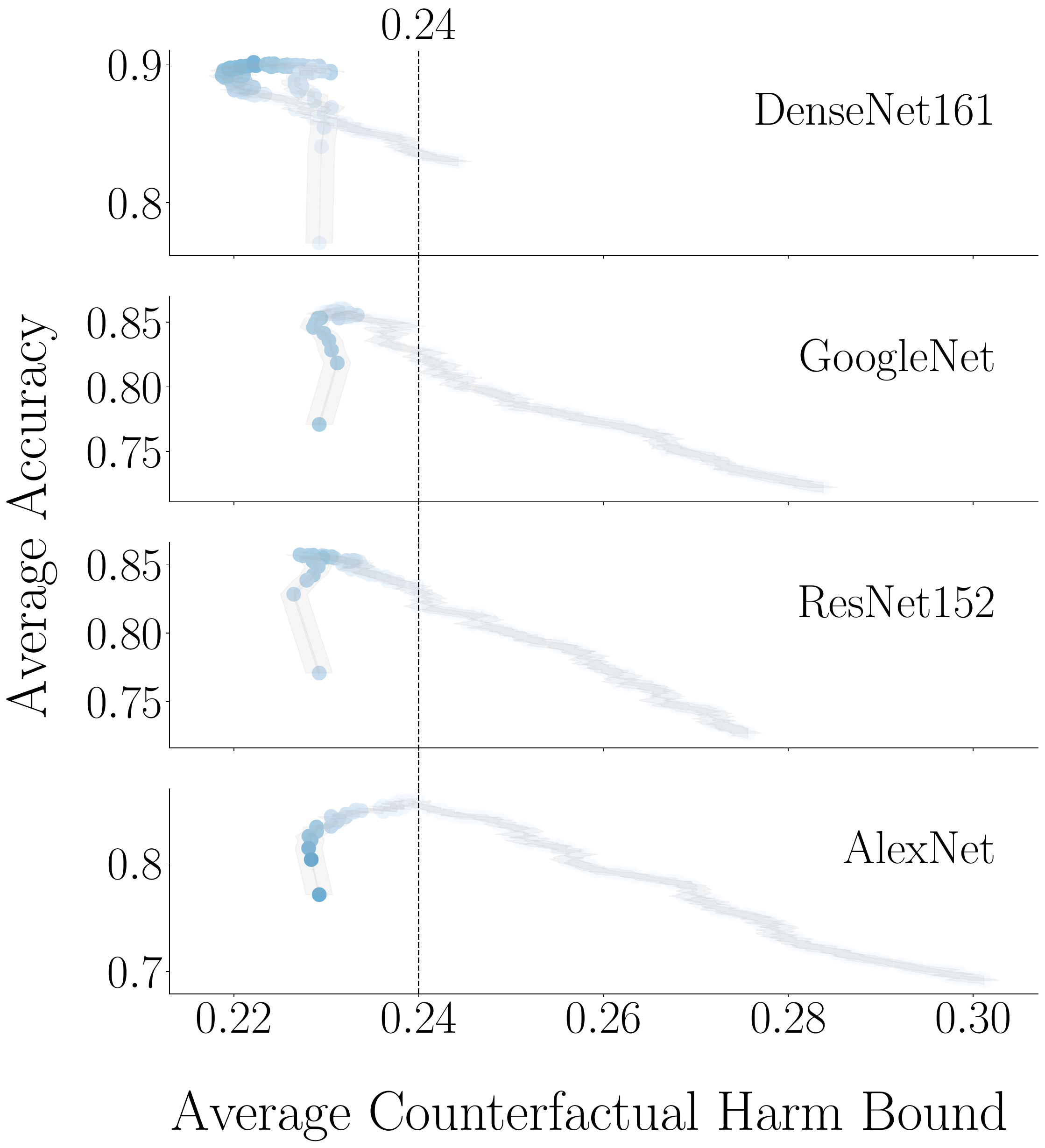}}
    \hspace{0.1mm}
    \includegraphics[scale=0.201, valign=t]{figures/colorbar_big.pdf}
     \caption{Average accuracy estimated by the mixture of MNLs against the average counterfactual harm upper bound for images with $\omega = 110$.
    % % 
    Each point corresponds to a $\lambda$ value from $0$ to $1$ with step $0.001$ and the coloring indicates the relative frequency with which each $\lambda$ value is in $\Lambda'(\alpha)$ across random samplings of the calibration set for a fixed $\alpha'$.
    Each row corresponds to decision support systems with a different pre-trained classifiers as in Figure~\ref{fig:acc_vs_harm_mnl}. 
    For the classifier VGG19, the average counterfactual harm bound is $\sim0.23$ for each $\lambda$ value. 
    % 
    % For each classifier, we choose the $\alpha'$   
    % $\Ccal_{\lambda}$ with a different pre-trained classifier with average accuracies   $0.891$~(VGG19),  $0.892$~(DenseNet),  $0.802$~(GoogleNet), $0.804$~(ResNet152), and $0.784$~(AlexNet).
    % % 
    % The average accuracy of human experts on their own is $0.9$.
    % % 
    The results are averaged across $50$ random samplings of the test and calibration set.
    In both panels, $95\%$ confidence intervals have width always below $0.02$ and are represented using shaded areas.
    }
    \label{fig:acc_vs_harm_interv_noise110}
    \vspace{-3mm}
\end{figure}

\end{document}